\documentclass[journal]{IEEEtran}

\usepackage{ifpdf}
\usepackage{cite}
\usepackage{amsmath,amssymb,amsfonts,bm}
\usepackage{algorithmic}
\usepackage{graphicx, adjustbox}
\usepackage{array}
\usepackage{enumitem}
\usepackage{breqn}
\usepackage{bm}
\usepackage{ dsfont }
\usepackage{ upgreek }
\usepackage{textcomp}
\usepackage{multirow}
\usepackage{algorithm}
\usepackage{algorithmic}
\usepackage{caption}
\usepackage{pgfplots}
\usepackage{tkz-euclide}
\usepackage{subcaption}
\usepackage{tikz}{\tiny }
\usepackage{hyperref}
\usepackage{balance}
\usepackage{ tipa }

\DeclareMathOperator*{\minimize}{minimize}
\DeclareMathOperator*{\maximize}{maximize}
\usepackage{fancyhdr}
\usepackage{ upgreek }
\usepackage{soul}
\usepackage{breqn}
\usepackage[english]{babel}

\newcommand{\Rho}{\raisebox{0pt}{\scalebox{1}{$P$}}}

\newcommand{\Eta}{\raisebox{0pt}{\scalebox{1}{$H$}}}

\addto\captionsenglish{}
\allowdisplaybreaks

\usepackage{amsthm}
\newtheorem{theorem}{Theorem}
\newtheorem{lemma}{Lemma}

\newtheorem{definition}{Definition}

\newtheorem{corollary}{Corollary}

\newtheorem{assumption}{Assumption}
\usepackage{fancyhdr}
\usepackage{ upgreek }
\usepackage{soul}
\makeatletter
\newcommand{\vast}{\bBigg@{4}}

\newcommand{\Vast}{\bBigg@{5}}
\makeatother
\begin{document}

\title{\huge Latency Optimization for Blockchain-Empowered Federated Learning in Multi-Server Edge Computing}
	
	\author{Dinh C. Nguyen,~\IEEEmembership{Member,~IEEE,} Seyyedali Hosseinalipour,~\IEEEmembership{Member,~IEEE,}
	David J. Love,~\IEEEmembership{Fellow,~IEEE,}
	\\
		Pubudu N. Pathirana,~\IEEEmembership{Senior Member,~IEEE,} and 
	Christopher G. Brinton,~\IEEEmembership{Senior Member,~IEEE}
		
		\thanks{Dinh C. Nguyen, Seyyedali Hosseinalipour,  David J. Love, and Christopher G. Brinton are with the Elmore Family School of Electrical and Computer Engineering, Purdue University, USA (e-mails: \{nguye772, hosseina, djlove, cgb\}@purdue.edu)}
		\thanks{Pubudu N. Pathirana is with the School of Engineering, Deakin University, Waurn Ponds, VIC 3216, Australia (e-mail: pubudu.pathirana@deakin.edu.au).}}
			% <-this % stops a space 
	
	\markboth{Accepted to IEEE Journal on Selected Areas in Communications}%
	{}

	\maketitle
	
	\begin{abstract}
	In this paper, we study a new latency optimization problem for blockchain-based federated learning (BFL) in multi-server edge computing. In this system model, distributed mobile devices (MDs) communicate with a set of edge servers (ESs) to handle both machine learning (ML) model training and block mining simultaneously. To assist the ML model training for resource-constrained MDs, we develop an offloading strategy that enables MDs to transmit their data to one of the associated ESs. We then propose a new decentralized ML model aggregation solution at the edge layer based on a consensus mechanism to build a global ML model via peer-to-peer (P2P)-based blockchain communications. \textcolor{black}{Blockchain builds trust among MDs and ESs to facilitate reliable ML model sharing and cooperative consensus formation, and enables rapid elimination of manipulated models caused by poisoning attacks.} We formulate latency-aware BFL as an optimization aiming to minimize the system latency via joint consideration of the data offloading decisions, MDs' transmit power, channel bandwidth allocation for MDs' data offloading, MDs' computational allocation, and hash power allocation. Given the mixed action space of discrete offloading and continuous allocation variables, we propose a novel deep reinforcement learning scheme with a parameterized advantage actor critic algorithm. We theoretically characterize the convergence properties of BFL in terms of the aggregation delay, mini-batch size, and number of P2P communication rounds. \textcolor{black}{Our numerical evaluation demonstrates the superiority of our proposed scheme over baselines in terms of model training efficiency, convergence rate, system latency, and robustness against model poisoning attacks.}
	\end{abstract}
	
	
	\IEEEpeerreviewmaketitle

\section{Introduction}
\label{introduce}
In recent years, the demand for deploying machine learning (ML) in wireless networks and Internet of Things (IoT) applications has increased dramatically. However, due to growing concerns associated with data privacy, it is infeasible to transmit all collected data from IoT edge devices to a central location (e.g., a datacenter) for model training. Federated learning (FL) has emerged as a popular approach for distributed ML which allows for model training without requiring data sharing \cite{2,3}. \textcolor{black}{Under FL, devices operate as workers to train local ML models using their own datasets, and exchange their model updates with an aggregator, e.g., an edge server (ES), over multiple communication rounds to converge on a global model.} While FL distributes the data processing step across devices, the model aggregation step is still often carried out at a single location, which imposes security issues including single-point-of-failure and server malfunction. Additionally, this poses scalability restrictions for ML training processes, especially as the number of IoT devices involved and their geographic reach continue to expand \cite{4}. Therefore, it is desirable to develop a more decentralized FL architecture for realizing scalable model training while preserving security in next-generation intelligent networks. 

In this context, blockchain is a promising technology for enabling reliable decentralized FL via its peer-to-peer (P2P) networking topology empowered by an immutable, transparent and tamper-proof data ledger \cite{5}. The use of blockchain in FL can mitigate the single-point-of-failure issue, and builds trust between devices and multiple servers for secure ML model training \cite{11}. Given these benefits, blockchain-based federated learning (BFL) has been recently investigated in different domains, such as vehicular communications \cite{7} and mobile crowdsensing \cite{8}. To implement BFL systems, devices need to interface with decentralized servers in settings where communication and computation resources are limited, which will impact the ML model training quality. \textcolor{black}{Moreover, each device  exhibits interest in participating in block mining to further gain blockchain rewards, e.g., cryptocurrency tokens, which in turn enhances the reliability and security for FL.} This leads to the concept of \textit{mobile mining} which has been adopted in practical BFL environments \cite{9,10}. The concurrence of both ML model training and block mining introduces new challenges to network service latency and resource management, motivating a holistic optimization architecture for efficient BFL in wireless networks. %, which is one of the main motivations of this paper.

\subsection{Related Works}
We summarize related works in latency optimization and resource allocation for standard FL and decentralized BFL. 

\subsubsection{Standard FL} 
Latency optimization has recently received significant attention in FL research. The work in \cite{12} proposed a joint device scheduling and bandwidth allocation framework for wireless FL to improve the convergence rate of ML model training.  Another study in \cite{13} investigated semi-asynchronous FL, focusing on the convergence analysis of model training under edge heterogeneity and non-independent and identically distributed (non-IID) data distributions across the edge devices. An asynchronous FL scheme was considered in \cite{14} for unmanned aerial vehicles (UAVs)-assisted  networks to minimize the model exchange latency and ML training loss via deep reinforcement learning (DRL). To mitigate  straggler effects caused by resource-limited clients, the authors in \cite{26}  presented a partial offloading-assisted FL scheme using game theory. A partial offloading-based FL solution was also proposed in \cite{15} for edge computing, focusing on the delay analysis of the data offloading and model update. Similarly, a convex optimization approach was applied in \cite{16} to minimize the energy consumption of model updating and sharing.

\subsubsection{Decentralized and Blockchain-based FL}
Several recent works have considered techniques for decentralizing FL aggregation schemes. The authors in \cite{17} introduced a decentralized FL scheme based on model segmentation with a gossip protocol for client sampling in each model aggregation round. In \cite{18}, a decentralized FL solution was proposed using the device-to-device (D2D) concept in serverless edge networks, with a graph-coloring based scheduling policy to characterize the data communications between two devices. Our recent works in \cite{19,20} developed D2D-based semi-decentralized and hybrid FL frameworks, where we jointly modeled communication efficiency and statistical heterogeneity of data and obtained new convergence bounds for distributed ML.

Recently, the integration of blockchain into FL has been investigated. The study in \cite{21} analyzed end-to-end latency for ML model training, update transmission and block mining in a BFL system. \cite{22} analyzed the communication latency and consensus delays for BFL-based vehicular networks, deriving an optimal block arrival rate for vehicles based on system dynamics. The authors in \cite{23} focused on developing a bandwidth allocation and device scheduling solution for digital twin-enabled BFL. Moreover, the work in \cite{24} proposed a BFL-based privacy-preserving UAV network to optimize the energy consumption of UAVs, vehicular device service coverage and a composite service hit ratio. In  recent work \cite{25}, a dynamic resource allocation framework for BFL was proposed with a focus on maximizing the training data size with respect to energy usage constraints. 
%The data learning and block mining latency was analyzed, but how to aggregate the global model is still unclear in this work. Moreover, this work mostly investigated in a static BFL scenario where the prior information of system statistics such as channel conditions and computational resource states is not always available. 

\begin{table}
	\scriptsize
	\centering
    \caption{\textcolor{black}{Comparison of methodology design features between our paper and related works in latency optimization and resource allocation for FL.}}
	\textcolor{black}{\begin{tabular}{|p{2.7cm}||p{0.2cm}|p{0.2cm}|p{0.23cm}|p{0.23cm}|p{0.23cm}|p{0.23cm}|p{0.23cm}|p{0.4cm}|c|}
		\hline
		\centering \multirow{2}{*}{\textbf{Features}} 
		& 	\cite{12}  &	\cite{26} &	\cite{19,20}&	\cite{23} &	\cite{ganguly2022multi}&	\cite{hosseinalipour2022parallel}&\cite{31}, \cite{37} &	Our work\\
		\hline
		FL design in multi-edge servers&	&\checkmark		&	&\checkmark	&	\checkmark&&	& 	\checkmark 
		\\ \hline
		Data offloading-assisted FL&	&\checkmark		&	&		&\checkmark		&\checkmark&&	 	\checkmark
		\\ \hline
        P2P consensus-based model aggregation &		&	&\checkmark		&&	&	 	&&	\checkmark
        \\ \hline
        Blockchain-based FL design &		&	&	&\checkmark	&& 	&&	\checkmark
        \\ \hline
        DRL-based resource allocation&	&	&	&	\checkmark	&&	&\checkmark&	 \checkmark
        \\ \hline
        Parameterized A2C design&	&	&		&	&	&	&& 	\checkmark
        \\ \hline
        FL latency optimization&\checkmark	&	&	&	&	&	\checkmark&&	 	\checkmark
        \\ \hline
	\end{tabular}}
	\label{table:FeatureComparisons}
	\vspace{-0.1in}
\end{table}

\subsection{\textcolor{black}{Motivations and Key Contributions}}
\label{SectionI-motivate}
Despite such research efforts, several limitations still exist in current BFL works, which are highlighted below:
\begin{itemize}
	\item Most current standard FL \cite{12,13}, \cite{15} and decentralized FL frameworks \cite{19} still rely on a single ES to coordinate model aggregations.  In these architectures, single-point-of-failure bottlenecks may disrupt the entire FL system if the server is attacked. Only the work in \cite{23} has considered a multi-server edge computing model for BFL, but its model aggregation still follows traditional FL.
	\item End-to-end latency optimization, i.e., for both model training and block mining, remains understudied in current BFL systems \cite{21,22,23}. Existing works mostly aim to characterize, rather than optimize, BFL latency. \textcolor{black}{Moreover, the benefits of blockchain to support robust BFL training against model attacks have not been yet investigated.}
	\item The potential benefits of edge computing have not been well exploited in existing BFL schemes. Most works \cite{14,16,25} have not considered practical resource constraints of mobile IoT devices and the potential for ESs to mitigate resulting straggler effects. Only \cite{26,15} have considered such a scenario, with the focus instead on partial offloading for traditional FL.
	\item None of the existing works have analyzed the convergence properties of BFL in a multi-server system. A comprehensive theoretical analysis will provide insights into BFL operations, leading to potential optimization techniques.
\end{itemize}
\textcolor{black}{Motivated by the aforementioned limitations, we propose a novel consensus-based BFL model for efficient and robust ML model training via blockchain. Specifically,  we develop a new cooperative offloading-assisted model learning and resource trading-assisted block mining framework for FL. We then propose a partial model aggregation solution for facilitating global model aggregations at the edge layer using blockchain-enabled P2P communications. Blockchain is important for our methodology in two key ways: (i) it builds trust among MDs and ESs to facilitate reliable ML model sharing and cooperative consensus formation for our federated learning approach; and (ii) it allows for rapid elimination of manipulated models from compromised ESs caused by poisoning attacks, thereby enhancing the robustness of global model training.  The system latency is subsequently formulated by considering both model learning latency and block mining latency, which is then optimized by a  parameterized actor-critic algorithm.} The comparison of our paper with related works in terms of several key design features is summarized in Table~\ref{table:FeatureComparisons}. In summary, the unique contributions of this paper are:
\begin{enumerate}
	\item We propose a multi-server-assisted BFL architecture, where geo-distributed mobile devices (MDs) communicate with a set of ESs for ML model training and block mining simultaneously. To mitigate straggler effects caused by resource-constrained MDs, an offloading strategy is proposed that enables MD data transmission to an ES for ML model training. Moreover, we develop a resource trading strategy to alleviate block mining latency from resource-limited MDs.
	%considering the latency of mobile block mining in BFL,
	\item We provide a holistic convergence analysis of BFL. In doing so, we consider a new partial model aggregation solution for facilitating global model aggregations at the edge layer via P2P-based blockchain communications. Our resulting bound reveals the impact of the aggregation delay, mini-batch size, and number of P2P communication rounds on the convergence rate.
	\item We formulate a new system latency minimization problem, taking into account both offloading-assisted ML model training latency, model consensus latency  and block mining latency. This optimization couples data offloading decisions, MD transmit powers, and the allocation of channel bandwidth, MD computation, and hash power resources to minimize latency with a model quality constraint.
	\item To solve the resulting optimization over a mixed discrete and continuous solution space, we propose a novel DRL method based on a parameterized advantage actor critic (A2C) algorithm. We provide a holistic design of the actor, including offloading and allocation policies empowered by trust region policy optimization (TRPO), along with a critic for the state-value training. 
	\item We conduct numerical experiments for both our consensus-based BFL and parameterized  A2C schemes. The results reveal that our BFL scheme outperforms existing FL approaches in terms of model loss and accuracy convergence in both IID and non-IID data settings. Our proposed parameterized A2C scheme also helps lower the system latency by up to 38\% compared with state-of-the-art DRL schemes. \textcolor{black}{Moreover, our blockchain-empowered BFL scheme shows  high robustness against model poisoning attacks. }
\end{enumerate}

\textcolor{black}{\subsection{Paper Organization}
The remainder of this paper is organized as follows. Section~\ref{Section:SystemModel} presents the BFL architecture and its different components including the ML training and model aggregation procedures. Moreover, we conduct an analysis of the BFL model training to characterize the impact of different system parameters on learning convergence. In Section~\ref{Section:Latency}, we formulate the corresponding system latency minimization problem, taking into account offloading-assisted ML model training latency, model consensus latency, and block mining latency. To solve the formulated problem, we propose a DRL method based on a parameterized advantage actor critic (A2C) algorithm in Section~\ref{Section:DRL}, where the utility measure captures the latency objective and the action space is restricted by the learning and resource constraints. We present experiments comparing latency and accuracy obtained by our methodology against several baselines  in Section~\ref{Section:Simulate}. Finally, Section~\ref{Section:Conclude} concludes the paper. The key acronyms and notations used in this paper are summarized in Table~\ref{Table:listofAcronyms} and Table~\ref{Table:listofNotation}, respectively. }

\begin{table}
	\caption{\textcolor{black}{List of key acronyms.}}
	\label{Table:listofAcronyms}
	\scriptsize
	\centering
	\captionsetup{font=scriptsize}
	\setlength{\tabcolsep}{5pt}
	\textcolor{black}{\begin{tabular}{|p{1cm}|p{2.6cm}|p{1cm}|p{2.6cm}|}
		\hline
		\textbf{Acronym}& 
		\textbf{Definition}&
		\textbf{Acronym}& 
		\textbf{Definition}
		\\
		\hline
		FL &	Federated learning & ML &Machine learning
		\\
		\hline
        BFL & Blockchain-based federated learning &MD &Mobile device
        \\
        \hline
        ES &Edge server & P2P &Peer-to-peer
        \\
        \hline
         UAV &Unmanned aerial vehicle & IID & Independent and identically distributed
        \\
        \hline
         DRL &Deep reinforcement learning & A2C &Advantage actor critic
        \\
        \hline
         TRPO &Trust region policy optimization & SGD &Stochastic gradient descent 
        \\
        \hline
         DNN &Deep neural network &   DDPG &Deep deterministic policy gradient
        \\
		\hline
		MSPBE &Mean Squared Projected Bellman Error &   KL&Kullback–Leibler
        \\
		\hline
	\end{tabular}}
	\label{tab1}
	\vspace{-0.15in}
\end{table}

%\textbf{Due to space limitations, note that certain math details have been deferred to our online technical report~\cite{tech}.}

%\subsection{Structure of this Paper}
%The remainder of this paper is organized as follows. Section~\ref{Section:SystemModel} presents the overall BFL system and explain its working procedure. We also present a detailed analysis of the characteristics of federated data learning with blockchain, and the convergence properties are provided. We then elaborate the system models of data learning and block mining, followed by the formulation of a system latency minimization problem for our BFL system in Section~\ref{Section:Latency}. Subsequently, we present a detailed design of a new parameterized A2C algorithm in Section~\ref{Section:DRL} for system latency optimization. The simulation results are given and analyzed in Section~\ref{Section:Simulate}, and Section~\ref{Section:Conclude} concludes the paper. 
\begin{table}
	\caption{\textcolor{black}{List of key  notations.}}
	\label{Table:listofNotation}
	\scriptsize
	\centering
	\captionsetup{font=scriptsize}
	\setlength{\tabcolsep}{5pt}
	\textcolor{black}{\begin{tabular}{|p{1cm}|p{2.6cm}|p{1cm}|p{2.6cm}|}
		\hline
		\textbf{Notation}& 
		\textbf{Definition}&
		\textbf{Notation}& 
		\textbf{Definition}
		\\
		\hline
		$M$& Number of ESs & $N$  & Number of MDs
		\\
		\hline
       $K$ & Number of sub-channels  & $D_n$  &   MD's Data size
        \\
        \hline
        $f^{\mathsf{\ell}}_n$ &  MD's CPU workload  & $f_m$ &  ES's CPU workload 
        \\
        \hline
        $p_n$ & MDs' transmit power & $b_{n,g}$ & MD's bandwidth
        \\
        \hline
         $\vartheta$ & MD's model size & $\Psi_n$ & MD's hash rate
        \\
        \hline
         $\hbar$ & Hash amount of a block  & $b_{m'}$ & ES's bandwidth
        \\
        \hline
         $p_{m'}$ & ES's transmit power &     $g$&Wireless channel
         \\
        \hline
        $x_{n,m,g}$&MD's offloading decision &  $k$ &Global aggregation round
        \\
        \hline
         $\boldsymbol{w}$&ML model parameter & $\textbf{x}_{m}$&ES's gradient
        \\
        \hline
        $T_{n,m}^{\mathsf{off},{(k)}}$ &MD's  offloading latency &        $T^{\mathsf{exe},(k)}_m$&ES's   execution latency
        \\
		\hline
		$T_n^{\mathsf{loc},(k)}$ &MD's local data processing latency & $T_{n,m}^{\mathsf{up},(k)}$ &MD's model uploading latency
		\\
		\hline
		$T_m^{\mathsf{update},{(k)}}$ &ES's model updating latency & $T^{\mathsf{learn},(k)}$ &Total learning latency
		\\
		\hline
		$T^{\mathsf{cons},(k)}$ &Total model consensus latency &   $T_n^{\mathsf{mine},(k)}$ &Total mining latency
		\\
		\hline
	\end{tabular}}
	\label{tab1}
	\vspace{-0.15in}
\end{table}

\section{System Model}
\label{Section:SystemModel}
\begin{figure}
	\centering
	\includegraphics [width=0.99\linewidth]{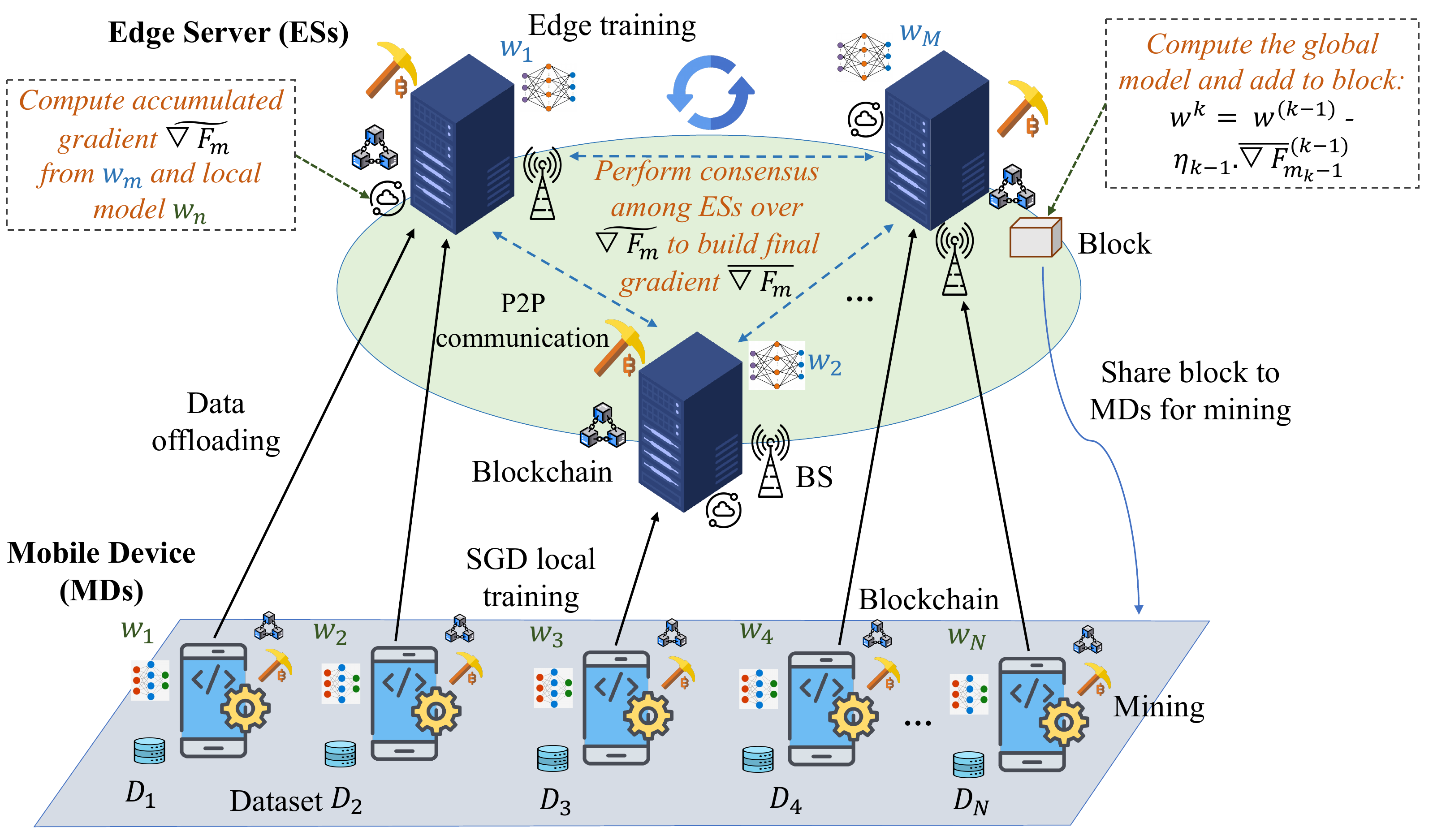}
	\caption{{Our proposed BFL architecture in multi-server edge computing.}}
	\label{Fig:Overview}
	\vspace{-0.2in}
\end{figure}

In this section, we describe our BFL system and detail the BFL task model, and then analyze its convergence properties. 
\subsection{Overall System Architecture}
Our proposed BFL architecture is illustrated in Fig.~\ref{Fig:Overview}. Each ES is located at a base station (BS) to provide computation services for multiple MDs concurrently. We consider $N$ MDs gathered in the set $\mathcal{N}$, $N=|\mathcal{N}|$, which are connected to $M$ ESs collected in the set $\mathcal{M}$, $M=|\mathcal{M}|$, in a multi-server edge computing network. The goal of the system is to train an ML model, with training proceeding through a series of global aggregation rounds collected via set $\mathcal{K}$. In each round $k\in\mathcal{K}$, each MD $n$ possesses a dataset $\mathcal{D}_n^{(k)}$, which may vary from one round to the next, with size $D_n^{(k)} = |\mathcal{D}_n^{(k)}|$. Each dataset $\mathcal{D}_n^{(k)}$ contains multiple data samples, each consisting of an input feature vector and (e.g., supervised learning) a label. The MDs employ these datasets for local training in round $k$, formalized in Section~\ref{subsection:BFL_model}. The total dataset is given by the set $\mathcal{D}^{(k)} = \cup_{n \in \mathcal{N}}\mathcal{D}_n^{(k)}$ with size $D^{(k)} = \sum_{n \in \mathcal{N}} D_n^{(k)}$.

The interactions between MDs, ESs, and blockchain components of our BFL system are summarized as follows:
%local model training iterations followed by 
\begin{itemize}
%    \item \textcolor{black}{\textit{Datasets} are collected and stored at local MDs. They are used for ML model training for MDs' intelligence applications, such as neural network-based object classification.} 
    \item \textit{MDs} participate in ML model training in a federated manner to serve their intelligence applications (e.g., object detection). Moreover, they work as blockchain nodes to mine blocks containing global models in each communication round to support BFL model sharing.
    \item \textit{ESs} assist in MD model training by providing computation resources through the offloading process. They also coordinate the model aggregation process to build the global model that is shared with MDs via blockchain.
    \item \textit{Blockchain} \textcolor{black}{allows MDs to securely transmit their computed local model to ESs via blockchain. It also facilitates the model consensus process between ESs with a traceable data ledger. Moreover, blockchain enables secure global model sharing from ESs to MDs.}  \textcolor{black}{After the model aggregation process completes, the global model is added to the blockchain, where mining is executed for secure model sharing. Specifically, an ES will be randomly sampled to function as a leader node and build an unverified block which contains the global model. This block is then shared with other ESs and MDs for mining. Subsequently, each MD downloads the verified block from the blockchain to extract the global model which is used for the next training round.}   As illustrated in Fig.~\ref{Fig:Blockchain_Arch}, each block includes (i) a header, with a hash and cryptographic nonce, and (ii) the data part. To construct a block, an ES generates a \textit{transaction} using its aggregated model and \textit{hashes} it, resulting in an output of a fixed length. The block is then shared with the MDs for model training.
% and making it more difficult for attackers to recover information in the BFL model
\end{itemize}
\begin{figure}
	\centering
	\includegraphics [width=0.95\linewidth]{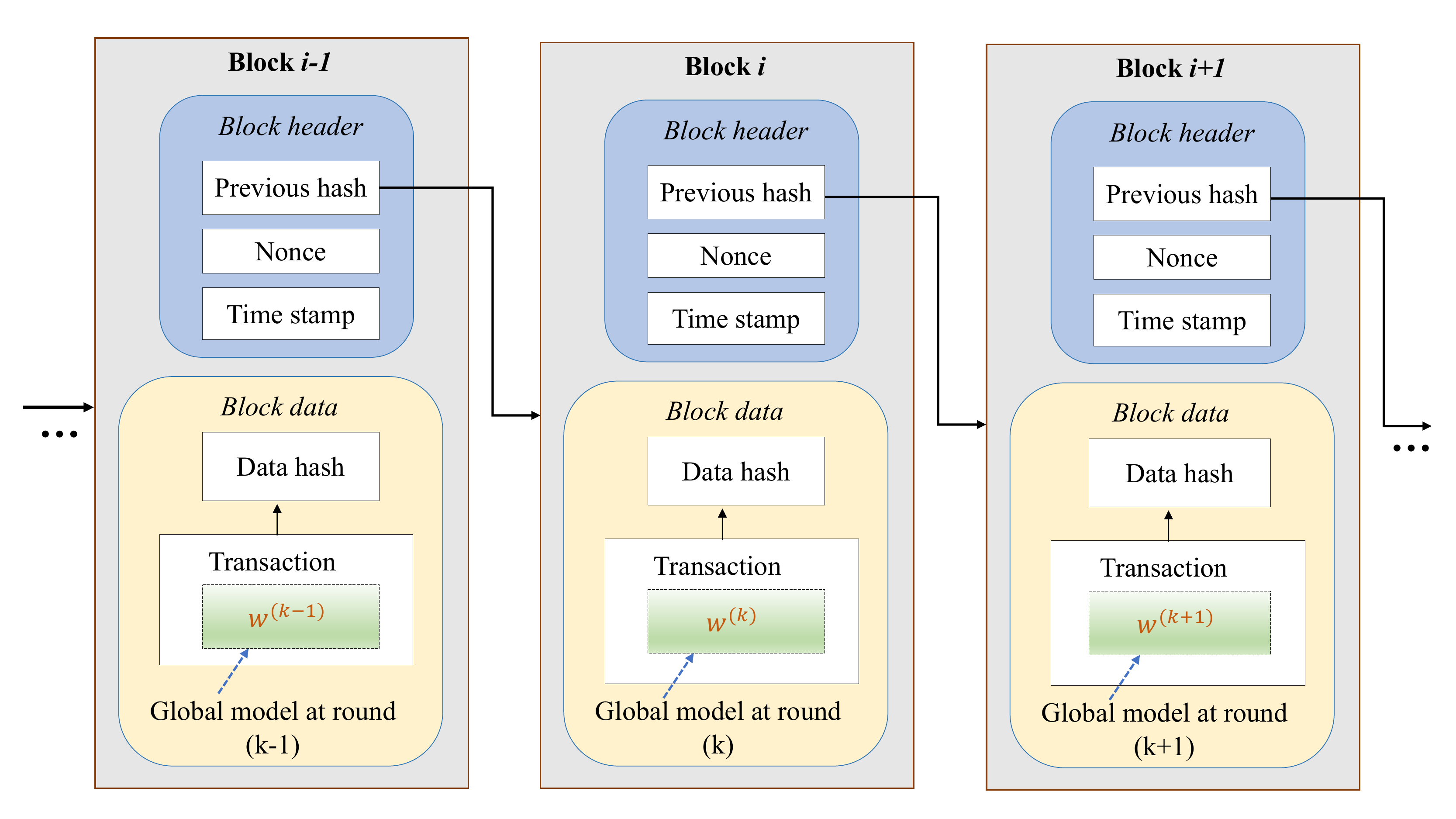}
	\caption{\textcolor{black}{A block architecture in our blockchain for global model sharing, where a global model $\boldsymbol{w}^{(k)}$ is embedded into the transaction of a block. The transaction is also used to transfer the local gradient $ {\nabla {F}}_n^{\mathsf{C},(k)}$ of MDs and edge gradient $\textbf{x}_{m}^{(l)}$ of ESs in the model uploading and model consensus processes, respectively. }}
	\label{Fig:Blockchain_Arch}
	\vspace{-0.2in}
\end{figure}

We let $\mathcal{G}$, $G=|\mathcal{G}|$, denote the set of available sub-channels for communication at a BS. We assume that uplink communications between MDs and ESs follow an OFDMA-based protocol, where each MD is assigned to a sub-channel $g \in \mathcal{G}$ to offload its data in each time slot, and thus each ES can serve at most $G$ MDs in every offloading period. We define the data offloading policy, which incorporates the uplink sub-channel scheduling, through a binary variable $x_{n,m,g}^{(k)}$,  ($n \in \mathcal{N}, m \in \mathcal{M}, g \in \mathcal{G}$), where $ x_{n,m,g}^{(k)} = 1$ indicates that data $\mathcal{D}^{(k)}_n$ from MD $n$ is offloaded to ES $m$ via sub-channel $g$ in round $k$, and $x_{n,m,g}^{(k)} = 0$ otherwise. Each dataset can be either trained locally at the MD or offloaded to the ES under a feasible offloading policy  $\boldsymbol{X}^{(k)} = \{x_{n,m,g}^{(k)}| x_{n,m,g}^{(k)} \in \{0,1 \}, \forall n \in \mathcal{N}, m \in \mathcal{M}, g \in \mathcal{G}\}$. In each training round, the BFL operation consists of four key steps, as depicted in Fig.~\ref{Fig:Overview}:
\begin{enumerate}
	\item \textit{Data Offloading and Processing:} Depending on its available resource, each MD can choose to offload its data to one of the nearby ESs for edge learning, or learn the model locally via local data processing. We assume that the model learning process begins once the offloading phase is completed~\cite{15}. This assumption is realistic in practical scenarios, where each MD needs to obtain an offloading policy on whether it should offload the data to an ES or not, before it  allows the associated ES to train its data.  For both training cases, MDs and ESs conduct stochastic gradient descent (SGD) iterations and  synchronize their model parameters through  block mining.  \textcolor{black}{After model training, a MD uploads its ML model parameters to its nearby ES via blockchain.} Considering an ES, if it receives data from MDs, it allocates its resources to conduct ML model training.
	%, aiming to build an aggregated ML model
	\item \textit{Partial Model Aggregation and Consensus Update:} \textcolor{black}{Each ES combines its computed ML model with models received from its associated MDs to perform a partial model aggregation. \textcolor{black}{Then, ESs will join a consensus update process in which they conduct multiple rounds of blockchain-enabled \textit{peer-to-peer (P2P)} communications to exchange their models.}} After the consensus update, an ES will be randomly sampled to work as a leader and build an unverified block that contains the aggregated model for mining.
	\item \textit{Block Mining:} \textcolor{black}{The leader ES broadcasts the block via the blockchain to the connected ESs and MDs for mining.} In this work, we are mostly interested in the mining latency from the users' perspective, and thus we analyze the mining process at the MDs. Given resource constraints at the MDs, we develop a resource trading strategy where MDs can purchase hash power from the edge/cloud service provider (ESP) (e.g., Amazon cloud services) to run the mining task~\cite{nguyen2020privacy}. 
	%The key purpose of mining is to ensure secure model sharing between ESs and MDs via block verification for reliable model training.
	%Accordingly, the leader shares the unverified block to local MDs for mobile mining. 
	\item \textit{Block Generation:} 
	After receipt of the unverified block,
	the MD that is the first to successfully verify the block will append it to the blockchain and receive a reward. Then, each MD downloads the verified block via its blockchain account~\cite{nguyen2020privacy} that contains the global model, and uses this for local model synchronization to begin the next round of model training.
\end{enumerate}
%\textcolor{black}{Therefore, in the BFL system, the blockchain is integrated in the global model sharing stage of the FL process. After the global model aggregation at the edge layer completes, blockchain is adopted to establish a model sharing platform, where the global model is embedded into the blockchain to securely share with local MDs via an immutable data ledger. Blockchain thus replaces  the centralized server in traditional FL networks to decentralize the model training and manage the model updating in a fashion that the information of global model is tamper-proof and reliable.}
\subsection{\textcolor{black}{Federated Learning Model}}
%{ML Task, ML Model Training, and Model Aggregation}}
\label{subsection:BFL_model}
%Since local model updates begin after the data offloading phase is completed, the ESs and MDs follow a synchronous discrete timescale for the ML model training in each model training (global aggregation) round. 
Let $f(\boldsymbol{w}, i) \in \mathds{R}$ denote the loss function of the ML model (e.g., neural network) associated with data point $i$, where $\boldsymbol{w} \in \mathds{R}^d$ is the parameter vector (e.g., weights on neurons). \textcolor{black}{In this work, we consider a scenario of full ML training at MDs, i.e., each MD either keeps its entire data locally or offloads it completely to an ES.  When the MD keeps its data local, it offloads the entire trained ML model to the ES.} We measure the \textit{online loss function} of each MD $n$ at each global aggregation round $k$ as
\begin{equation}
\label{loss_device}
F_n^{(k)}(\boldsymbol{w}) = \frac{1}{D_n^{(k)}} \sum_{i \in \mathcal{D}_n^{(k)}} f(\boldsymbol{w}, i),
 \end{equation}
 and subsequently define the \textit{online global loss} as
 \begin{equation}
\label{loss_device}
F^{(k)}(\boldsymbol{w}) = \frac{1}{D^{(k)}} \sum_{n \in \mathcal{N}} D_n^{(k)} F^{(k)}_n(\boldsymbol{w}),~D^{(k)}=\sum_{n \in \mathcal{N}} D_n^{(k)}.
 \end{equation}
% where $\boldsymbol{w}_n$ is the model parameter of MD $n$. 
% In the case that MD $n$ chooses to offload its data $D_n$ to ES $m$, this server will train the received data on behalf of the device. The  loss function at ES $m$ for training the data $D_n$ of MD $n$ is defined as
% \begin{equation}
% \label{loss_server}
% {F_m(\boldsymbol{w}_m^n)} = \frac{1}{D_n} \sum_{i \in \mathcal{D}_n} f_i(\boldsymbol{w}_m^n),
% \end{equation}
% where $\boldsymbol{w}_m^n$ is the model parameter of ES $m$ w.r.t MD $n$. 

%Assuming that each ES receives models from at least one MD, 
In our BFL system, each ES processes datasets offloaded from a portion of MDs, and also receives the computed ML models uploaded from another portion of nearby MDs which had chosen local processing. We denote $\mathcal{N}_m^{\mathsf{off},(k)}$ and $\mathcal{N}_m^{\mathsf{loc},(k)}$ as the sets of MDs engaged in data offloading and model uploading with ES $m$, respectively ($\mathcal{N}_m^{\mathsf{off},(k)}, \mathcal{N}_m^{\mathsf{loc},(k)} \subseteq \mathcal{N}$ and $\mathcal{N}_m^{\mathsf{off},(k)} \cap \mathcal{N}_m^{\mathsf{loc},(k)} = \emptyset$).  
\textcolor{black}{Letting $\mathcal{D}_m^{(k)} = \cup_{n \in \mathcal{N}_m^{\mathsf{off},(k)}} \mathcal{D}_n^{(k)}$ denote ES $m$'s dataset received from MDs, with size ${D}_m^{(k)}$, the loss at ES $m$ is given as 
\begin{equation}\label{loss_server}
F_m^{(k)}(\boldsymbol{w}) = \frac{1}{{D}_m^{(k)}}  \sum_{i \in \mathcal{D}_m^{(k)}} f(\boldsymbol{w},i). 
\end{equation}}

We now formalize the ML model training procedure at ESs and MDs. Each model training round $k$ starts with the broadcast of a global model, $\boldsymbol{w}^{(k)}$, from one of the ESs.
% We assume that the MDs associated with ES $m$ during global aggregation $k$ constitute the set $\mathcal{N}_m^{(k)}$, where for simplicity of notations $\mathcal{N}_m^{(k)}$ also contains ES $m$ itself. 
During round $k$, each
	ES $m$ performs $e^{(k)}_m$ iterations of SGD over its local/offloaded dataset, which may vary from one ES to another, where the evolution of its local model parameters is given by
	\begin{equation}\label{eq:Weightupdate}
	% \begin{aligned}
	\boldsymbol{w}_m^{(k),e}\hspace{-0.5mm}=\boldsymbol{w}^{(k),e-1}_{m}\hspace{-0.5mm} - \hspace{-0.5mm}{\frac{\eta_{_k}}{{B}^{(k)}_m}} \hspace{-1mm}\sum_{d\in \mathcal{B}^{(k),e}_{m}} \hspace{-3mm} {\nabla  f(\boldsymbol{w}^{(k),e-1}_{m},d)}\hspace{-0.1mm},
	% \end{aligned}
	%\vspace{-2.5mm}
	\end{equation}
	where  $\eta_{_k} > 0$ is the step-size and $e\in\{1,\cdots,e^{(k)}_n\}$ is the index of local iteration with $\boldsymbol{w}^{(k),0}_{m}=\boldsymbol{w}^{(k)}$. In~\eqref{eq:Weightupdate}, $\mathcal{B}^{(k),e}_{m}$ denotes the set of data points sampled at the $e$-th iteration from the local dataset of the respective MD to perform mini-batch SGD. 
	We assume that the mini-batch size ${B}^{(k)}_{m}=|\mathcal{B}^{(k),e}_{m}|$, $\forall e$, is fixed during each local model training round $k$ for each ES $m$.
	The local model training at each MD $n$ is also similar to \eqref{eq:Weightupdate}, where each MD $n$ performs $e^{(k)}_n$ iterations of SGD with local updates as
	\begin{equation}\label{eq:WeightupdateMD}
	\boldsymbol{w}^{(k),e}\hspace{-0.5mm}=\boldsymbol{w}^{(k),e-1}_{n}- {\frac{\eta_{_k}}{{B}^{(k)}_n}} \hspace{-1mm}\sum_{d\in \mathcal{B}^{(k),e}_{n}} {\nabla  f_n(\boldsymbol{w}^{(k),e-1}_{n},d)}
	\end{equation}
	After model training, each ES $m$ computes its \textit{cumulative gradient}:
	\begin{equation} \label{eq:cumulativeES}
	   {\nabla {F}}_m^{\mathsf{C},(k)} = \big(\boldsymbol{w}^{(k)}-\boldsymbol{w}_m^{(k),e^{(k)}_m}\big)\big/\eta_{_k}. 
	\end{equation}
	Similarly, each MD $n$ also obtains its gradient:
	\begin{equation}\label{eq:cumulativeMD}
	 {\nabla {F}}_n^{\mathsf{C},(k)} = \big(\boldsymbol{w}^{(k)}-\boldsymbol{w}_n^{(k),e^{(k)}_n}\big)\big/\eta_{_k}.
	\end{equation}
	
	Subsequently, the MDs offload their cumulative gradients to their associated  ES via blockchain. 
	%which conduct local model training and offload their cumulative gradients to the ES, 
	Each ES $m\in \mathcal{M}$ subsequently acquires its  \textit{aggregated gradient}, \textcolor{black}{which is a scaled sum (with respect to the number of SGD iterations and the number of data points) of its cumulative gradient and that of its associated MDs, as}
	\begin{equation}
	\label{equation:modelconsensus}
	{\nabla {F}}^{\mathsf{A},(k)}_{m}= \hspace{-2mm}
	\sum_{n\in \mathcal{N}^{\mathsf{loc},(k)}_m} \hspace{-1mm} \frac{{{D}}^{(k)}_n}{{D}^{(k)} e^{(k)}_n}{\nabla {F}_n^{\mathsf{C},(k)}} + \frac{{{D}}^{(k)}_m}{{D}^{(k)} e^{(k)}_m}{\nabla {F}_m^{\mathsf{C},(k)}}.\hspace{-1mm}
	\end{equation}
	
	\textcolor{black}{The ESs then engage in P2P communications for cooperative consensus formation among their aggregated gradients. For this purpose, we assume that they exploit \textit{linear distributed consensus} iterations~\cite{xiao2004fast}, where during \textit{training round} $k$, each ES $m\in\mathcal{M}$ conducts $\phi^{(k)}\in \mathbb{N}$
	\textit{consensus rounds} of P2P communications with its neighboring ESs.}
	During each round $l=0,\cdots,\phi^{(k)}-1$ of P2P communications, the evolution of the local gradient of ES $m\in\mathcal{M}$ can be expressed as:
	\begin{equation}\label{eq:ConsCenter}
	\textbf{x}_{m}^{(l+1)}= \lambda^{(k)}_{m,m} \textbf{x}_{m}^{(l)}+ \sum_{m'\in \varrho(m)} \lambda^{(k)}_{m,m'}\textbf{x}_{m'}^{(l)},
	\end{equation}
	where $\textbf{x}_{m}^{(0)} = {\nabla {F}}_{m}^{\mathsf{A},(k)}$ is ES $m$'s initial local aggregated gradient, and $\textbf{x}_{m}^{\big(\phi^{(k)}\big)}$ denotes the local gradient after the consensus process concludes. In~\eqref{eq:ConsCenter}, $\varrho(m)\subseteq \mathcal{M}$ denotes the set of ESs in the neighborhood of ES $m$, and $\lambda^{(k)}_{m,m'} \in [0, 1]$, $m'\in \{m\} \cup \varrho(m)$
	%   \nm{for convenience you can simply assume that $\mathcal{\zeta}^{(k)}(n)$ contains n as well} \ali{Yes. But writing as 7  may get a better perspective to the reader. Finally, we will work with the matrix form that encapsulates the neighbors in its elements.}
	are the \textit{consensus weights} employed at $m$.
	
	Let ${\nabla {F}}_m^{\mathsf{L},(k)}=\textbf{x}_{m}^{\big(\phi^{(k)}\big)}$ denote the final local gradient at ES $m$ after the P2P communication process for training round $k$ concludes, which can be expressed as
	\begin{equation}\label{eq:proof26}
	{\nabla {F}}_m^{\mathsf{L},(k)}= \underbrace{\sum_{m\in\mathcal{M}}     {\nabla {F}}_{m}^{\mathsf{A},(k)}}_{(a)} + \bm{c}^{(k)}_m,
	\end{equation}
	where term $(a)$ is the perfect average of the local aggregated gradients and $\bm{c}^{(k)}_m$ denotes the error of consensus caused by finite P2P rounds. The selected ES at aggregation round $k$, denoted by $m_k \in\mathcal{M}$, then adds a boosting coefficient to its local gradient ${\nabla {F}}_{m_k}^{\mathsf{L},(k)}$, forming the vector
	%$ \sum_{m\in \mathcal{M}} \sum_{n\in \mathcal{N}_m^{(k)}}\frac{{{D}}^{(k)}_{n}e^{(k)}_{n}}{{D}^{(k)}}$ 
	\begin{equation} \label{eq:proof-boosting}
	\overline{\nabla {F}}_{m_k}^{(k)}= \sum_{m\in \mathcal{M}} \sum_{n\in \mathcal{N}_m^{(k)}}\frac{{{D}}^{(k)}_{n}e^{(k)}_{n}}{{D}^{(k)}}{\nabla {F}}_{m_k}^{\mathsf{L},(k)},
	\end{equation}
	and updates the global model parameter as follows:
	\begin{equation} \label{equa:global_update-final}
	\boldsymbol{w}^{(k+1)}=\boldsymbol{w}^{(k)}-\eta_{k} \overline{\nabla {F}}_{m_k}^{(k)}.
	\end{equation}
	$\boldsymbol{w}^{(k+1)}$ is then broadcast across the ESs and MDs via block mining to begin the next round of local model training.  \textcolor{black}{The training procedure of our BFL  scheme is summarized in Algorithm~\ref{Al:BFL-algorithm}. At the beginning of each global communication round, each MD makes a training decision: edge model training at the ESs or local model training, which is an output of Algorithm~\ref{Al:DRL} (line~\ref{line:BFL-offloading}) that will be developed in Section~\ref{Section:DRL}. In the case of edge model training, the MD offloads its data to its associated ES (line~\ref{line:BFL-offloadingstart}); otherwise, the MD trains its data locally (lines~\ref{line:BFL-localexestart}-\ref{line:BFL-localexestop}). After receiving all models from local MDs, each ES performs edge model training (lines~\ref{line:BFL-offloadingstart1}-\ref{line:BFL-offloadingstop}). Once the local training and edge training processes are completed, ESs collaborate to perform P2P-based consensus on aggregated gradients to update the global model (lines~\ref{line:BFL-modelagg}-\ref{line:BFL-modelupdateglobal}). Finally, the mining is executed, where an ES leader adds the global model parameter to an unverified block  and broadcasts it to other ESs and MDs. In particular, each MD performs resource trading to purchase hash power from the ESP to run the Proof-of-Work-based mining. The confirmed block is then appended into the blockchain for global model sharing, where each MD downloads the block for the next round of training (lines~\ref{line:BFL-modelbroadcast}-\ref{line:BFL-modelupdateglobalfinish}). }

	\begin{table*}[t!]
		\vspace{-6mm}
		{\color{black} \footnotesize
			\begin{minipage}{0.99\textwidth}
				\begin{align}\label{eq:finalResTheoMain}
				&\hspace{-16mm}\frac{1}{K} \sum_{k=0}^{K-1} \mathbb{E} \left[\Vert \nabla F^{(k)}(\boldsymbol{w}^{({k})}) \Vert^2\right] \leq \frac{8 \sqrt{{e}^{\mathsf{max}}_{\mathsf{avg}}}}{\alpha \hat{e}^{\mathsf{min}}_{\mathsf{avg}}\sqrt{K}}
				\left({F^{(0)}(\boldsymbol{w}^{(0)}) - F^{{(k)}^\star}}\right)+\underbrace{\frac{8 \sqrt{{e}^{\mathsf{max}}_{\mathsf{avg}}}}{\alpha\hat{e}^{\mathsf{min}}_{\mathsf{avg}}\sqrt{K}}\sum_{k=1}^{K-1}{ \Delta^{(k)}}}_{(a)}
				\nonumber\\[-.5em]
				&\hspace{-6mm}+ \underbrace{\frac{80 \beta^2 \alpha^2}{K^2{e}^{\mathsf{min}}_{\mathsf{avg}}}\sum_{k=0}^{K-1}
				\sum_{m\in\mathcal{M}}
				\sum_{y\in \mathcal{N}_m^{(k)}}\frac{{D}^{(k)}_y}{D^{(k)}} { \left(e_y^{(k)}-1\right)\left(1-\frac{{B}^{(k)}_y}{{D}^{(k)}_y}\right) \frac{({\sigma}_{y}^{(k)})^2}{{B}^{(k)}_y}\Theta_y^2}}_{(b)}
				+\underbrace{\frac{80 \beta^2 \alpha^2}{K^2{e}^{\mathsf{min}}_{\mathsf{avg}}}\sum_{k=0}^{K-1} {\zeta_2 \left(e_{\mathsf{max}}^{(k)}\right)\left(e_{\mathsf{max}}^{(k)}-1\right)}}_{(c)}
				\nonumber\\[-.5em]&\hspace{-8mm}
				+ \underbrace{ \frac{1}{K}\sum_{k=0}^{K-1} 24{M} (\lambda^{(k)})^{2\varphi^{(k)}} \left(\Xi^{(k)}\right)^2}_{(d)}+ \underbrace{\frac{16 \beta \alpha \hat{e}^{\mathsf{max}}_{\mathsf{avg}}}{K\sqrt{K}\sqrt{{e}^{\mathsf{min}}_{\mathsf{avg}}}}\sum_{k=0}^{K-1}
				\sum_{m\in\mathcal{M}}
				\sum_{y\in \mathcal{N}_m^{(k)}}\left(\frac{{D}^{(k)}_y}{{D}^{(k)} \sqrt{e^{(k)}_y}}\right)^2
				\left(1-\frac{{B}^{(k)}_y}{{D}^{(k)}_y}\right) \frac{({\sigma}_{y}^{(k)})^2}{{B}^{(k)}_y} \Theta^2_y}_{(e)}.\hspace{-6mm}
				\end{align}
			\end{minipage}
		}
		\hrule
\vspace{-5mm}
	\end{table*}

\subsection{Convergence Analysis of ML Model Training}
\label{Convergence_BFL}

%	 Let $\mathcal{M}^{(k)}\subseteq \mathcal{M}$ with size $m^{(k)}=|\mathcal{M}^{(k)}|$ denote the set of ESs engaged in data processing or parameter acquisition from their subsequent users  during local model training round $k$.
%	Also, assume that the MDs associated with ES $m$ during global aggregation $k$ constitute the set $\mathcal{N}_m^{(k)}$, where for simplicity of notations $\mathcal{N}_m^{(k)}$ also contains the ES $m$ itself. 

% Given the loss function defined in~\ref{loss_device}, we inclusively analyze the model update procedure at MDs/ESs, and then characterize the convergence properties of our BFL system. 

We now study the convergence of our BFL scheme. We first make a few assumptions and define some quantities of interest. To obtain convergence guarantees for the distributed consensus process, we make the following assumptions on the consensus weights in~\eqref{eq:ConsCenter}:
% there are several potential choices for the consensus weights in~\eqref{eq:ConsCenter} which are elaborated as follows.
	\begin{assumption}[Conditions on Consensus Weights~\cite{xiao2004fast}, \cite{xiao2007distributed}]\label{assump:cons}
		The consensus matrix $\Lambda^{(k)}=[\lambda^{(k)}_{m,m'}]_{m,m'\in\mathcal{M}^{(k)}}$ satisfies the following properties: (i) $\lambda^{(k)}_{m,m'}=0~~\textrm{if}$ ESs $m$ and $m'$ are not connected, (ii) $\Lambda^{(k)}\textbf{1} = \textbf{1}$, (iii) $\Lambda^{(k)} = \Lambda^{(k),\top}$, and (iv)~$ \rho \left(\Lambda^{(k)}-\frac{\textbf{1} \textbf{1}^\top}{|\mathcal{M}^{(k)}|}\right) \leq \lambda^{(k)} < 1$, where $\textbf{1}$ represents the 1s's vector  and $\rho(\mathbf{A})$ defines $\mathbf{A}$'s spectral radius.
	\end{assumption}
	
	If
	$\lambda^{(k)}_{m,m}= 1- d |\varrho(m)|$ and $\lambda^{(k)}_{m,m'}= d $, $m\neq m'$, $0 < d < 1 / M$~\cite{xiao2004fast}, the conditions in Assumption~\ref{assump:cons} hold. These weights can be distributedly obtained at the ESs given a predefined $d$.
	
% 	\textcolor{black}{After performing $\phi$ rounds of P2P communications at the aggregation instance $k$, the consensus error at each ES $m$ satisfies 	
% 	\begin{multline}
% 	    \big\Vert\bm{e}^{(k)}_m\big\Vert \leq   (\lambda^{(k)})^{\varphi^{(k)}} \sqrt{M} \max_{m',m''\in\mathcal{M}} \Vert \widetilde{\nabla {F}}_{m'}^{(k)} -{\nabla {F}}_{m''}^{(k)} \big\Vert.
% 	\end{multline}
%     \begin{proof}
% 		See Appendix~\ref{app:th}.
% 	\end{proof}}
	
	\begin{definition}[Gradient Divergence]\label{def:clustDiv}
		The divergence of local aggregated gradients across the ESs at global aggregation round $k$, denoted by $\Xi^{(k)}$, is defined as follows:
		\begin{equation}
		\left\Vert{\nabla {F}}_{m}^{\mathsf{A},(k)} -{\nabla {F}}_{m'}^{\mathsf{A},(k)} \right\Vert \leq \Xi^{(k)} , ~\forall {m},{m'}\in\mathcal{M}.
		\end{equation} 
	\end{definition}
	\begin{assumption}[\textcolor{black}{Smoothness of the Loss Functions \cite{dinh2020federated, wang2020tackling}}]\label{Assup:lossFun} 
		For each MD $n$ that conducts local model training during aggregation round $k$, the local loss function $F^{({k})}_n$  is  $\beta$-smooth:
		\begin{equation}
		\Vert \nabla F^{({k})}_n(\boldsymbol{w})- \nabla F^{({k})}_n(\boldsymbol{w}') \Vert \leq \beta \Vert \boldsymbol{w}-\boldsymbol{w}' \Vert,~ \hspace{-.8mm}\forall \boldsymbol{w},\boldsymbol{w}'. \hspace{-2.3mm} 
		\end{equation}
		Also, for each ES $m$ that conducts model training, the loss function $F^{({k})}_m$ is assumed to be $\beta$-smooth,
		which verifies that the global loss function $F^{(k)}$ achieves $\beta$-smoothness. 
		%   \begin{equation}
		%     \Vert \nabla f_i(\boldsymbol{w};x)- \nabla f_i(\boldsymbol{w}';x) \Vert \leq L \Vert \boldsymbol{w}-\boldsymbol{w}' \Vert,~\forall \boldsymbol{w},\boldsymbol{w}',x.
		% \end{equation}
	\end{assumption}

	\begin{table*}[t!]
% 		\vspace{-6mm}
		{\color{black} \footnotesize
			\begin{minipage}{0.99\textwidth}
				\begin{align}\label{eq:coRes}
				&\hspace{-16mm}\frac{1}{K} \sum_{k=0}^{K-1} \mathbb{E} \left[\Vert \nabla F^{(k)}(\boldsymbol{w}^{({k})}) \Vert^2\right] \leq \frac{8 \sqrt{{e}^{\mathsf{max}}_{\mathsf{avg}}}}{\alpha \hat{e}^{\mathsf{min}}_{\mathsf{avg}}\sqrt{K}}
				\left({F^{(0)}(\boldsymbol{w}^{(0)}) - F^{{(k)}^\star}}\right)+\frac{8\Upsilon \sqrt{{e}^{\mathsf{max}}_{\mathsf{avg}}}}{\alpha\hat{e}^{\mathsf{min}}_{\mathsf{avg}}\sqrt{K}}
				\nonumber\\&\hspace{-6mm}+\frac{80 \beta^2 \alpha^2}{K{e}^{\mathsf{min}}_{\mathsf{avg}}}
				\left(e_{\mathsf{max}}-1\right)\vartheta
				+\frac{80 \beta^2 \alpha^2}{K{e}^{\mathsf{min}}_{\mathsf{avg}}} {\zeta_2 \left(e_{\mathsf{max}}\right)\left(e_{\mathsf{max}}-1\right)}
				+ \frac{24\xi}{\sqrt{K}}+\frac{16 \beta \alpha \hat{e}^{\mathsf{max}}_{\mathsf{avg}}}{\sqrt{K}\sqrt{{e}^{\mathsf{min}}_{\mathsf{avg}}}}\vartheta
				\end{align}
			\end{minipage}
		}
		\hrule
		\vspace{-5mm}
	\end{table*}
	
	 Let $y\in\mathcal{N}\cup \mathcal{M}$ denote the index of an arbitrary MD/ES. Also, let $\mathcal{N}_m^{(k)}= \mathcal{N}_m^{\mathsf{loc},(k)}\cup \{m\}$ denote a set containing the devices which conduct local model training and offload their cumulative gradients to ES $m$ as well as ES $m$ itself. \textcolor{black}{We measure the heterogeneity of data across the MDs/ESs via the following assumption \cite{wang2020tackling}}:
	\begin{assumption}[Bounded Dissimilarity of Local Loss Functions]\label{Assup:Dissimilarity}
		The finite constants $\zeta_1 \geq 1$, $\zeta_2 \geq 0$ exist for which the following inequality defined on the gradient of the local loss in \eqref{loss_device} and \eqref{loss_server} holds for any set of coefficients $\{a_y\geq 0\}$ where $\sum_{m\in \mathcal{M}} \sum_{y\in \mathcal{N}^{(k)}_m} a_y=1$:
		%   For any set of coefficients $\{a_n\}_{n\in \mathcal{N}}$, where $\sum_{n\in \mathcal{N}} a_n=1$, there exist two finite constants $\zeta_1 \geq 1$, $\zeta_2 \geq 0$, such that
		\begin{multline}
		\sum_{m\in \mathcal{M}} \sum_{y\in \mathcal{N}^{(k)}_m} a_y \Vert \nabla F^{({k})}_y(\boldsymbol{w}) \Vert^2  \\ \leq  \zeta_1  \Big\Vert   \sum_{m\in \mathcal{M}} \sum_{y\in \mathcal{N}^{(k)}_m}a_y \nabla F^{({k})}_y(\boldsymbol{w}) \Big\Vert^2 +\zeta_2,~\forall k.\label{eq:BoundDis}
		\vspace{-1mm}
		\end{multline}
	\end{assumption}
	As can be seen, $\zeta_1=1$ and $\zeta_2=0$ in~\eqref{eq:BoundDis} correspond to a scenario in which the data across the MDs/ESs is completely homogeneous, and $\zeta_1$ and $\zeta_2$ increase as the heterogeneity across the datasets increases. We next quantify the heterogeneity of data inside each local dataset:
	\begin{definition}[Local Data Variability]\label{Assump:DataVariabilit}
		The local data variability at each MD/ES $y$ is denoted by $\Theta_y\geq 0$, which $\forall \boldsymbol{w} , k$ satisfies
		\begin{equation}
		\hspace{-5mm}
		\resizebox{.93\linewidth}{!}{$
			\Vert \nabla f_y\hspace{-.4mm}(\boldsymbol{w},d) \hspace{-.4mm}-\hspace{-.4mm} \nabla f_y\hspace{-.4mm}(\boldsymbol{w},d')\Vert  \hspace{-.4mm} \leq \hspace{-.4mm} \Theta_y \Vert d\hspace{-.4mm}-\hspace{-.4mm}d' \Vert,\hspace{-1.2mm}~\forall d,d'\hspace{-.7mm}\in\hspace{-.7mm}\mathcal{D}^{(k)}_n\hspace{-.9mm}.$}\hspace{-5mm}
		\end{equation}
		We further define $\Theta = \max_{y\in\mathcal{N}\cup \mathcal{M}}\{\Theta_y\}$.
	\end{definition}
	
	 Additionally, we let $({\sigma}_{y}^{(k)})^2$ denote the variance across the feature vectors of dataset $\mathcal{D}_y^{(k)}$ at ES/MD $y$ which conducts local model training. We further quantify the dynamics of local datasets via their impact on the ML performance~\cite{hosseinalipour2022parallel}:
	\begin{definition} [Model/Concept Drift]\label{def:conceptdrift}
% 		Let $D_n(t)$ and $D(t)$ denote the instantaneous number of datapoints at MD/ES $n$ and total number of data points at wall clock time $t$ (measured in seconds). 
		For each MD/ES $y$, we calculate the model/concept drift between two FL rounds $k-1$ and $k$, $\Delta_y^{(k)}\in\mathbb{R}$, which characterizes the local loss's variation induced by data arrival/departure at the MDs and data collection at the ESs, as follows:
		\begin{equation}\label{eq:conceptDrift}
% 		\resizebox{.93\linewidth}{!}{$
			% \begin{aligned}
			\frac{D^{(k)}_y}{D^{(k)}} F^{(k)}_y \left(\boldsymbol w\right)- \frac{D_y^{(k-1)}}{D^{(k-1)}}F^{(k-1)}_y   \left(\boldsymbol w\right)\leq \Delta_y^{(k)},  \forall \boldsymbol{w}.
			%  \end{aligned}
% 			$}
		\hspace{-1mm}
		\end{equation}
	\end{definition}
% 	Let $\delta^{\mathsf{Agg},(k)}$ denote the delay of P2P aggregation and global model broadcast during which the MDs do not conduct ML model training. 

	We next present one of our main results, the general convergence behavior of BFL:
	
	\begin{theorem}[Convergence Characteristics]\label{th:main}
		Let $\mathcal{N}^{(k)}$, ${N}^{(k)}=|\cup_{m\in\mathcal{M}} \mathcal{N}_m^{(k)}|$ denote the set of all MDs and ESs engaged in model training during the $k$-th local model training round. Also, let $e^{(k)}_{\mathsf{max}}=\max_{y\in\mathcal{N}^{(k)}}\{e^{(k)}_y\}$, ${e}^{(k)}_{\mathsf{avg}}=\sum_{y\in \mathcal{N}^{(k)}}e^{(k)}_{y}/{N^{(k)}}$, and  $\hat{e}^{(k)}_{\mathsf{avg}}=\sum_{y\in \mathcal{N}^{(k)}}{{D}^{(k)}_{y}e^{(k)}_{y}}/{D^{(k)} }$. Assume that  
		${e}^{\mathsf{min}}_{\mathsf{avg}} \leq {e}^{(k)}_{\mathsf{avg}}\leq  {e}^{\mathsf{max}}_{\mathsf{avg}}$ and
		$\hat{e}^{\mathsf{min}}_{\mathsf{avg}} \leq \hat{e}^{(k)}_{\mathsf{avg}}\leq  \hat{e}^{\mathsf{max}}_{\mathsf{avg}}$, $\forall k$, where ${e}^{\mathsf{min}}_{\mathsf{avg}}$, ${e}^{\mathsf{max}}_{\mathsf{avg}}$,
		$\hat{e}^{\mathsf{min}}_{\mathsf{avg}}$ and  $\hat{e}^{\mathsf{max}}_{\mathsf{avg}}$ are four finite positive constants. Further, let $\Delta^{(k)}=\sum_{m\in\mathcal{N}}\sum_{y\in \mathcal{N}^{(k)}_m}{\Delta}_y^{(k)}$. 
		In conducting $K$ global aggregation rounds, if the step size satisfies $\eta_k =\frac{\alpha}{{\sqrt{K e^{(k)}_{\mathsf{avg}}}}}$, where $\alpha$ is chosen such that $\eta_k \leq \min \Big\{\frac{1}{2\beta} \sqrt{ \frac{1}{(4\zeta_1+1)\left( e_{\mathsf{max}}^{(k)}\left(e_{\mathsf{max}}^{(k)}-1\right)\right)}}, \frac{{D}^{(k)}}{2\beta\sum_{y\in \mathcal{N}}{{D}}^{(k)}_{y}e^{(k)}_{y}}\Big\}$, then
		the convergence characteristics of the global loss function under BFL follow the bound in~\eqref{eq:finalResTheoMain}.
	\end{theorem}
	\begin{proof}
		See Appendix~\ref{subsection:Appen-Theorem1}. % in the technical report~\cite{tech}.
		%See Appendix~A in the technical report~\cite{tech}.
	\end{proof}

		The bound in~\eqref{eq:finalResTheoMain} reveals the impact of different device/network configurations on the ML model performance. In particular, the impact of model drift is captured in term $(a)$. Also, the divergence of the global model caused by bias of the local models due to the heterogeneity of data across the MDs/ESs is captured via term $(c)$ which encapsulates $\zeta_2$. Larger daatset heterogeneity captured via $\zeta_1$ imposes a stricter condition on the step size $\eta_k$ described in the statement of the theorem as well. Terms $(b)$ and $(e)$ in~\eqref{eq:finalResTheoMain} capture the impact of local data heterogeneity ($\{\Theta_y\}$) and mini-batch sizes ($\{B^{(k)}_y\}$) on the performance of the model. Finally, term $(d)$ captures the impact of imperfect local aggregations caused by finite P2P rounds (${\phi^{(k)}}$) and divergence of gradients $\Xi^{(k)}$, where  the consensus matrix's spectral radius across the ESs $(\lambda^{(k)}<1)$ defined in Assumption~\ref{assump:cons} determines the rate under which the consensus error decays to zero with respect to the number of P2P rounds $\phi^{(k)}$.

	\textcolor{black}{We derive Theorem \ref{th:main} to obtain conditions for the online global gradient under BFL converges:}
	\begin{corollary}[Convergence under Proper Choice of Mini-batch Size and P2P Communication Rounds]\label{co:main} 
		Besides conditions in Theorem~\ref{th:main}, we assume that (i) the model/concept drift is small enough such that $\Delta^{(k)}\leq \frac{\Upsilon}{K}$, $\forall k$, for some positive constant $\Upsilon$, (ii) the choice of mini-batch size ${B}^{(k)}_y$ at each MD/ES $y$ ensures a unified bound
		$\left(1-\frac{{B}^{(k)}_y}{{D}^{(k)}_y}\right) \frac{({\sigma}_{y}^{(k)})^2}{{B}^{(k)}_y}\Theta_y^2\leq \vartheta$, $\forall k$, where $\vartheta$ is a finite positive constant, and (iii) $e_{\mathsf{max}}^{(k)} \leq e_{\mathsf{max}}$, $\forall k$ for some positive constant $e_{\mathsf{max}}$. If the number of P2P communications  among the ESs at each global aggregation round $k$, i.e., $\phi^{(k)}$, satisfies ${\varphi^{(k)}} \geq \frac{1}{2} \left[\log_{\lambda^{(k)}}\left( \frac{\xi}{\sqrt{K}\left(\Xi^{(k)}\right)^2{m}^{(k)}}\right)\right]^+$ for some positive constant $\xi$, where $\lambda^{(k)}$ is defined in Assumption~\ref{assump:cons}, the gradient of the global loss under BFL satisfies the upper bound in~\eqref{eq:coRes}, which implies $\frac{1}{K} \sum_{k=0}^{K-1} \mathbb{E} \left[\Vert \nabla F^{(k)}(\boldsymbol{w}^{({k})}) \Vert^2\right] = \mathcal{O}\left(\frac{1}{\sqrt{K}}\right)$, and thus $\lim_{K\rightarrow \infty}\frac{1}{K} \sum_{k=0}^{K-1} \mathbb{E} \left[\Vert \nabla F^{(k)}(\boldsymbol{w}^{({k})}) \Vert^2\right]\rightarrow 0$.
	\end{corollary}
	\begin{proof}
		See Appendix~\ref{subsection:Appen-Corollary}. % in the technical report~\cite{tech}.
		%See Appendix~B in the technical report~\cite{tech}.
	\end{proof}

\begin{algorithm}
\footnotesize
	\caption{\textcolor{black}{Proposed BFL algorithm}}
	\textcolor{black}{\begin{algorithmic}[1]
		\label{Al:BFL-algorithm}
		\STATE \textbf{Input:}  Global communication rounds $\mathcal{K}$, local training round $e_n, \forall n\in \mathcal{N}$, edge training round $e_m, \forall m\in \mathcal{M}$, number of MDs $\mathcal{N}$, number of ESs $\mathcal{M}$ 
		\STATE \textbf{Initialization:} Initialize global model $\boldsymbol{w}^{(0)}$
		\FOR{each global communication round $k \in \mathcal{K}$}
		\FOR{each MD $n \in \mathcal{N}$}
		\STATE Determine each training decision $x_{n,m,g}$ via Algorithm~\ref{Al:DRL}  \label{line:BFL-offloading}
		\IF{$x_{n,m,g}$ = 1}
		\STATE Offload dataset $D_n^{(k)}$ to  ES $m$ \label{line:BFL-offloadingstart}
    	\ELSE
    	\STATE Perform local model training on local dataset $D_n^{(k)}$ \label{line:BFL-localexestart}
		\FOR{each local training epoch $e \in e_n$}
		\STATE Update local parameters $\boldsymbol{w}^{(k),e}_n$ via~\eqref{eq:WeightupdateMD}
    	\ENDFOR
    	\STATE Compute the cumulative gradient via~\eqref{eq:cumulativeMD}   
    	\STATE Add the gradient to a transaction based on the  blockchain framework defined in~Fig.~\ref{Fig:Blockchain_Arch} to transfer to an ES \label{line:BFL-localexestop} 
		\ENDIF
		\ENDFOR
		\FOR{each ES $m \in \mathcal{M}$} \label{line:BFL-offloadingstart1}
		\FOR{each edge training epoch $e \in e_m$}
		\STATE Update local parameters $\boldsymbol{w}_m^{(k),e}$ via~\eqref{eq:Weightupdate}
    	\ENDFOR
    	\STATE Compute its cumulative gradient via~\eqref{eq:cumulativeES} \label{line:BFL-offloadingstop}
    	\ENDFOR
		\STATE Each ES $m\in\mathcal{M}$ computes an aggregated gradient: ${\nabla {F}}^{\mathsf{A},(k)}_{m}$ using $\nabla {F}_m^{\mathsf{C},(k)}$ and $\nabla {F}_m^{\mathsf{C},(k)}$ via~\eqref{equation:modelconsensus}  \label{line:BFL-modelagg}
		\STATE Set  $\textbf{x}_{m}^{(0)} = {\nabla {F}}_{m}^{\mathsf{A},(k)}$
	    \FOR{each P2P consensus round $l \in \phi^{(k)} $}
	    \FOR{each ES $m \in \mathcal{M}$} 
	    \FOR{each neighboring ES $m'\in \varrho(m)$ of ES $m$}
	    \STATE Transmit the gradient $\textbf{x}_{m'}^{(l)}$ via a transaction based on the blockchain framework defined in~Fig.~\ref{Fig:Blockchain_Arch} to ES $m$
	    \ENDFOR
	    \STATE ES $m$ downloads the block of all transactions to obtain  neighboring ESs' gradients and computes its gradient $\textbf{x}_{m}^{(l+1)}$ via~\eqref{eq:ConsCenter} 
	    \ENDFOR
	    \ENDFOR
	    \STATE Obtain the final local gradient at ES $m$: ${\nabla {F}}_m^{\mathsf{L},(k)}=\textbf{x}_{m}^{\big(\phi^{(k)}\big)}$  
	    \STATE Add a boosting coefficient to ${\nabla {F}}_m^{\mathsf{L},(k)}$ to form $\overline{\nabla {F}}_{m_k}^{(k)}$ via~\eqref{eq:proof-boosting}
	    \STATE Update the global model parameter $\boldsymbol{w}^{(k+1)}$ via~\ref{equa:global_update-final} \label{line:BFL-modelupdateglobal}
		\STATE Add the global model parameter to an unverified block $B$ by the ES leader and broadcast it to other ESs and MDs for mining  \label{line:BFL-modelbroadcast}
		\STATE Each MD trades hash resource $ \Psi_n^{(k)}$ from the ESP to mine the block
		\STATE The fastest MD propagates the verified block to ESs and other MDs for confirmation
		\STATE Add the block to the blockchain for global model sharing
		\STATE Each MD downloads the block to obtain the global model for next round of training \label{line:BFL-modelupdateglobalfinish}
		\ENDFOR
		\STATE \textbf{Output:} Final global model $\boldsymbol{w}^{(K)}$
	\end{algorithmic}}
\end{algorithm}

\section{System Latency Model}
\label{Section:Latency}
\textcolor{black}{In this section, we analyze the latencies of the model training and block mining processes in detail, and present our latency optimization problem.}
\subsection{Latency of Model Training in BFL}
\textcolor{black}{In our BFL system, an MD can offload its data to the ES for edge learning or choose to train the ML model locally in each global aggregation $k\in \mathcal{K}$.} \textcolor{black}{In the case of offloading ($x_{n,m,g}^{(k)} = 1$), the latency for model training consists of data communication latency and execution latency at the ES. The data communication latency of  MD $n$ when offloading its data $D_n^{(k)}$ to ES $m$ is given by}
\begin{equation} 
T_{n,m}^{\mathsf{off},{(k)}} =  \sum_{g \in \mathcal{G}} x^{(k)}_{n,m,g} \frac{D_n^{(k)}}{R_{n,m}^{(k)}}, \forall n \in \mathcal{N}, 
\end{equation} 
where $R_{n,m}^{(k)}$ is the transmission rate (in bits/s) from MD $n$ to ES $m$, which is given by $R_{n,m}^{(k)} = \sum_{g \in \mathcal{G}} x^{(k)}_{n,m,g}b_{n,g}^{(k)} \log_2 \left(1+\frac{p_n^{(k)} h_{n,m,g}^{(k)}}{\sigma^2 + \sum_{j \in  \mathcal{N}} \setminus \mathcal{N}_m^{\mathsf{off}}\left(x_{j,m,g}^{(k)} p_{j}^{(k)} h_{j,m,g}^{(k)}\right)}\right)$, under interference $\sum_{j \in  \mathcal{N}} \setminus \mathcal{N}_m^{\mathsf{off}}\left(x_{j,m,g}^{(k)} p_{j}^{(k)} h_{j,m,g}^{(k)}\right)$ caused by a group of MDs ($\mathcal{N} \setminus \mathcal{N}_m^{\mathsf{off}}$) that are associated with other ESs on sub-channel $g$. Here, $b_{n,g}^{(k)}$ is the bandwidth (in Hz) allocated to channel $g$ under the policy $\boldsymbol{B}^{(k)} = \{ b_{n,g}^{(k)} | 0 < b_{n,g}^{(k)} \leq W, \forall n \in \mathcal{N}, g \in \mathcal{G}, k \in \mathcal{K}  \}$, where $W$ is the maximum system bandwidth constraint. Further, $p_n^{(k)}$ is transmission power (in Watts) of MD $n$ in the offloading subject to the power constraint $P_n$ under a policy $\boldsymbol{P}^{(k)} = \{p_n^{(k)} | 0 < p_n^{(k)} \leq P_n$, $n \in \mathcal{N}_m^{\mathsf{off}}$, $\forall m \in \mathcal{M}\}$, and $h_{n,m,g}^{(k)}$ is the  gain of wireless channel between the MD $m$ and ES $n$ on sub-channel $g$.  Accordingly, the energy consumption for data offloading at aggregation round $k$ at MD $n$ is given as $E_{n,m}^{\mathsf{off}(k)} =  \sum_{g \in \mathcal{G}} x^{(k)}_{n,m,g} p_n^{(k)}T_{n,m}^{\mathsf{off},{(k)}}$.

\textcolor{black}{After the offloading process,  ES $m$ receives a combined dataset $\mathcal{D}_m^{(k)}$ (as defined in Section~\ref{subsection:BFL_model}) with size $D_m^{(k)} =|\mathcal{D}_m^{(k)}|$ from MDs engaged in data offloading. Then, the ES allocates its computational resource to perform the SGD-based model training. The data processing latency at ES $m$ is thus given by
\begin{equation}
 T^{\mathsf{exe},(k)}_m =  \sum_{g \in \mathcal{G}} x_{n,m,g}^{(k)} \frac{C_m D_m^{(k)}\varrho^{(k)}_m e_m^{(k)}}{f_m},  m \in \mathcal{M},  
\end{equation}
where $C_m$ represents how many CPU cycles is required to calculate the gradient per data point at ES $m$, and $\varrho_m^{(k)} \in (0,1]$ is the mini-batch size ratio, from which the  size of SGD mini-batches can be written as $B_m^{(k)}=\varrho^{(k)}_m D^{(k)}_m$.} Moreover, $e_m^{(k)}$ is the number of SGD iterations, and $f_m$ is the fixed computational capability of ES $m$ (in CPU cycles/s).
%which is much higher and more stable compared with MDs \cite{nguyen2021cooperative}, thus we do not model its allocation.}
%Therefore, the total learning latency of MD $n$ when choosing the offloading mode can be given as

On the other hand, each MD $n$ can also choose to locally process its data, implying $x_{n,m,g}^{(k)} = 0, \forall m,g$. We denote $f^{\mathsf{\ell},(k)}_n$ as the computational capability of MD $n$ (in CPU cycles/s) allocated to train the data given maximum capacity $F_n$, which is represented via a policy:  $\boldsymbol{F}^{(k)} = \{ f_n^{\mathsf{\ell},(k)} | 0 < f_n^{\mathsf{\ell},(k)} \leq F_n, \forall n \in \mathcal{N} \}$. The latency of model training over $e_n^{(k)}$ SGD iterations at MD $n$ is given by
\begin{equation}
	T_n^{\mathsf{loc},(k)} = \left(1-\sum_{g \in \mathcal{G}} x^{(k)}_{n,m,g} \right) \frac{C_nD_n^{(k)}\varrho^{(k)}_n e_n^{(k)}}{f_n^{\mathsf{\ell},(k)}},
\end{equation} 
where $\varrho^{(k)}_n \in (0,1]$ is the SGD mini-batch ratio, and $C_n$ is  the number of CPU cycles required to compute the gradient per data point at MD $n$. Also, the local energy consumption of MD $n$ is given by $E_n^{\mathsf{loc},(k)} =  \left(1-\sum_{g \in \mathcal{G}} x^{(k)}_{n,m,g} \right) \kappa (f_n^{\mathsf{\ell},(k)})^2 C_n$,  where $\kappa$ is the energy coefficient depending on the chip architecture. \textcolor{black}{After completing the local model training,  MD $n$ then transmits the computed model ${\nabla {F}}_n^{\mathsf{A},(k)}$ to its associated ES via blockchain.} The latency of this model uploading can be given as
\begin{equation} 
T_{n,m}^{\mathsf{up},(k)} =  \left(1-\sum_{g \in \mathcal{G}} x^{(k)}_{n,m,g} \right) \frac{\vartheta}{R_{n,m}^{(k)}}, \forall n \in \mathcal{N}, m \in \mathcal{M}, 
\end{equation}
where $\vartheta$ is the  gradient/model size (in bits) that is the same across the MDs.  \textcolor{black}{Based on above formulations, the total latency of model training of the BFL system at each global aggregation $k$ is  
\begin{multline}
T^{\mathsf{learn},(k)} = \sum_{n \in \mathcal{N}}\sum_{m \in \mathcal{M}} T_{n,m}^{\mathsf{off},(k)} + \sum_{m \in  \mathcal{M}}T^{\mathsf{exe},(k)}_m \\ + \sum_{n \in \mathcal{N}} T_n^{\mathsf{loc},(k)} +\sum_{n \in \mathcal{N}}\sum_{m \in \mathcal{M}} T_{n,m}^{\mathsf{up},(k)}, \forall k \in \mathcal{K}.
\end{multline} }

After the model training process, the ESs in the edge layer perform the P2P-based consensus on the computed model. An ES performs the model updating where the neighbours exchange their aggregated gradients defined in~\eqref{equation:modelconsensus} via P2P communications. The latency of parameter updating at ES $m$ in each P2P communication round $l$ is determined by 
\begin{equation}
T_m^{\mathsf{update},{(k)}}(l) = \max_{m' \in \Xi(m)} \frac{\vartheta}{R^{(k)}_{m,m'}(l)},
\end{equation}
\textcolor{black}{where  $R^{(k)}_{m,m'}(l)$ is the transmission rate from  ES $m'$ to ES $m$ via wired communications.
% as $R_{m'}^m = b_{m'} \log_2  \left( 1 + \frac{p_{m'} c_{m'}^m}{N_0 b_{m'}} \right)$, where $b_{m'}$ is the assigned bandwidth of ES ${m'}$, $p_{m'}$ is the transmit power, $c_{m'}^m$ is the channel gain of the ES $m'$ to ES $m$, and $N_0$ is the power spectral density of the
% Gaussian noise.
Thus the updating latency of each ES $m$ over $ \varphi$ consensus rounds is $T^{\mathsf{cons},{(k)}}_m = \sum_{l=1}^{\varphi} T_m^{\mathsf{update},{(k)}}(l)$. Finally, the latency of the model consensus is determined by the slowest ES as below:
\begin{equation}
\hspace{-3mm} T^{\mathsf{cons},(k)} = \max_{m \in \mathcal{M}} T^{\mathsf{cons},(k)}_m  = \max_{m \in \mathcal{M}} \left[ \sum_{l=1}^{\varphi} T_m^{\mathsf{update},{(k)}}(l) \right].\hspace{-2mm} 
\end{equation}}

\subsection{Latency of Block Mining in BFL}
\label{subsection:mining} 
After model consensus among ESs, an ES will be selected to work as a  leader that builds an unverified block $B$ and broadcasts it to all connected ESs and participating MDs for universal block mining. \textcolor{black}{This selection can be based on the ESs' reputation  in the previous aggregation round from the P2P collaboration process, where the reputation evaluation methodology can be used to quantify each ES's contribution to block generation \cite{huang2021zkrep}. Inspired by our previous work \cite{nguyen2021cooperative}, we adopt mining latency as the reputation metric: an ES which generates the fastest block in the previous mining round will have the highest reputation and is selected as the leader for mining coordination in the current mining round.}

\textcolor{black}{After leader selection, the mining process is executed. Similar to existing works \cite{9,10}, we focus on \textit{mobile mining} analysis, where the mining latency at MDs is considered, and thus the mining analysis at ESs is ignored.} \textcolor{black}{We adopt the popular Proof-of-Work mining mechanism \cite{7,8} for our BFL, where MDs compete to mine the block}. The adoption of other mining mechanisms in BFL will be considered in future works. Conceptually, the mining latency at each MD $n$ mainly consists of block generation latency and block propagation latency \cite{25}. \textcolor{black}{Due to resource constraints, we allow MDs to implement resource trading to buy hash power from the ESP to assist their mining.} \textcolor{black}{Accordingly, MDs compete with each other to gain the maximum hash power allocation from the ESP's hash resource pool, aiming to increase the probability of becoming the mining winner for gaining rewards.} \textcolor{black}{In every global aggregation round $k$, each MD $n$ specifies its mining  demand to trade hash resource, denoted as $ \Psi_n^{(k)}$ (Hash/sec) subject to the constraint of the total hash power of the BFL system $\Psi_{\mathsf{max}}$ under a policy $\bm{\Psi}^{(k)}= \{ \Psi_n^{(k)}| 0 < \Psi_n^{(k)} \leq \Psi_{\mathsf{max}}, \forall n \in \mathcal{N} \}$. Let  $\hbar^{(k)}$ denote the hash amount (in hash) required to mine the block $B$ (which also represents the size) in aggregation round $k$, the block generation latency at MD $n$ can be specified as $T_n^{\mathsf{gen},(k)}= \frac{\hbar^{(k)}}{\Psi_n^{(k)}}$. Thus, the energy consumption for block generation at MD $n$ can be given as $E_n^{\mathsf{gen},(k)} = \Xi_n \hbar^{(k)}$, where $\Xi_n$ is the power efficiency of the mining rig of MD $n$ (J/hash). }

After generating a block, MD $n$ will propagate it to ESs and other MDs for confirmation. The latency of this block propagation process can be determined as $T_n^{\mathsf{prop},(k)} = \xi (B|L^{M+N-1}|)^{(k)}$, 
where $\xi$ is a parameter to quantify the efficiency of block verification. Further, $B|L^{M+N-1}|$ represents the average delay of the repeated verification on the block $B$ of all  entities except MD $n$ \cite{nguyen2021cooperative}. Therefore, the mining latency at an MD can be given as $T_n^{\mathsf{mine},(k)} = T_n^{\mathsf{gen},(k)} + T_n^{\mathsf{prop},(k)}$.

However, in the block mining process, there is the possibility that a MU $n$ generates the block and propagates it slower than other miners which will discard this block from  blockchain. This issue is called forking and such a block is called an orphaned one. The forking probability can be determined as $P_{\mathsf{fork}} = 1- e^{-\iota\upphi(s_n)}$, where $\iota$ is set to $\iota =1/600(sec)$ \cite{nguyen2020privacy}. Moreover, $s_n$ represents how many transactions are included in the block mined by MD $n$, and $\phi (s_n)$ is a function of block size. Therefore, the mining latency at MD $n$ can be  rewritten as
\begin{equation}
T_n^{\mathsf{mine},(k)} = \zeta(T_n^{\mathsf{gen},(k)} + T_n^{\mathsf{prop},(k)}), \forall k \in \mathcal{K}, 
\end{equation}
where $\zeta$ is the number of forking occurrences in each global training round. Therefore, the total mining latency of BFL system can be given as $T^{\mathsf{mine},(k)} = \sum_{n \in \mathcal{N}} T_n^{\mathsf{mine},(k)}$.

% which can be modelled as a geometric distribution with mean $1/(1-P_{fork})$. 
% Finally, the MD that is the first to successfully mine the block will be received a reward such as coins. However, in this paper we only focus on mining latency and thus ignore reward analytics. 

\subsection{Formulation of System Latency Problem }
Our objective is to optimize the total latency of the BFL system from the user perspective as the sum of model training latency, model consensus latency and block mining latency at a certain aggregation round $k$:
\begin{subequations}
	\label{Equa:Optimization}
	\begin{align} 
	&\hspace{-2mm}  \underset{\boldsymbol{X},\boldsymbol{P}, \boldsymbol{B}, \boldsymbol{F},\bm{\Psi}}{\minimize} 
	&& \frac{1}{K} \sum_{k=0}^{K-1} \left(T^{\mathsf{learn},(k)} + T^{\mathsf{cons},(k)}  + T^{\mathsf{mine},(k)} \right)\hspace{-2mm}  \label{Objective}\\
	& \hspace{-5mm} \phantom{P1} \text{s.t.} 
	&& \hspace{-7mm} x_{n,m,g}^{(k)} \in \{0,1\}, \forall n \in \mathcal{N}, m \in \mathcal{M}, g \in \mathcal{G}, \label{constraint1} \\
	&&& \hspace{-7mm}  \sum_{m \in \mathcal{M}} \sum_{g \in \mathcal{G}} x_{n,m,g}^{(k)} \leq 1, n \in \mathcal{N}, \label{constraint2}\\
	&&&\hspace{-7mm}  0 < p_n^{(k)} \leq P_n , \forall n \in \mathcal{N}_n, \label{constraint3}\\
	&&&\hspace{-7mm}  0 < b_{n,g}^{(k)} \leq W, \forall n \in \mathcal{N}, g \in \mathcal{G}, \label{constraint31}\\
	&&&\hspace{-7mm}  0 < f_n^{\mathsf{\ell},(k)} \leq F_n, \forall n \in \mathcal{N},  \label{constraint4}\\
	&&&\hspace{-7mm}  0 < \Psi_n^{(k)} \leq H, \forall n \in \mathcal{N}_n, \label{constraint5}\\
	&&&\hspace{-7mm}  0 < E_n^{\mathsf{learn},(k)} +  E_n^{\mathsf{gen},(k)}  \leq E_n^{\mathsf{max},(k)}, \forall n \in \mathcal{N}_n, \label{constraint51}\hspace{-1mm} \\
	&&& \hspace{-7mm} \frac{1}{K} \sum_{k=0}^{K-1} \mathbb{E} \left[\Vert \nabla F^{(k)}(\boldsymbol{w}^{({k})}) \Vert^2\right] \leq \epsilon, \label{constraint6}
	\end{align}
\end{subequations}
where $E_n^{\mathsf{learn},(k)}  = \sum_{g \in \mathcal{G}} x^{(k)}_{n,m,g}E_{n,m}^{\mathsf{off},(k)}  + \left( 1-\sum_{g \in \mathcal{G}} x^{(k)}_{n,m,g}\right)E_n^{\mathsf{loc},(k)}$  is the energy consumption for model training. Here,  constraints \eqref{constraint1} and \eqref{constraint2} imply that each dataset can be either trained locally or offloaded to at most one ES via a sub-channel. \eqref{constraint3} ensures the transmit power constraint of each MD. Constraint \eqref{constraint31} guarantees that each MD $n$ is allocated a feasible bandwidth resource for data offloading. The MD also allocates a positive computational resource to train its ML model with respect to the maximum CPU capability $F_n$, as indicated in \eqref{constraint4}. Constraint \eqref{constraint5} guarantees that the hash power allocated to each MD is limited  by the total system hash resource. Further, constraint \eqref{constraint51} implies that the energy consumption of MD $n$ for model training and blockchain mining is limited by its battery energy level. Finally,  the global loss function should be less than a desirable value $\epsilon$ to ensure the required training quality, as indicated in constraint \eqref{constraint6}. 

The optimization problem in~\eqref{Equa:Optimization} is non-convex with respect to the mixed discrete offloading and continuous allocation variables. Due to the time-varying nature of system states, such as channel condition and computational availability, it is challenging to directly solve the formulated problem via conventional optimization approaches such as Lyapunov optimization \cite{25}. Therefore, we will propose to use a learning-based approach, where a new DRL algorithm is developed to well capture the dynamics of system and integrate them into the solution design.

 \begin{figure*}
	\centering
	\includegraphics [width=0.99\linewidth]{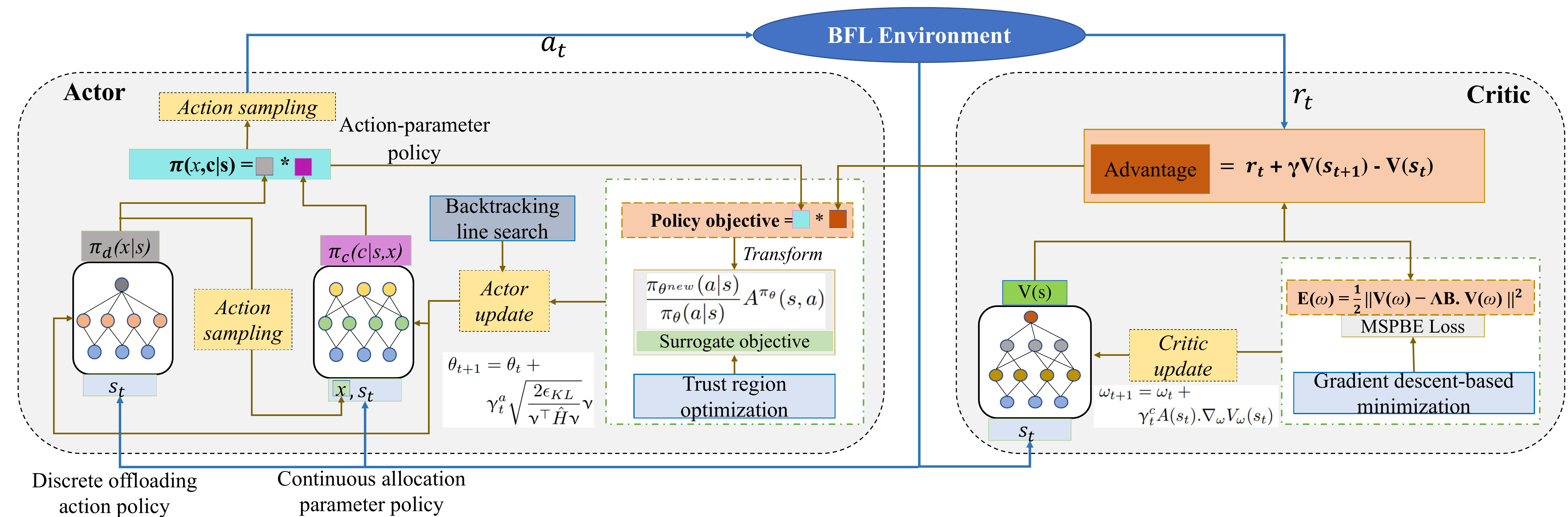}
	\caption{{The proposed parameterized A2C architecture for our BFL environment.} } %Each MD performs both data learning and block mining in the BFL system. To assist data learning, MDs can offload their data to one of the nearby ESs for edge training. Moreover, MDs participate in mining the block that contains  model updates to achieve a global model aggregation.  } }
	\label{Fig:A2C_Architecture}
\end{figure*}
\section{DRL Design with Parameterized A2C for BFL}
\label{Section:DRL}
Different from the existing DRL algorithms which consider either purely discrete \cite{23}, \cite{30}, \cite{31} or purely continuous actions \cite{14}, \cite{32}, \cite{37}, we here study a more practical DRL setting with a hybrid discrete-continuous action for improving the training performance.  Even though such a hybrid action setting has been previously mentioned in a few related works such as \cite{33}, a holistic investigation on the sampling of discrete and continuous actions has not been given. \textcolor{black}{Therefore, we propose a parameterized advantage actor critic (A2C) algorithm to optimize the system latency, as illustrated in Fig.~\ref{Fig:A2C_Architecture}.  We consider a hybrid discrete and continuous action space, where the resource allocation variables in~\eqref{Equa:Optimization} are continuous, while the offloading decision variables are discrete. The actor is designed to train both offloading and resource allocation policies, which we demonstrate in Section~\ref{subsection:actordesign}. The critic is then designed to evaluate the efficiency of the actor policy training, which we present in Section~\ref{subsection:CriticDesign}. We also include the advantage function in the critic design as it has been shown to reduce the variance in actor policy training compared with conventional actor-critic approaches \cite{37}, \cite{34}. Finally, we develop a training procedure to optimize the long-term system utility of our developed A2C algorithm (Section~\ref{subsubsection:trainingA2C}), which corresponds to minimizing the long-term BFL system latency.}

\subsection{DRL Formulation} 
\label{Subsection:DRLformulate}
\textcolor{black}{We consider a single-agent DRL problem setting \cite{14,37}, where a virtual centralized agent interacts with the BFL environment induced by the interaction between MDs and ESs during model training and block mining. Our DRL scheme aims to optimize the total system latency in~\eqref{Objective} consisting of model training latency, consensus latency and block mining latency. To handle the formulated optimization problem, we build a centralized agent which defines its reward based on the total utility, and restricts its action space based on the resource allocation and learning constraints in~\eqref{Equa:Optimization}. With a comprehensive view of the model training and block mining processes, the agent can obtain the system state and employ it to take efficient actions via well-trained data offloading and resource allocation policies that are observed to minimize the system latency.} We consider a parameterized  Markov Decision process characterized by $\mathcal{M} =<\mathcal{S}, \mathcal{A}, r>$. Here, $\mathcal{S} = \{s_1,...,s_N\}$ and $\mathcal{A}=\{a_1...,a_N\}$ are the finite sets of state and parameterized action, respectively. Further, $r(s,a): \mathcal{S} \times \mathcal{A} \rightarrow \left[-R_o,R_o\right]$, ($R_o \in \mathds{R}_{>0}$) denotes the bounded system reward. We also denote $\Rho(s'|s,a)$ as the  probability of transition executed by action $a$ of the agent from $s$ to $s'$.  

In our BFL latency optimization problem, the offloading decisions for ML model training at MDs are closely associated with resource allocation variables. For instance, to offload the data to an ES, an MD needs to tune its transmit power and determine channel availability; or to execute data locally, it should determine its  computational utilization. 
% Thus, we  focus on an action selection policy given relevant continuous parameters. We adopt a two-tier approach to perform the action training, where the discrete offloading action is selected, following by the selection of associated parameters. 
We consider a parameterized action space: a finite set of discrete offloading action $a \in \mathcal{A}_a = \{a_1,a_2,...,a_n\}$ defined via a discrete-action policy $\pi^d(a|s)$ and each $a_n$ has a set of continuous parameters $\{c \in \mathcal{A}_c  \}$ defined via an action-parameter policy $\pi^c(c|s,x)$. Thus, the joint action is given by conditional probability $(a,c) \sim \pi_\theta(a,c|s) = \pi^d_{\theta_d}(a|s) \pi^c_{\theta_c}(c|s,a)$, where $\theta$ is the parameter of the overall action policy, and $\theta_d$ and $\theta_c$ are parameters of  discrete action policy and parameter policy, respectively, with $\theta = [\theta_d, \theta_c]$. To simplify the notation, $ \pi_\theta(a,c|s) $ is expressed by $ \pi_\theta(a,s) $  in the rest of the paper.

At each time step $t$, the agent at state $s_t$ takes an action $a_t$ to transition to the next state $s_{t+1}$ and observe a reward $r_{t+1} \leftarrow r(s_t,a_t,s_{t+1})$. As a result, the training data for DRL is produced in a form of \textit{trajectory} where each data point on the trajectory can be represented by tuple $\{s_t,a_t,r_t,s_{t+1}\}$. Given a policy $\pi$, the long-term discounted system reward is characterized by the state-value function $V^{\pi_\theta} (s): \mathcal{A} \rightarrow \mathds{R}$ and the action-value function $Q^{\pi_\theta} (s,a): \mathcal{S} \times \mathcal{A} \rightarrow  \mathds{R} $, which are $V^{\pi_\theta} (s)  = \mathds{E}^{\pi_\theta}\left[\sum_{t=0}^{\infty} \gamma_t r_{t+1} |s_t =s \right]$ and $Q^{\pi_\theta} (s,a) = \mathds{E}^{\pi_\theta}\left[\sum_{t=0}^{\infty} \gamma_t r_{t+1} |s_t =s, a_t=a \right]$, where $\gamma_t \in [0,1]$ is the discount value and $\mathds{E}^{\pi_\theta}[.]$  represents expectation of the executed reward function under policy $\pi$. By using the Bellman optimality equation, the Q-value associated with a state-action pair can also be expressed by $Q^{\pi_\theta} (s_t,a_t) = \mathds{E}^{\pi_\theta}\left[ r_{t+1} +\gamma V^{\pi_\theta} (s_{t+1}) \right]$. In this paper, we focus on the A2C model that can be characterized by the advantage function $A^{\pi_\theta}(s,a)$
\begin{multline}
\label{equ:Advant}
A^{\pi_\theta}(s_t,a_t) = Q^{\pi_\theta} (s_t,a_t) - V^{\pi_\theta} (s_t) \\= 
r_{t+1} +\gamma V^{\pi_\theta} (s_{t+1}) - V^{\pi_\theta} (s_t), \forall s \in \mathcal{S}, a \in \mathcal{A}. 
\end{multline}

\textcolor{black}{It is worth noting that the proposed A2C algorithm training is performed in a certain FL aggregation round $k$ to allow the agent to obtain an optimal latency-aware offloading and allocation strategy for MDs. Based on the well trained A2C model, we then deploy it into the BFL environment to guide the FL training across aggregation rounds. Therefore, the index $k$ is dropped for the simplicity of notations in our DRL formulation.} In the following, we define state, action and reward for our DRL algorithm. 

\subsubsection{State} In our BFL environment, the  system state consists of five components: data state $S_{\mathsf{data}}(t)$, channel state $S_{\mathsf{channel}}(t)$, bandwidth state $S_{\mathsf{band}}(t)$, computation state $S_{\mathsf{comp}}(t)$, and hash power state $S_{\mathsf{hash}}(t)$. Therefore, the system state is defined as:
\begin{equation} 
s(t) = \{S_{\mathsf{data}}(t), S_{\mathsf{channel}}(t), S_{\mathsf{band}}(t),S_{\mathsf{comp}}(t),S_{\mathsf{hash}}(t)\}.
\end{equation}
Here, $S_{\mathsf{data}}(t)$ is defined as $S_{\mathsf{data}}(t) = \{D_n(t)\}_{ n \in \mathcal{N}}$. \textcolor{black}{Further, $S_{\mathsf{channel}}(t) = \{q_{n,g}(t)\}_{n\in\mathcal{N},g\in\mathcal{G}}$ indicates whether the sub-channel $g$ is used by MD $n$ at time slot $t$. If yes, $q_{n,g}(t) =1$, otherwise $q_{n,g}(t) =0$. The bandwidth state $S_{\mathsf{band}}(t)$ can be given $S_{\mathsf{band}}(t) = \{b_{n,g}(t)\}_{n\in\mathcal{N},g\in\mathcal{G}}$ under the total radio system bandwidth  $W$.} The computation state $S_{\mathsf{comp}}(t)$ contains the information of current  computational resource $S_{\mathsf{comp}}(t) = \{f_n^{\mathsf{\ell}}(t)\}_{n\in\mathcal{N}}$. Lastly, the hash power state $ S_{\mathsf{hash}}(t)$ presents the current  hash power $\Psi_n$ of all MDs: $S_{\mathsf{hash}}(t) = \{\Psi_1(t), \Psi_2(t),..., \Psi_N(t)\}$. 
\subsubsection{Action} Our BFL system features a parameterized action space with offloading or local execution, as elaborated below:

\begin{itemize}
	\item \textbf{Offloading} \textit{(transmit power, channel bandwidth allocation, hash power allocation)}: When MD $n$ chooses the offloading mode $x_{n,m,g} = 1$, it must determine relevant parameters, i.e., transmit power $p_n$ and channel bandwidth $b_{n,g}$ that are needed for offloading.  Also, MDs perform mining regardless of their learning status, and thus we also involve a hash power allocation parameter $\Psi_n$. Therefore, the joint action on each MD can be expressed by $a_n = \{x_{n,m,g},  p_n, b_{n,g}, \Psi_n\}$ and the complete system action  in this mode is  $a(t) = x_{n,m,g}(t), p_n(t), b_{n,g}(t), \Psi_n(t), \forall n \in \mathcal{N}, m \in \mathcal{M},  g \in \mathcal{G} $.
% 	\begin{equation} 
% 	a(t) =  
% 	\begin{bmatrix}
% 	&x_1&p_1&b_1&\Psi_1 \\
% 	&x_2&p_2&b_2&\Psi_2 \\
% 	& \vdots& \vdots &\vdots& \vdots \\
% 	&x_N &p_N &b_N &\Psi_N
% 	\end{bmatrix}. 
% 	\end{equation}
	\item \textbf{Local Execution} \textit{(computational allocation, hash power allocation)}: When MD $n$ chooses the local execution mode $ x_{n,m,g} = 0$, it must specify its necessary parameters to execute the model training, e.g., $f^{\mathsf{\ell}}_n$. Moreover, similar to the offloading mode, the parameter of hash power allocation is also involved to support the block mining task executed after the model learning. Therefore, the joint action on each MD is $a_n = \{x_{n,m,g},  f_n^{\mathsf{\ell}}, \Psi_n \}$ and the complete action for the BFL system in this local execution mode is given by $a(t) = x_{n,m,g}(t), f_n^{\mathsf{\ell}}(t), \Psi_n(t), \forall n \in \mathcal{N}, m \in \mathcal{M},  g \in \mathcal{G}$.
% 	\begin{equation} 
% 	a(t) =  
% 	\begin{bmatrix}
% 	&x_1&f^l_1&\Psi_1 \\
% 	&x_2&f^l_2&\Psi_2 \\
% 	& \vdots&\vdots& \vdots \\
% 	&x_N &f_n^{\mathsf{\ell},(k)} &\Psi_N
% 	\end{bmatrix}. 
% 	\end{equation} 
\end{itemize}

\subsubsection{System Reward Function} 
\label{subsectionReward}
The reward in our BFL system comes from the joint model learning and block mining, by maximizing the system returns in the long run. However, in the  optimization problem~\eqref{Equa:Optimization}, our objective is to minimize the system latency in a certain aggregation round, which requires a negative multiplication before being used as the reward, i.e., we define the reward as $r(s_t,a_t) = - \left(T^{\mathsf{learn}} + T^{\mathsf{cons}}  + T^{\mathsf{mine}} \right)$. For better presentation, we transform the latency minimization problem into a system utility optimization problem by using a simple exponential equation:
\begin{equation}
\label{MiningUtility}
U = \left[e^{\left(1-\frac{\left(T^{\mathsf{learn}} + T^{\mathsf{cons}}  + T^{\mathsf{mine}} \right)}{\tau} \right)} -1\right], 
\end{equation}
where $\tau$ denotes an upper bound of the system latency. \eqref{MiningUtility} implies that the lower system latency results in a higher system utility. Therefore, instead of minimizing the latency, we are keen on maximining the system utility which better characterizes  the efficiency of our algorithm.  Accordingly, we transform~\eqref{Equa:Optimization} to the following optimization problem:
\begin{subequations}
	\label{Equa:Optimization1}
	\begin{align} 
	& \underset{\boldsymbol{X},\boldsymbol{P}, \boldsymbol{B}, \boldsymbol{F},\bm{\Psi}}{\maximize} 
	&&  U \label{Objective1}\\
	& \phantom{P1} \text{s.t.} 
	&& \ref{constraint1} - \ref{constraint6}.
	\end{align}
\end{subequations}
Thus, we  re-define the system reward as a result of executing the action with given states  as $r(s_t,a_t) = U(t)$. 
\subsection{{Policy Gradient Update for A2C}}
\textcolor{black}{We first analyze the policy gradient update necessary for the design of actor and critic components that will be elaborated later. In the BFL environment, the agent tries to search among the set of parameterized offloading policies to obtain the optimal policy $ \pi^{*}_\theta(a,s)$ that can return the maximum reward, i.e., system utility. However, in practice the search space may be very large and the agent may not be able to find the optimal policy.} Thus, we restrict the policy set by a vector $\theta \in \mathds{R}^z$ for some integer $z >0$ and perform the optimization over the group of the parameterized policies  $\pi_\theta(a,s)$.  To facilitate our analysis, the following assumptions are introduced.
 \begin{assumption}
 \label{Ass1}
\textit{The policy $\pi_\theta$ and $\mathcal{P}(s'|s,a)$ guarantee an irreducible and aperiodic Markov chain described  by  $\mathcal{P}^{\pi_\theta}(s'|s)$, $\forall \theta$. Therefore, there exists a stationary distribution defined as $d^{\pi_\theta}(s)$ on policy $\theta$. }
 \end{assumption}
\textcolor{black}{Assumption~\ref{Ass1} is a common assumption for actor-critic algorithms \cite{37}, \cite{34}.} Accordingly, we define $J(\pi_\theta) =r(\pi_\theta)$, where $r$ is the system reward defined in Section~\ref{subsectionReward}, as the performance function with respect to the policy parameterized by $\theta$. \textcolor{black}{Accordingly, the objective function with linear Q-value function approximation is given by
\begin{equation}
 J(\pi_\theta) = \sum_{s \in \mathcal{S}} d^{\pi_\theta}(s) Q^{\pi_\theta} (s,a). 
\end{equation}}
\textcolor{black}{Next, similar to \cite{34}, we make an assumption on the policy $\pi_\theta(a,s)$.} 
\begin{assumption}
\label{Ass2}
The following assumptions are made on the policy function:
\textit{\begin{itemize}
		\item Positive policy function: $\pi_\theta(a|s) >0, \forall \theta \in \mathds{R}^d$
		\item Bounded policy gradient: 
		\begin{align*}
		||\nabla \log \pi_\theta(a|s)||_2 < G_\pi, \forall \theta,\forall s, \forall a, G_\pi >0
		\end{align*}
		\item $\ell$-Lipschitz policy gradient:
		\begin{align*} 
		|| \nabla \log \pi_{\theta_1} -  \nabla\log \pi_{\theta_2}||_2 \leq \ell||\theta_1 - \theta_2||_2, \forall \theta_1,\forall \theta_2 \end{align*}
\end{itemize}}
\end{assumption}
Here, the regularity conditions in Assumption~\ref{Ass2} can be satisfied by using the Gibbs softmax distribution network, e.g., in deep neutral networks, for action selection in the actor. Under this assumption, the policy gradient $\pi_\theta$ can be updated  as $\nabla J(\pi_\theta) = \mathds{E}_{s\sim d^{\pi_\theta}(.),a\sim \pi_\theta(.|s)}   \left[ Q^{\pi_\theta}(s,a) \nabla  \log\pi_\theta(a|s)\right]$. 

% \begin{align}
% 	\nabla J(\pi_\theta) = \mathds{E}_{s\sim d^{\pi_\theta}(.)} \left[\sum_{a \in \mathcal{A}} Q^{\pi_\theta}(s,a) \nabla  \pi(a|s)\right],
% 	\end{align}
% 	which can be approximated by 
% 	\begin{dmath}
% 	\nabla J(\pi_\theta) = \mathds{E}_{s\sim d^{\pi_\theta}(.)}  \left[\pi(a|s) Q^{\pi_\theta}(s,a)  \frac{\nabla  \pi(a|s)}{ \pi(a|s)}\right] \\ \approx \mathds{E}_{s\sim d^{\pi_\theta}(.),a\sim \pi_\theta(.|s)}   \left[ Q^{\pi_\theta}(s,a) \nabla  \log\pi_\theta(a|s)\right].
% 	\end{dmath}
% 	\textcolor{black}{In this paper, we focus on the A2C version with the advantage function defined in~\ref{equ:Advant}}where the expectation form still holds
% 	\begin{equation}
% 	\label{advant}
% 	\nabla J(\pi_\theta) = \mathds{E}_{s\sim d^{\pi_\theta}(.),a\sim \pi_\theta(.|s)}   \left[ A^{\pi_\theta}(s,a)  \nabla_\pi \log\pi_\theta(a|s)\right].
% 	\end{equation}

\subsection{Actor Design}
\label{subsection:actordesign}
In our A2C-based DRL algorithm, the actor aims to update the parameter $\theta$ over time-step iterations to find the optimal policy $\pi_\theta^*$ that characterizes the best trajectory for our system utility optimization problem. In other words, the actor is expected to make optimal model learning and block mining  decisions in a fashion that the long-term reward (i.e., system utility) is maximized. In doing so, the actor needs to use the gradient $ \nabla J(\pi_\theta) $  to optimize its policy
\begin{equation}
\label{maximize_func}
\max_{\theta \in \mathds{R}^d} J(\pi_\theta)=   \mathds{E}_{s\sim d^{\pi_\theta}(.)} \left[ \pi_\theta(a|s) A^{\pi_\theta}(s,a)\right],
\end{equation}
Traditionally, the policy is optimized via a vanilla policy gradient algorithm by direct policy search over the entire exploration space which is known to be inefficient. \textcolor{black}{Instead, we use trust region policy optimization (TRPO) to  improve policy optimization} by maximizing a surrogate objective over a trust-region \cite{35}. Accordingly, \eqref{maximize_func} can be re-written 
\begin{equation}
\label{Equ:maximize}
\max_{\theta \in \mathds{R}^d} J(\pi_\theta)=   \mathds{E}_{\pi_{\theta}} \left[ \frac{\pi_{\theta^{\mathsf{new}}}(a|s)}{\pi_{\theta}(a|s)} A^{\pi_{\theta}}(s,a)\right],
\end{equation}
\begin{equation*}
subject~to~~~~ \mathds{E}_s \left[KL \left(\pi_{\theta}(.|s)|| \pi_{\theta^{\mathsf{new}}}(.|s)\right) \right] \leq \epsilon_{KL},
\end{equation*}
for some  $\epsilon_{KL} >0$. In~\eqref{Equ:maximize}, $\mathds{E}_s$ represents the state visitation distribution induced by $\pi_\theta$, and $\theta$ and $\theta^{\mathsf{new}}$ are the vectors of policy parameters before and after each update, respectively. By enforcing a Kullback-Leibler (KL)-divergence constraint $KL$, the probability distributions of the policy before and after the update will be kept closely in the parameter space to avoid divergence in the gradient update. Considering our parameterized action space including the offloading actions and their parameters, the optimization objective function can be expressed as $\mathds{E}_{\pi_{\theta}} \left[ \frac{\pi^d_{\theta_d^{\mathsf{new}}}(a|s) \pi^c_{\theta_c^{\mathsf{new}}}(c|s,a)}{\pi^d_{\theta_d}(a|s) \pi^c_{\theta_c}(c|s,a)} A^{\pi_{\theta'}}(s,a)\right]$. 

\textcolor{black}{To solve~\eqref{Equ:maximize}, we construct a closed form  solution for computing the KL-divergence between the distributions of the discrete offloading action policy and the action-parameter policy for our BFL problem as follows:}
\begin{align}
\label{Equa:KLDiverg}
&\scriptstyle \mathds{E}_s \left[KL \left(\pi_{\theta}(a,c|s)|| \pi_{\theta^{\mathsf{new}}}(a,c|s)\right) \right] 
\nonumber \\ &\scriptstyle  =
\mathds{E}_s \left[KL \left(\pi^d_{\theta_d^{\mathsf{new}}}(a|s)||\pi^d_{\theta_d}(a|s)\right) 
 +  KL \left(\pi^c_{\theta_c^{\mathsf{new}}}(c|s,a)||\pi^c_{\theta_c}(c|s,a)\right) \right] 
 \nonumber \\ &\scriptstyle =
\mathds{E}_s \left[KL \left(\pi^d_{\theta_d^{\mathsf{new}}}(a|s)||\pi^d_{\theta_d}(a|s)\right) \right] +
\nonumber \\ & \hspace{14mm} \scriptstyle
\mathds{E}_s \mathds{E}_{a \sim \pi^d_{\theta_d^{\mathsf{new}}(a|s)}}
 \left[KL \left(\pi^c_{\theta_c^{\mathsf{new}}}(c|s,a)||\pi^c_{\theta_c}(c|s,a)\right) \right]. 
\end{align}
By using the analytical form of the discrete offloading action policy $\pi^d_{\theta_d^{\mathsf{new}}(a|s)}$  via its trajectory probability, we can further reduce the variance of the KL-divergence in the last term of~\eqref{Equa:KLDiverg} as
\begin{multline}
\scriptstyle \mathds{E}_s \left[KL \left(\pi_{\theta}(a,c|s)|| \pi_{\theta^{\mathsf{new}}}(a,c|s)\right) \right] \approx
\mathds{E}_s \left[KL \left(\pi^d_{\theta_d^{\mathsf{new}}}(a|s)||\pi^d_{\theta_d}(a|s)\right) \right] \\ 
\scriptstyle +
\mathds{E}_s
\left[ \left( -\log(\pi^d_{\theta_d^{\mathsf{new}}(a|s)}) \right)  KL \left(\pi^c_{\theta_c^{\mathsf{new}}}(c|s,a)||\pi^c_{\theta_c}(c|s,a)\right) \right]. 
\end{multline}
Based on the above approximation steps, the optimization of the original policy in~\eqref{maximize_func} is transformed into a conjugate gradient form which allows for estimating the expectations of the policy objective in~\eqref{Equ:maximize} with policy improvement guarantees. In this regard, the update direction can be approximated by $\upnu \approx \hat{H}^{-1}\nabla_\theta J(\pi_\theta),$ where $\hat{H}^{-1}$ is the Hessian-vector product of sampled KL-divergence. Finally, the actor parameter is updated via backtracking line search \cite{36} subject to the KL constraint as follows:
\begin{equation}
\label{Equa:actorupdate}
\theta_{t+1} = \theta_t + \upgamma_t^a \sqrt{\frac{2\epsilon_{KL}}{\upnu^\top \hat{H} \upnu }} \upnu, 
\end{equation}
where $\upgamma_t^a \in (0,1) $ is the  backtracking step-size parameter which controls the line search for guaranteeing the conjugate gradient improvement given the KL divergence constraint. It is desirable to set up a fairly large step size $ \upgamma_t^a$ to initialize the line search space on the policy, and gradually shrink $\upgamma_t^a$ until a Armijo-Goldstein condition \cite{36} is satisfied where a critical (optimal) point is obtained. 
% If the step size is kept large and constant over the line search iteration, the condition of conjugate gradient-based convergence may not be satisfied which can lead to policy divergence. This will be investigated in the simulation section.  

\subsection{Critic Design}
\label{subsection:CriticDesign}

The role of the critic is to estimate the state-value function $ V^{\pi_\theta} (s)$ to guide the update of the actor by approximating its state-value function. Specifically, a feature function $\upphi: \mathcal{S} \mapsto  \mathds{R}^n$ is created as a full-ranked matrix with $I$ dimensions to create $i$-dimensional features ($i \leq I$) for any state $s \in \mathcal{S}$, i.e., $\upphi(s)= \left(\upphi^1(s),..., \upphi^i(s)\right)^\top$. Given a state $s$, the state-value function is thus approximated by a linear function $V_\omega(s) \approx \omega \upphi(s)^\top,$ where $\omega \in \mathds{R}^m $ is a parameter vector used to update the state-value function. By function approximation, the critic provides an inexact temporal difference (TD) solution to the value function $ V^{\pi_\theta} (s)$ under policy $\pi_\theta $. This naturally results in the minimization of TD error as a loss function, \textcolor{black}{i.e., the Mean Squared Projected Bellman Error (MSPBE) defined by
\begin{equation}
\label{Equa:loss}
E_\theta(\omega) = \frac{1}{2} || V_\omega -\Lambda B_\theta V_\omega  ||^2,
\end{equation}
where $	\Lambda = \upphi^\top(\upphi M \upphi^\top)^{-1}\upphi \Eta$ is the projector with $\Eta \in \mathds{R}^{|\mathcal{S}| \times |\mathcal{S}|}$ being a diagonal matrix whose elements are within the stationary state distribution $ d^{\pi_\theta}$ generated according to the policy $\pi_\theta$ when the entire state space is irreducible. Also, $ B_\theta$ denotes the Bellman operator implied by~$B_\theta V(s) \leftarrow r(s,a)+ \gamma \Rho_\theta(s,s',a) V(s)$, where $V(s)$ is the state value, $r(s,a)$ is reward with discount  $\gamma \in (0,1)$, and $\Rho_\theta(s,s',a)$ is the transition probability as defined in~\ref{Subsection:DRLformulate}.} During the value function evaluation process, the critic aims to minimize the loss function in~\eqref{Equa:loss} to obtain a fixed point of the projected Bellman operator by gradient TD learning, where the Lipschitz continuity property of the MSPBE loss function in~\eqref{Equa:loss} is significant to guarantee a \textcolor{black}{successful policy-parameter update}. Note that the loss function in is quadratic and thus convex with respect to $\theta$.
%\cite{chung2018two}
\begin{lemma}
\textit{Given the feature vector $\upphi(s)= \left(\upphi^1(s),..., \upphi^i(s)\right)^\top$, the $ E_\theta(\omega)$'s gradient is $\ell$-Lipschitz with $\ell = (1 +\gamma)^2 \max_{i}||\upphi^i||^2_2$, where $i \in \mathcal{I}$.}
\end{lemma}
\begin{proof}
See Appendix~\ref{subsection:Appen-LossCritic}. % in the technical report~\cite{tech}.  
%See Appendix~C in the technical report~\cite{tech}. 
\end{proof}
Based on the stochastic gradient descent on $ E_\theta(\omega)$ with respect to parameter $\omega$, the critic update can be achieved by 
\begin{equation}
\label{Equa:criticupdate}
\omega_{t+1} = \omega_t + \upgamma_t^c  A^{\pi_\theta}(s_t,a_t) \nabla_\omega V_\omega(s_t),
\end{equation}
where $\upgamma_t^c \in (0,1)$ is a step-size parameter.

\begin{algorithm}
	\footnotesize
	\caption{{Training procedure of our A2C algorithm}}
	\begin{algorithmic}[1]
		\label{Al:DRL}
		\STATE \textbf{Input:}   Time budget $T$, discount factor $\gamma$, BFL environment $env$
		\STATE \textbf{Output:} Optimal parameterized action policy $\pi^*_\theta$ and maximum reward $R$
		\STATE \textbf{Initialization:} Initialize offloading policy parameters $\theta_d$, parameterized allocation policy parameter $\theta_c$, critic parameter $\omega$,  actor's learning rate $\upgamma^a$,  critic's learning rate $\upgamma^c$, discount parameter $\upgamma$,  initial reward $R=0$ 
		\FOR{each episode} 
		\STATE Set up initial state $s_0$ \label{Algoline:initialize}
		\FOR{timestep $t = 1,2,...,T$} 
		\STATE Perform action sampling using $\pi^d(.|s)$ and $\pi^c(.|s)$ \label{Algoline:actionsample}
		\FOR{each MD $n \in \mathcal{N}$ } 
		\STATE Sample the discrete offloading action with policy $x_n \sim \pi^d_n{_{\theta_d}(.|s)}$
		\STATE Sample continuous allocation parameters for the selected offloading action $x_n$
		\[
		c_n= 
		\begin{cases}
		(p_n, b_{n,g}, \Psi_n) \sim \pi^c_{\theta_c}(.|s,x_n),& \text{if } x_{n,m,g} = 1\\
		(f_n^{\mathsf{\ell}}, \Psi_n) \sim \pi^c_{\theta_c}(.|s,x_n), & \text{if } x_{n,m,g} = 0
		\end{cases}
		\]
		\ENDFOR \label{Algoline:actionsample1}
		\STATE Execute the parameterized action $ a_t \leftarrow (x_{n,m,g,t},c_{n,t}), \forall n \in \mathcal{N}$ in the BFL environment  \label{Algoline:action}
		\STATE Obtain reward $r_t$ and next state $s_{t+1}$: $(r_t, s_{t+1}) \leftarrow env.step(a_t)$ \label{Algoline:reward}
		\STATE Calculate the accumulated reward $R_{t+1} \leftarrow r_t + \gamma_t*R_t $ \label{Algoline:reward1}
		\STATE Compute the advantage function $A^{\pi_\theta}(s,a)$ using~\eqref{equ:Advant} based on $r_t$ and  $ V^{\pi_\theta} (s)$ \label{Algoline:updateall1}
		\STATE Estimate the policy gradient for the actor: $\nabla J(\pi_\theta) = \mathds{E}_{s\sim d^{\pi_\theta}(.),a\sim \pi_\theta(.|s)}   \left[ A^{\pi_\theta}(s,a)  \nabla_\pi \log\pi_\theta(a|s)\right]$
		\STATE Optimize the actor policy via TRPO in \eqref{Equ:maximize} with computed advantage value and update its parameter with KL constraint in~\eqref{Equa:actorupdate}: $\theta_{t+1} = \theta_t + \upgamma_t^a \sqrt{\frac{2\epsilon_{KL}}{\upnu^\top \hat{H} \upnu }} \upnu$
		\STATE Calculate the MSPBE loss for critic via~\eqref{Equa:loss}
		\STATE Optimize $V_\omega$ for the critic and update its policy in~\eqref{Equa:criticupdate} via TD error $\sigma_t$: $\omega_{t+1} = \omega_t + \upgamma_t^c  A^{\pi_\theta}(s_t,a_t). \nabla_\omega V_\omega(s_t)$ \label{Algoline:updateall2}
		\ENDFOR 
		% \STATE Reset the state for the next training episode $s_t \leftarrow env.reset()$
		\ENDFOR
	\end{algorithmic}
	
\end{algorithm}
\vspace{-0.1in}
\subsection{Training Procedure of Parametrized A2C Algorithm in BFL}
\label{subsubsection:trainingA2C}

 The training procedure of our A2C algorithm is summarized in Algorithm~\ref{Al:DRL}. To prepare for the training, we build a BFL environment where multiple MDs participate in the FL training and block mining with a set of ESs connected via wireless links. The objective function in the system utility optimization problem built in~\eqref{Equa:Optimization1} is selected as a system DRL reward that is obtained via iterative training with parameterized actions and system states for an optimal offloading policy  $\pi^*_\theta$ to maximize the reward $R$ in the long run with a initial system state (line~\ref{Algoline:initialize}). We build an actor that consists of a deep neural network (DNN) for offloading decision sampling and another DNN for parameter sampling. At each timestep, we randomly sample a set of discrete offloading actions for all MDs via a DNN-empowered policy $\pi^d(.|s)$. \textcolor{black}{Then we also sample a set of continuous allocation parameters based on the sampled offloading decision using another policy $\pi^c_{\theta_c}(.|s,x_n)$ as defined in~\ref{Subsection:DRLformulate}} (lines~\ref{Algoline:actionsample}-\ref{Algoline:actionsample1}). With the result of the action sampling step, we  generate a complete set of parameterized actions for all MDs which is ready to be executed in the BFL environment (line~\ref{Algoline:action}). This allows us to obtain the reward, i.e., system utility, to calculate the long-term return and the agent moves to the next state needed for the following step of training (lines~\ref{Algoline:reward}-\ref{Algoline:reward1}). Then, the actor and critic updates their policy  (lines~\ref{Algoline:updateall1}-\ref{Algoline:updateall2}): (i) the actor computes the policy gradient via its advantage function and TRPO to optimize its policy, and (ii) the critic computes the MSPBE loss function and updates its gradient via TD error learning. 
% As the result of the A2C algorithm training, we generate an optimal offloading policy for all MDs in the BFL system along with an accumulated system utility. 

\section{Simulations and Performance Evaluation }
\label{Section:Simulate}
\subsection{Parameter Settings}
\begin{table}
	\caption{Simulation parameters.}
	\label{Table:Acronyms}
	\scriptsize
	\centering
	\captionsetup{font=scriptsize}
	\setlength{\tabcolsep}{5pt}
	\begin{tabular}{p{4.5cm}|p{2.4cm}}
		\hline
		\textbf{Parameter}& 
		\textbf{Value}
		\\
		\hline
		Number of ESs $M$& 5
		\\
		Number of MDs $N$& [20-100]
		\\
		Number of sub-channels at each ES $K$ & 5
		\\
		Data size $D_n$  & [0.5-2] MB
		\\
		CPU workload of MDs and ESs & [0.7-1.1] Gcyles
		\\
		MDs' transmit power  $p_n$ & [10-30] dBm
		\\
		MD's computational capability $f_n^{\mathsf{\ell}}$ & [0,2-2] GHz
		\\
		ES's computational capability $f_m$ & 5 GHz
		\\
		Background noise variance  & -100 dBm
		\\
		Maximum system bandwidth $W$ & 20 MHz
		\\
		MD's model size $\vartheta$ & 5 KB
		\\
		MD's energy coefficient $\kappa$ & $5*10^{-27}$
		\\
		MD's hash rate $\Psi_n$ & [100-1000] GHash/s 
		\\
		MD's mining power efficiency $\Xi_n$ &  $5*10^{-8}$ J/hash
		\\
		Hash amount of a block $\hbar$ & 50 GHash 
		\\
		Block broadcasting rate $\xi$ & 0.005 
		\\
		Number of forking occurrences  $\zeta$ & 3
		\\
		MD's maximum latency $\tau_n$ & 3 sec
		\\
		Noise power spectral density  $N_0$ & -174 dBm/Hz
		\\
		ES's bandwidth $b_{m'}$ & 5 MHz
		\\
		ES's transmit power $p_{m'}$ & [100-120] dBm
		\\
		\hline
	\end{tabular}
	\label{tab1}
	\vspace{-0.15in}
\end{table}
We conduct numerical experiments to verify our method under various parameter settings.  Inspired by related works \cite{14, 15,16 }, \cite{25, nguyen2021cooperative, nguyen2020privacy}, we set up all necessary parameters for our BFL environment as listed in Table~\ref{Table:Acronyms}. 

We consider a BFL system with 3 ESs and 10 MDs which aim to collaboratively train two popular image datasets: SVHN\footnote{http://ufldl.stanford.edu/housenumbers/} (including 73,257 training instances and testing 26,032 instances) and Fashion-MNIST\footnote{https://github.com/zalandoresearch/fashion-mnist} (including 60,000 training instances and testing 10,000 instances), where each dataset contains 10 labels/classes. \textcolor{black}{For the SVHN dataset, we deploy a convolutional neural network (CNN) with two 2-D convolutional layers followed by two hidden layers (the first with 256 units and the second with 72 units) with ReLU activation. The CNN architecture used for the Fashion-MNIST dataset is similar,} with the first hidden layer with 320 units and the second hidden layer with 50 units. We investigate the FL performance under both IID and non-IID data settings. \textcolor{black}{In IID data setting, each MD possesses datapoints from all the 10 labels, while in non-IID data setting,} each MD contains data samples from three of 10 labels for the SVHN dataset and two of 10 labels for the Fashion-MNIST dataset. We employ the Adam optimizer with mini-batch size of 25 and 10  SGD iterations.

In our A2C algorithm, the actor has two DNNs, one for the discrete offloading policy with two hidden neural layers \{ 64 and 32 in sizes\} and one for the continuous allocation parameter policy with two layers  \{128 and 64 in sizes\}. For the output layers, we used Softmax to generate offloading decisions for MDs and adopted Tanh to produce allocation parameters given the discrete offloading policy. 
%Since we have  four vectors of continuous parameters (i.e., transmit power, channel bandwidth allocation, local computational allocation, and hash power allocation), we employed a clipped function with a common bound (i.e., in the range of [0,1]) to sample these parameters, and then renormalize each of these parameters at the execution stage in the BFL environment. 
The KL divergence constraint $ \epsilon_{KL}$ is set to 0.01 for the TRPO-based actor policy optimizer \cite{35}. The critic was also built by a DNN that contains two hidden layers of sizes  \{200, 100\} \textcolor{black}{to train the state-value function of our A2C scheme with the Adam optimizer.} To implement our parameterized A2C algorithm, we set up a virtual agent that is allowed to interact with a pre-defined BFL environment to learn the parameterized offloading policy for all MDs and observe the return, i.e., system utility, after each iteration. \textcolor{black}{We trained the agent over 10000 episodes with 100 timesteps per each episode. }  All numerical results are averaged over 10 independent simulation runs. 
%All simulations were implemented in Python 3.8 with TensorFlow 2.0 using a computer with 4.7GHz CPU Intel Core i7 and 512 GB memory.
\begin{figure}[t!]
	\centering
	\begin{subfigure}[t]{0.24\textwidth}
		\centering
		\includegraphics[width=0.99\linewidth]{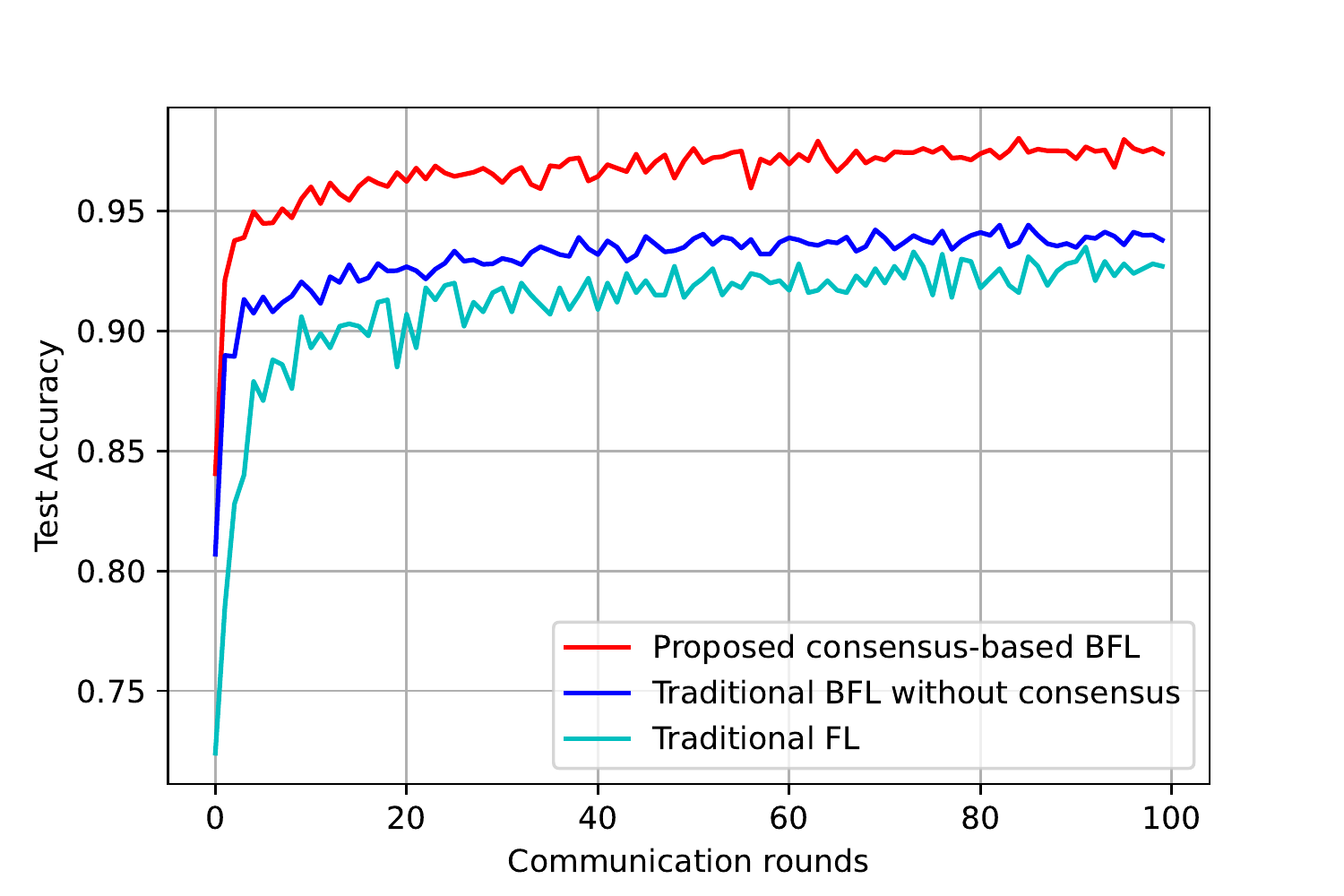} 
		\caption{Accuracy under IID  setting. }
	\end{subfigure}%
	~
	\begin{subfigure}[t]{0.24\textwidth}
		\centering
		\includegraphics[width=0.99\linewidth]{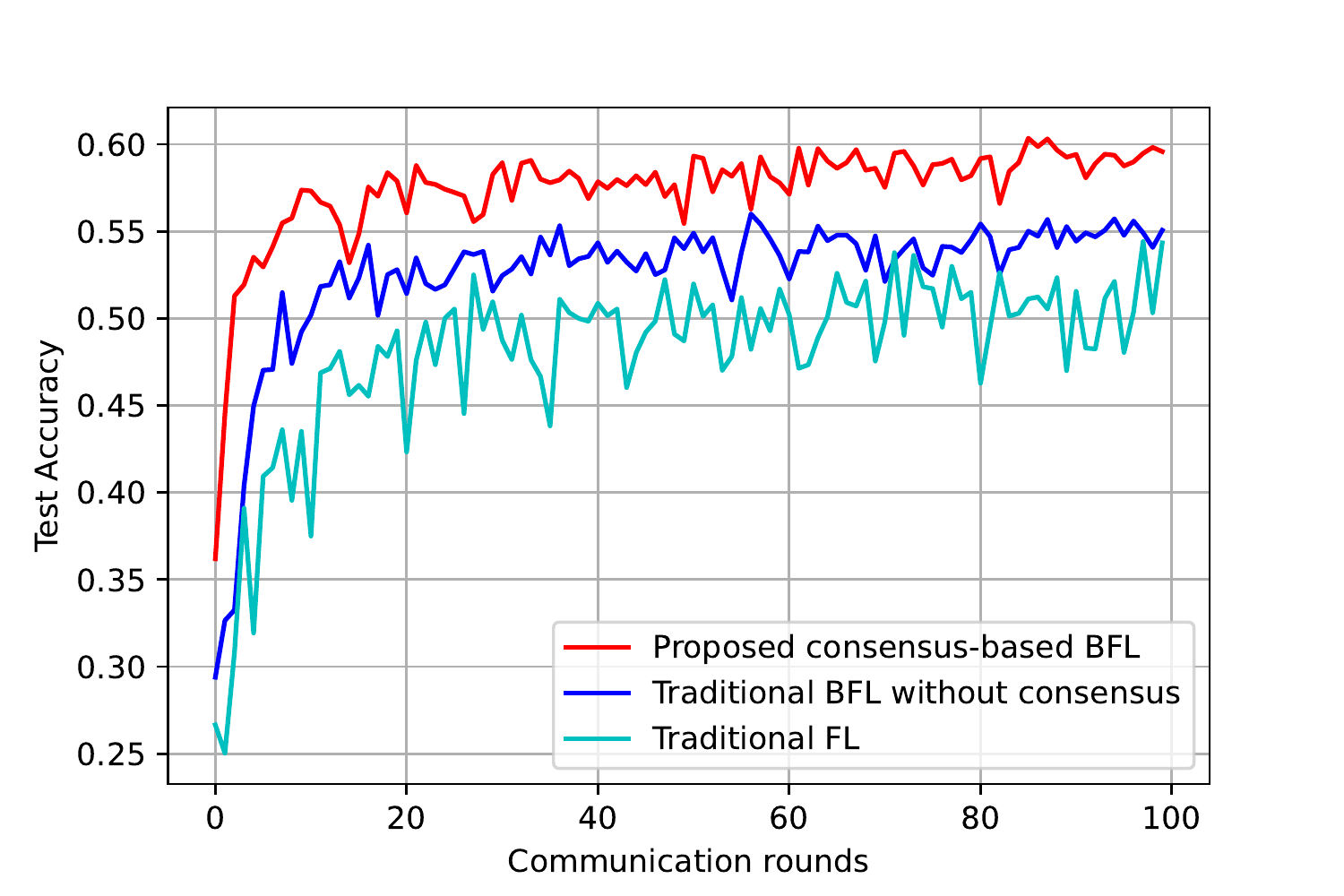} % 7-8.8
		\caption{Accuracy under non-IID  setting.  }
	\end{subfigure}
	~
	\begin{subfigure}[t]{0.24\textwidth}
		\centering
		\includegraphics[width=0.99\linewidth]{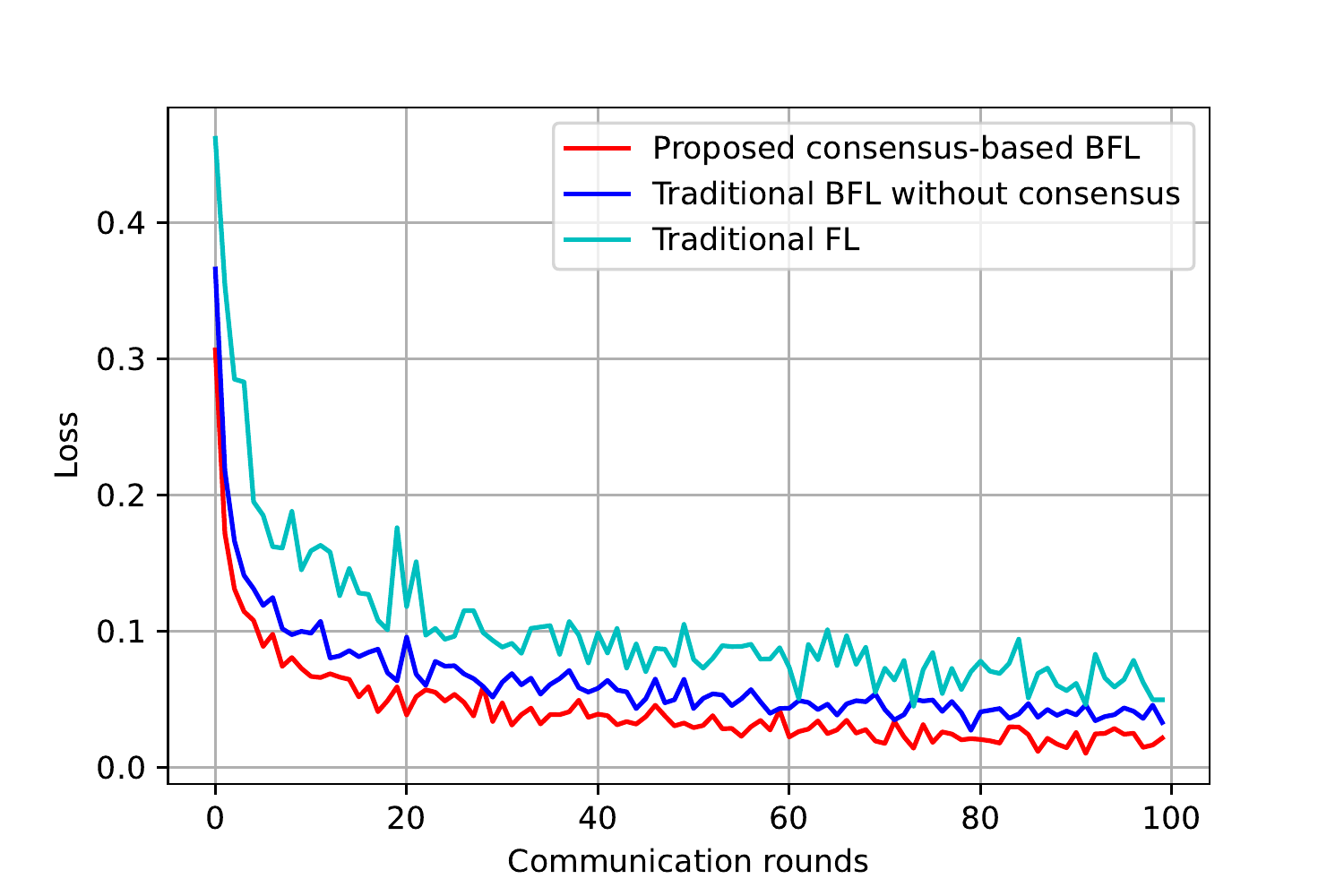} 
		\caption{Loss  under IID setting. }
	\end{subfigure}%
	~
	\begin{subfigure}[t]{0.24\textwidth}
		\centering
		\includegraphics[width=0.99\linewidth]{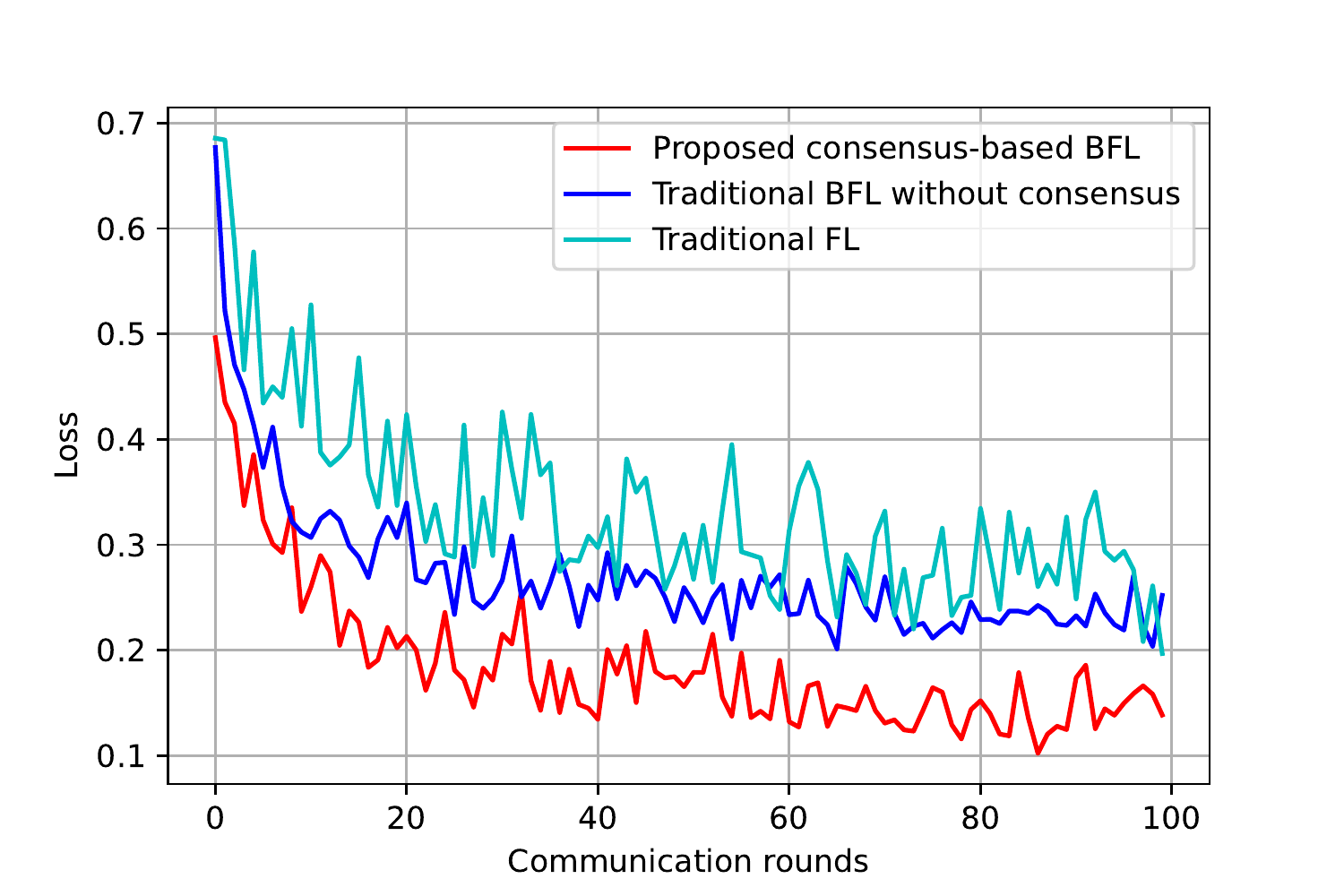} % 7-8.8
		\caption{Loss  under non-IID setting.   }
	\end{subfigure}
	\caption{{Comparison of different FL approaches on SVHN dataset.}}
	\label{FL-compare-SVHN_Result}
	\vspace{-0.1in}
\end{figure}

\begin{figure}[t!]
	\centering
	\begin{subfigure}[t]{0.24\textwidth}
		\centering
		\includegraphics[width=0.99\linewidth]{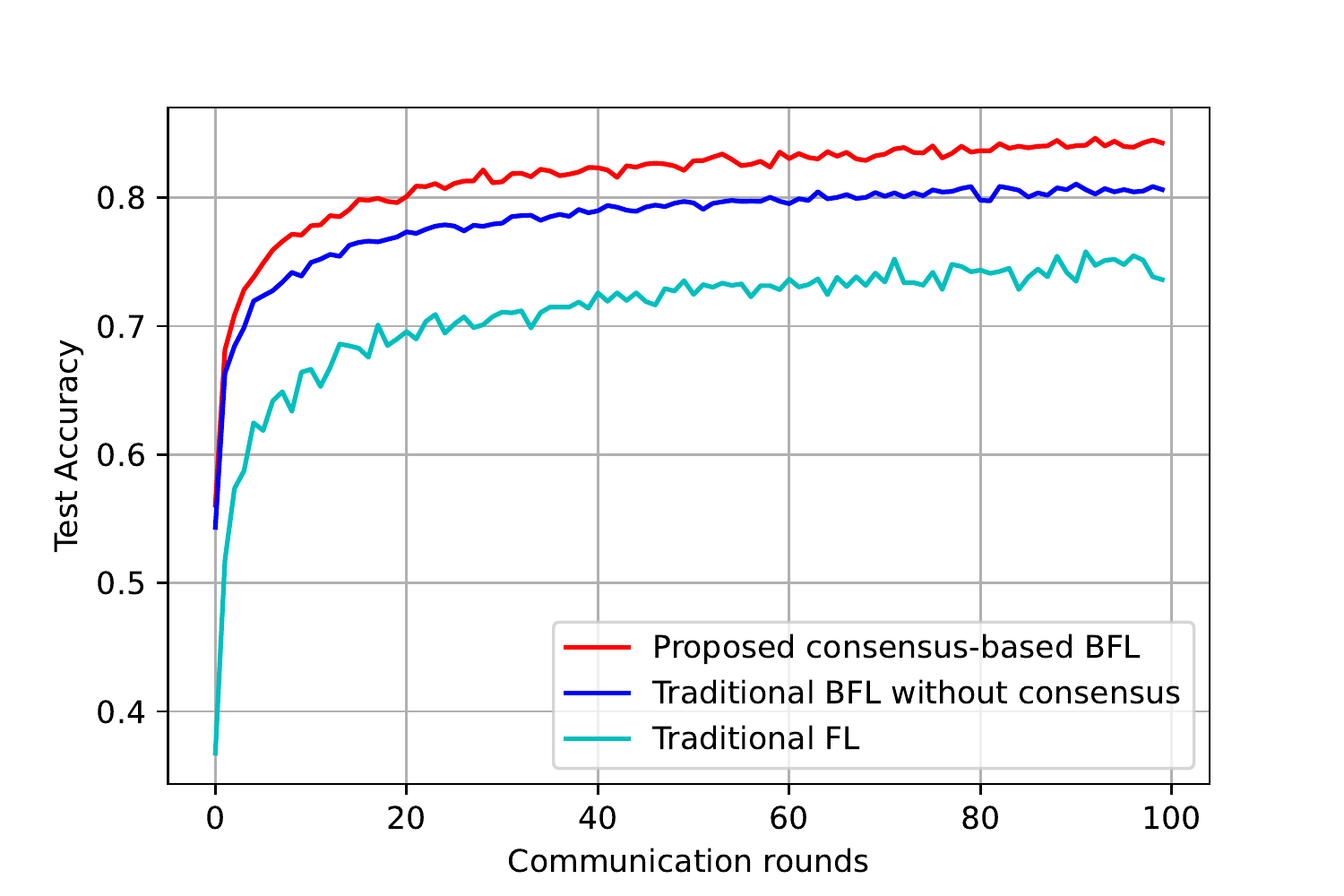} 
		\caption{Accuracy under IID setting. }
	\end{subfigure}%
	~
	\begin{subfigure}[t]{0.24\textwidth}
		\centering
		\includegraphics[width=0.99\linewidth]{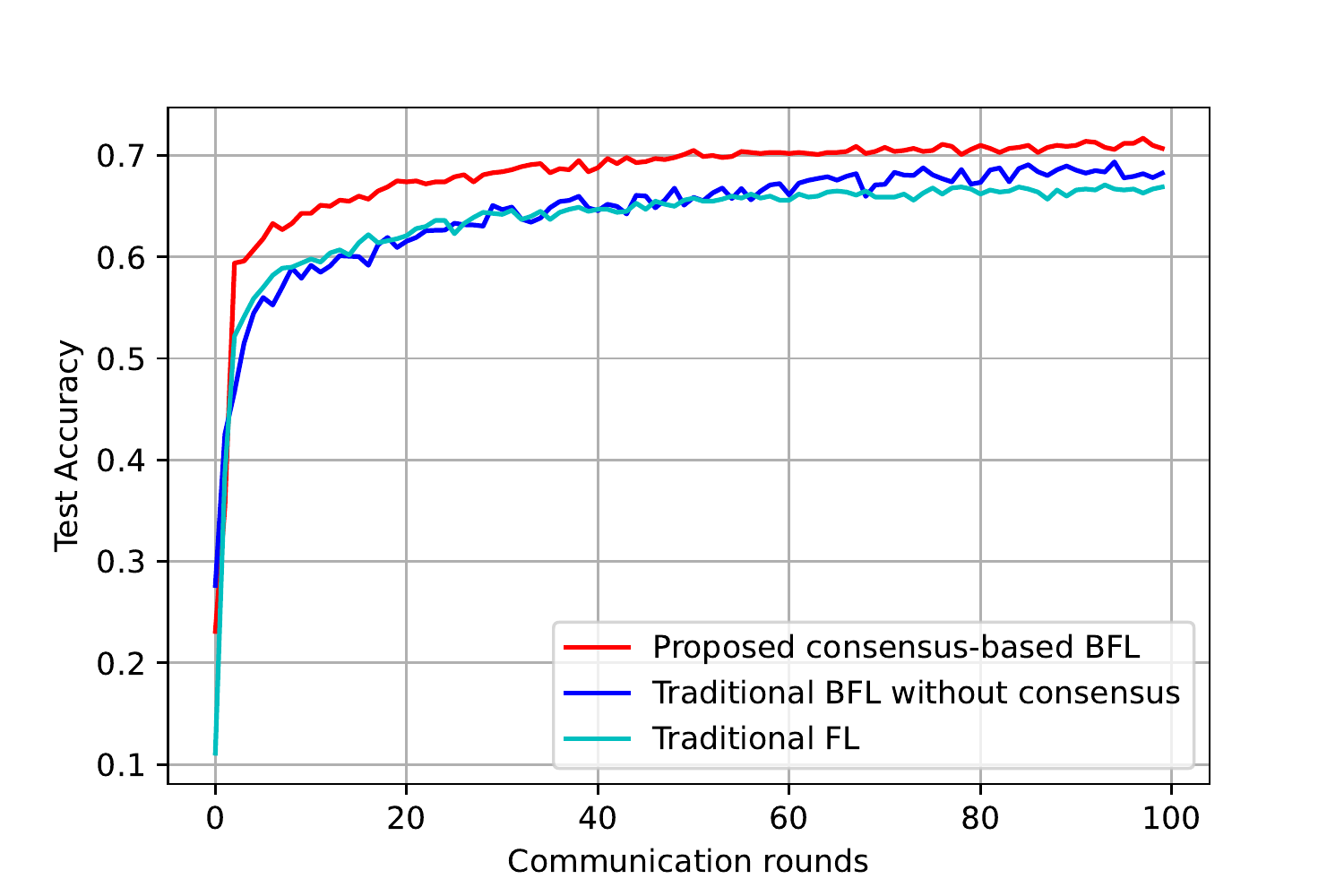} % 7-8.8
		\caption{Accuracy under non-IID setting.  }
	\end{subfigure}
	~
	\begin{subfigure}[t]{0.24\textwidth}
		\centering
		\includegraphics[width=0.99\linewidth]{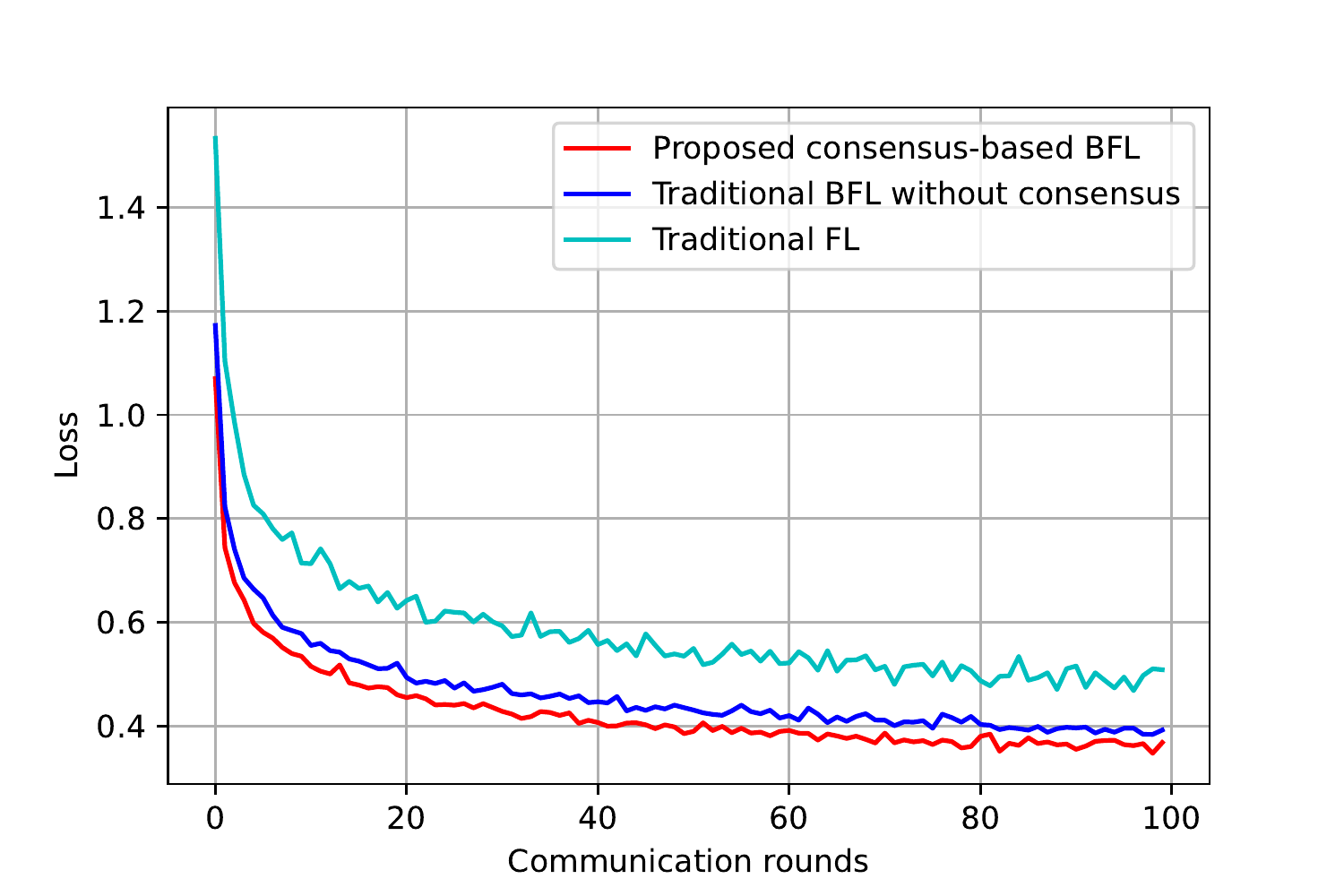} 
		\caption{Loss  under IID setting. }
	\end{subfigure}%
	~
	\begin{subfigure}[t]{0.24\textwidth}
		\centering
		\includegraphics[width=0.99\linewidth]{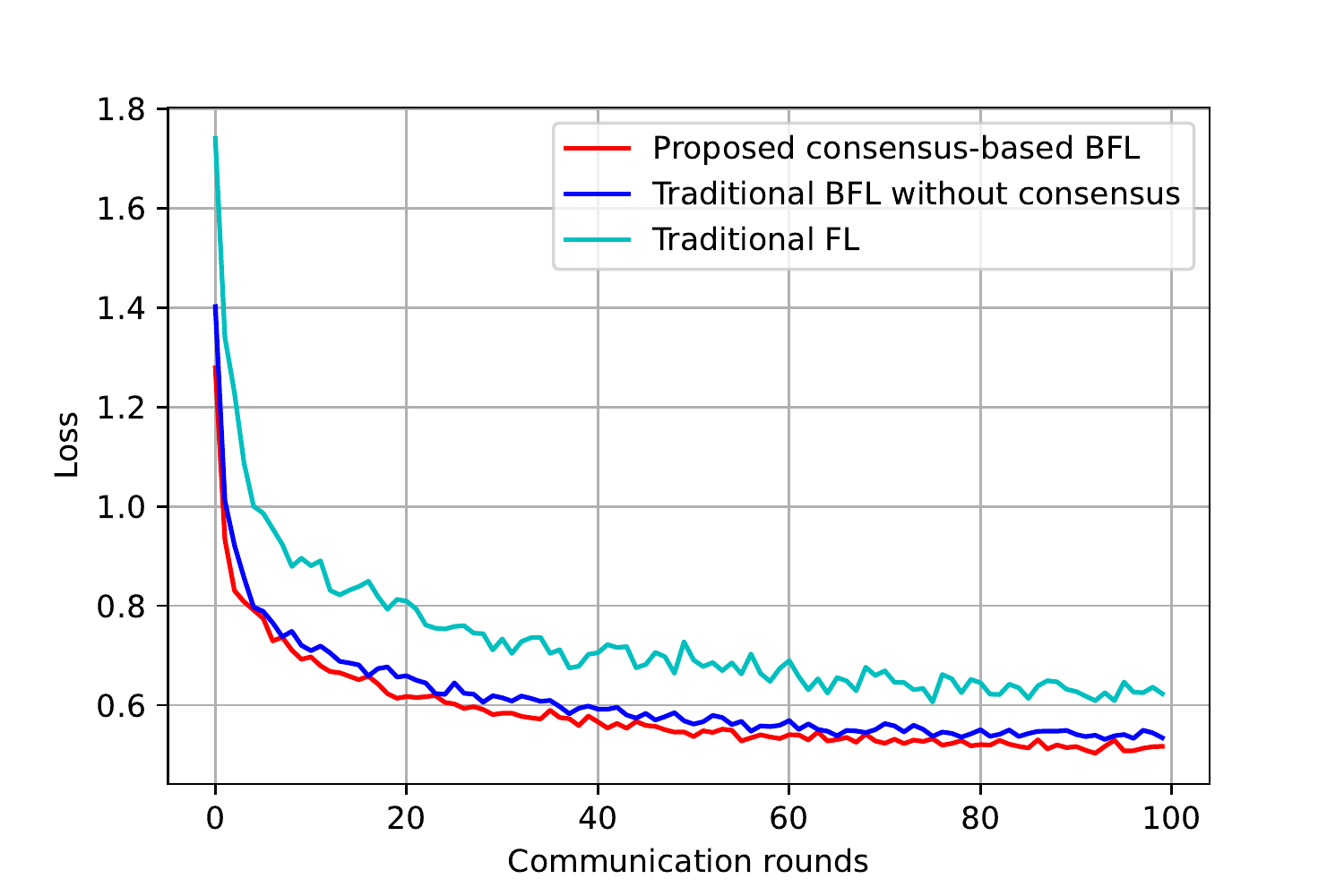} % 7-8.8
		\caption{Loss  under non-IID setting.   }
	\end{subfigure}
	\caption{{Comparison of different FL approaches on Fashion-MNIST dataset.}}
	\label{FL-MNIST_Result}
	\vspace{-0.1in}
\end{figure}
\subsection{{Evaluation of FL Performance}} 
{ We evaluate the classification accuracy and loss performance of our proposed BFL scheme (with 5 consensus rounds) and compare it with two related schemes. The first one is \textcolor{black}{a traditional FL scheme \cite{2}, where ESs work independently and each  only aggregates the models of its associated MDs and then broadcasts the resulting model across devices.  
% a single server working with 4 MDs to build the global model via federated averaging
The second one is a traditional BFL without P2P consensus \cite{5}, in which  each ES runs an averaging algorithm by collaborating with its MDs to build its aggregated model, and then a random ES is selected as a leader that builds a global model based on its local aggregated model without conducting P2P consensus.} Fig.~\ref{FL-compare-SVHN_Result} illustrates the performance when training the SVHN dataset, showing the considerable improvements in terms of higher accuracy and lower loss compared with the counterparts. Although the performance degrades when the dataset becomes non-IID, our consensus-based BFL scheme still outperforms other algorithms. The BFL scheme without consensus achieves a better training performance than the traditional FL scheme since its randomized leader ES selection avoid local model bias across the clients. Moreover, the performance gap between our scheme and the others becomes larger in the non-IID case which demonstrates the benefit of consensus-based model aggregation over existing approaches. The advantages of our scheme are also verified on the Fashion-MNIST dataset in both IID and non-IID settings, as indicated in Fig.~\ref{FL-MNIST_Result}. For example, in the IID setting, our consensus-based BFL scheme improves the accuracy rate by 8\% and 14\% in comparison with the traditional BFL scheme without consensus and the traditional FL scheme, respectively. }

{ Fig.~\ref{FL-compare-consensuses_Result} investigates the impact of P2P consensus rounds (i.e., 5, 10, and 15 rounds) on the learning performance. We can see that the increase of consensus rounds significantly improves the performance. Under the non-IID setting, the role of consensus on the model training becomes more significant with a larger performance gap, for example, between 5 rounds and 15 rounds, which shows the efficiency of our consensus-based BFL design for federated model training. }
 
\begin{figure}[t!]
	\centering
	\begin{subfigure}[t]{0.24\textwidth}
		\centering
		\includegraphics[width=0.99\linewidth]{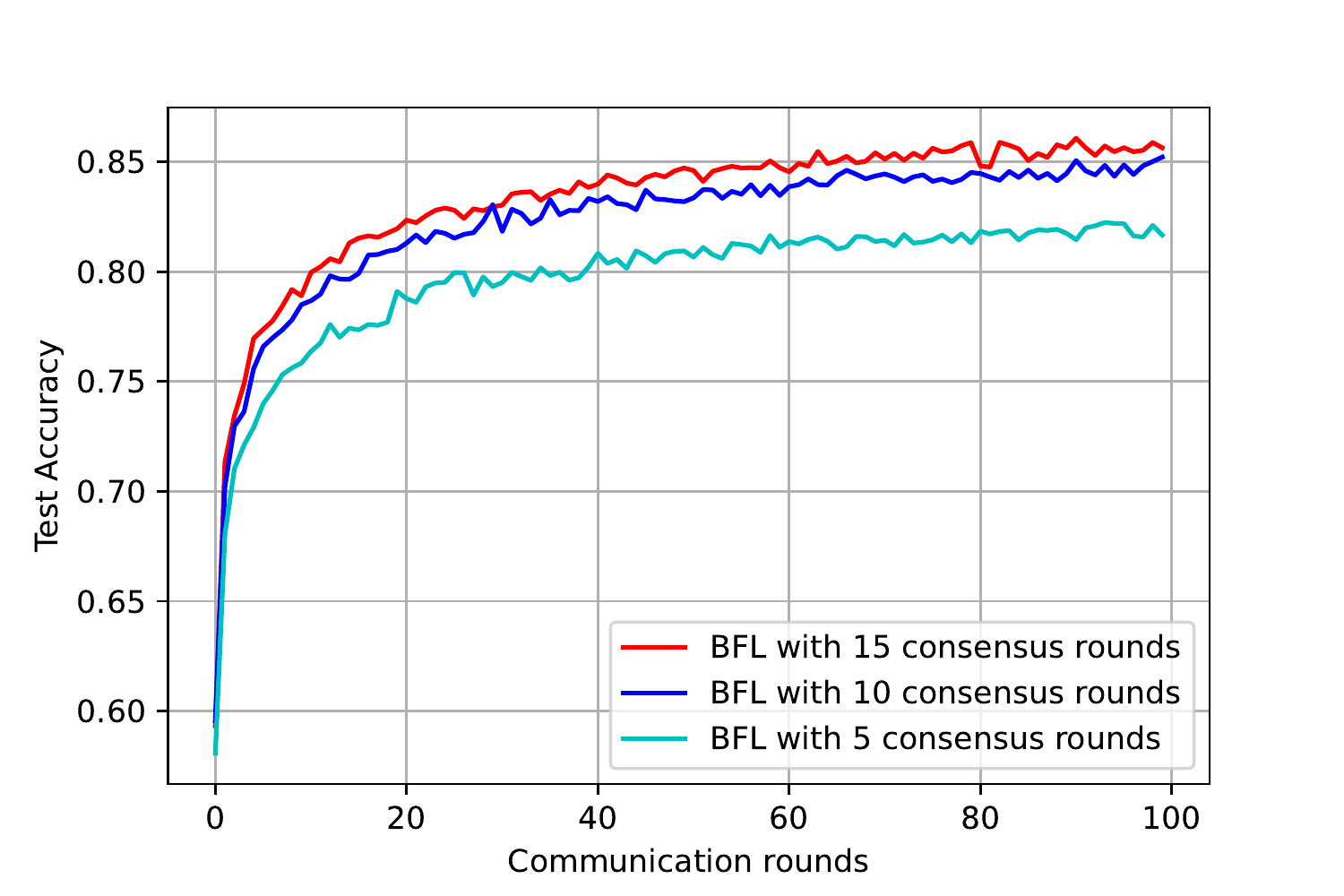} 
		\caption{Accuracy under IID setting. }
	\end{subfigure}%
	~
	\begin{subfigure}[t]{0.24\textwidth}
		\centering
		\includegraphics[width=0.99\linewidth]{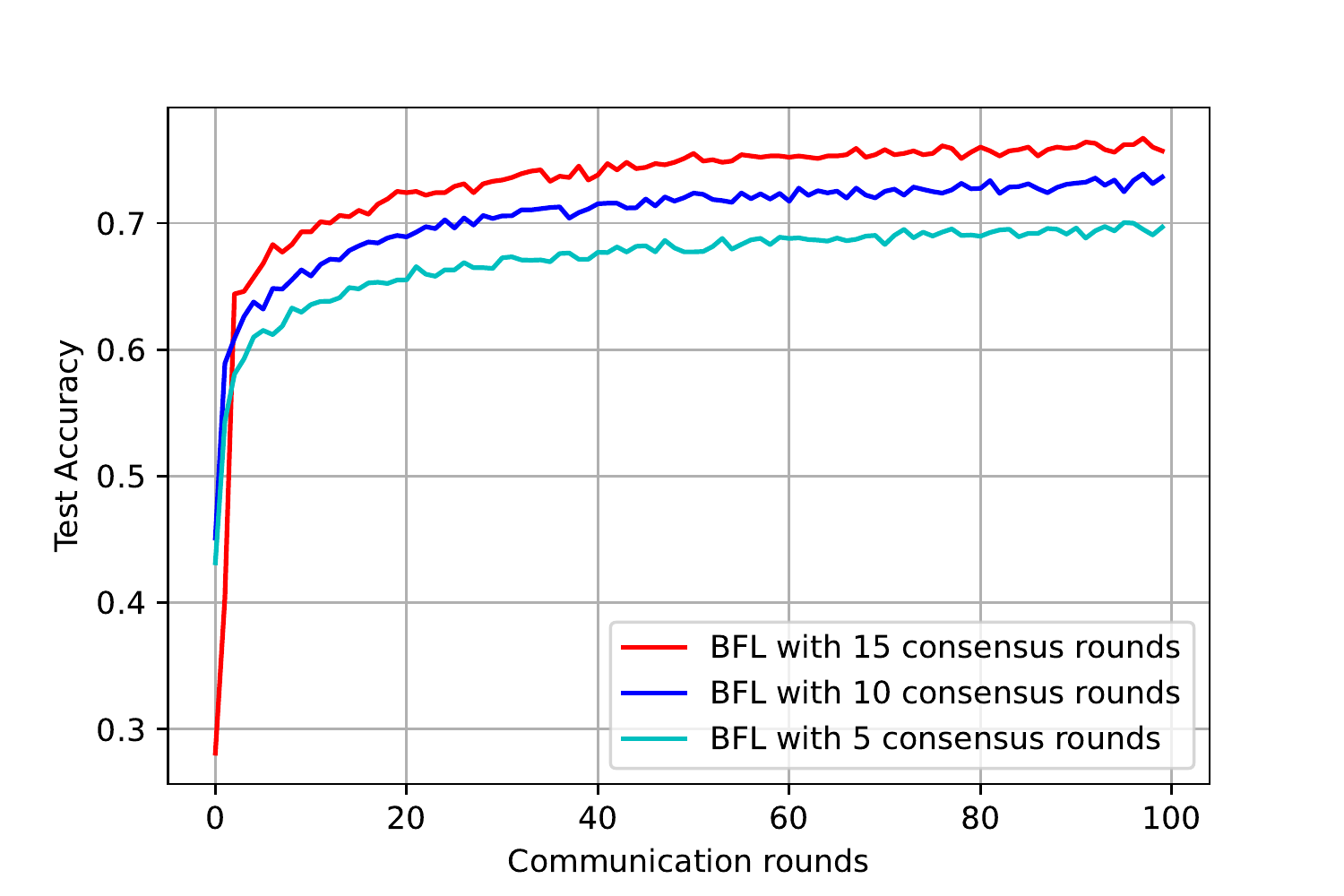} % 7-8.8
		\caption{Accuracy under non-IID setting.  }
	\end{subfigure}
	~
	\begin{subfigure}[t]{0.24\textwidth}
		\centering
		\includegraphics[width=0.99\linewidth]{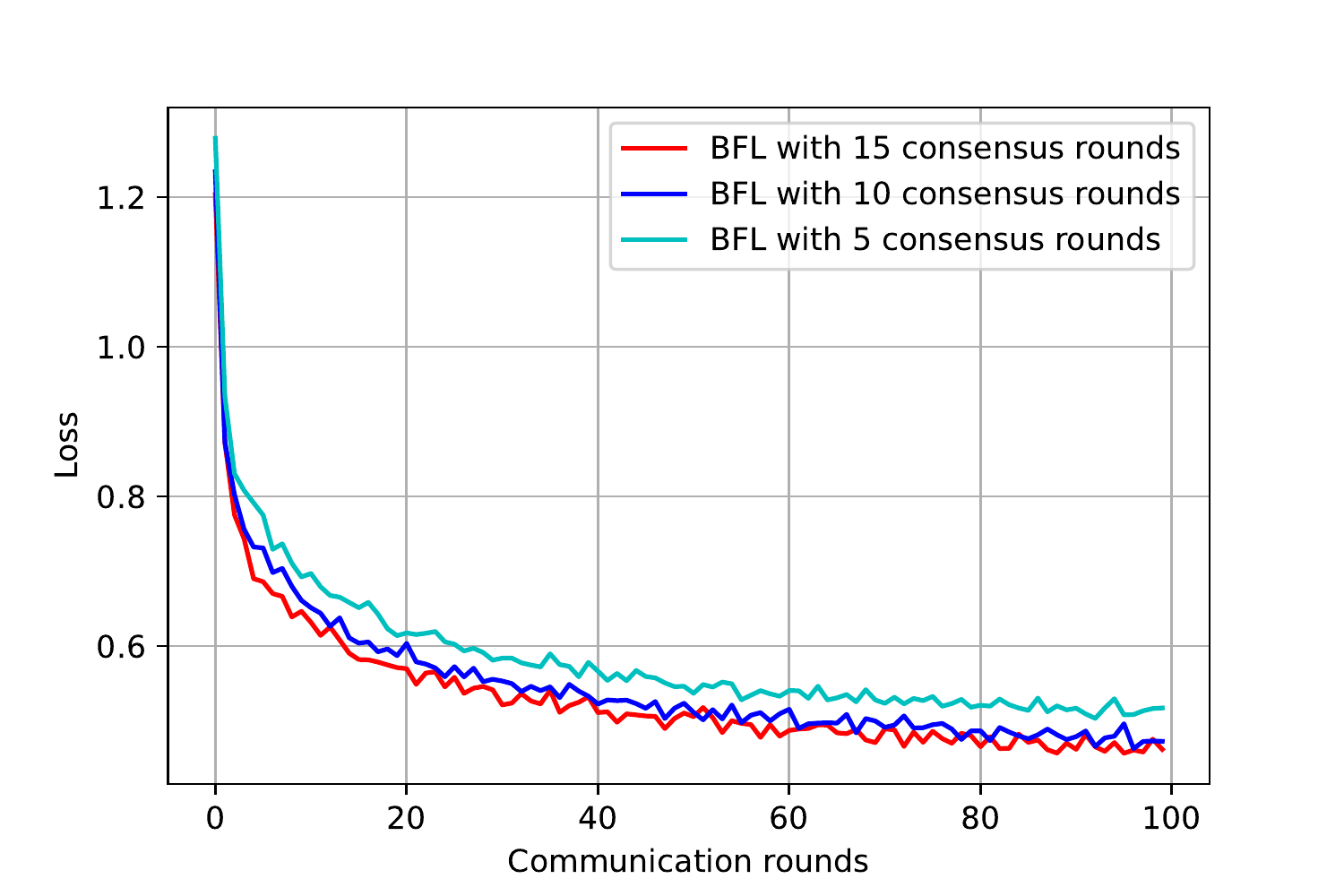} 
		\caption{Loss  under IID setting. }
	\end{subfigure}%
	~
	\begin{subfigure}[t]{0.24\textwidth}
		\centering
		\includegraphics[width=0.99\linewidth]{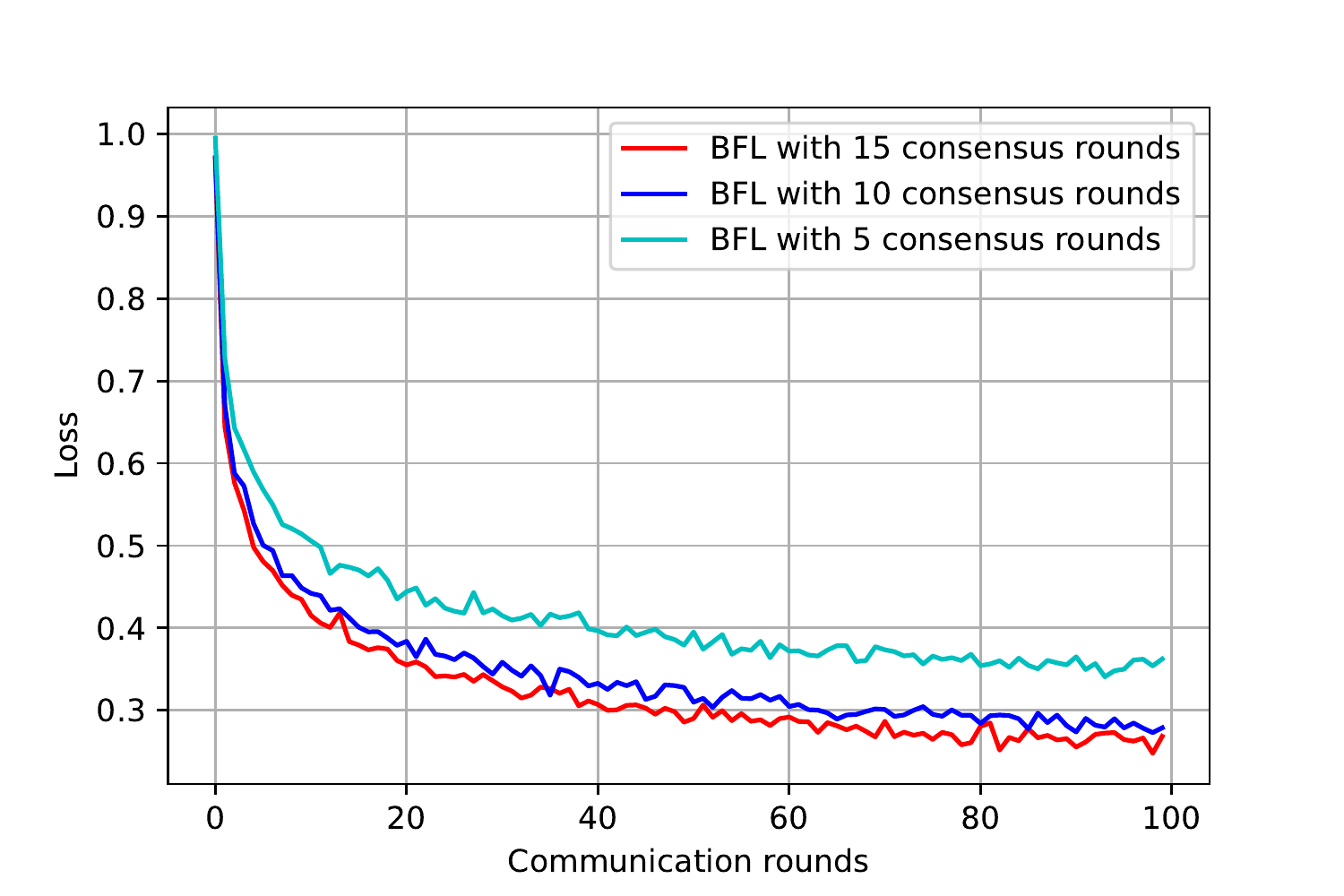} % 7-8.8
		\caption{Loss  under non-IID setting.   }
	\end{subfigure}
	\caption{{Comparison of BFL performance with different consensus rounds on Fashion-MNIST dataset.}}
	\label{FL-compare-consensuses_Result}
	\vspace{-0.1in}
\end{figure}

\subsection{Evaluation of DRL Training Performance}
\label{subsection:DRLtraining}
\begin{figure}[t!]
	\centering
	\begin{subfigure}[t]{0.25\textwidth}
		\centering
		\includegraphics[width=0.99\linewidth]{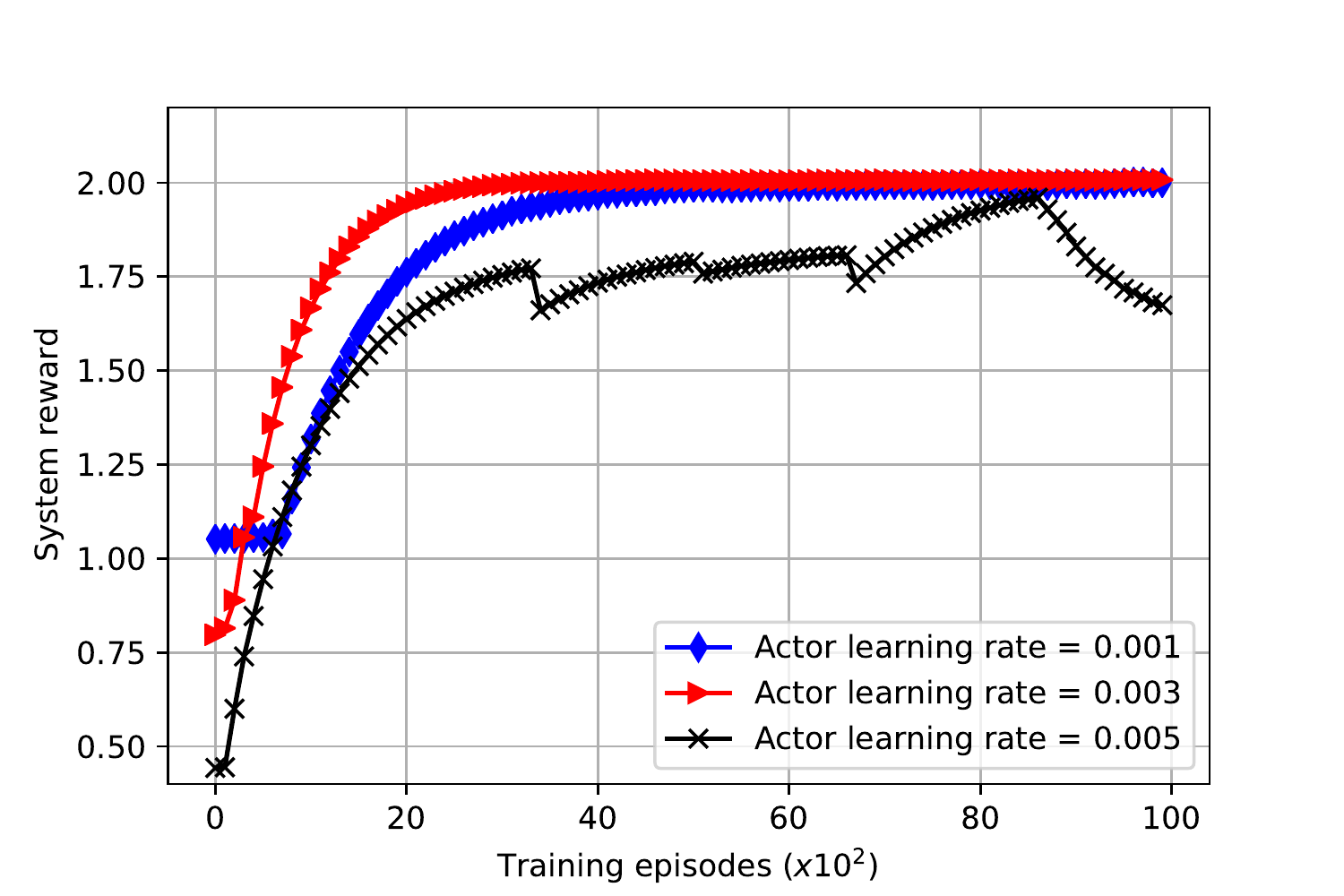} 
		\caption{System rewards with different actor learning rates. }
	\end{subfigure}%
	~
	\begin{subfigure}[t]{0.25\textwidth}
		\centering
		\includegraphics[width=0.99\linewidth]{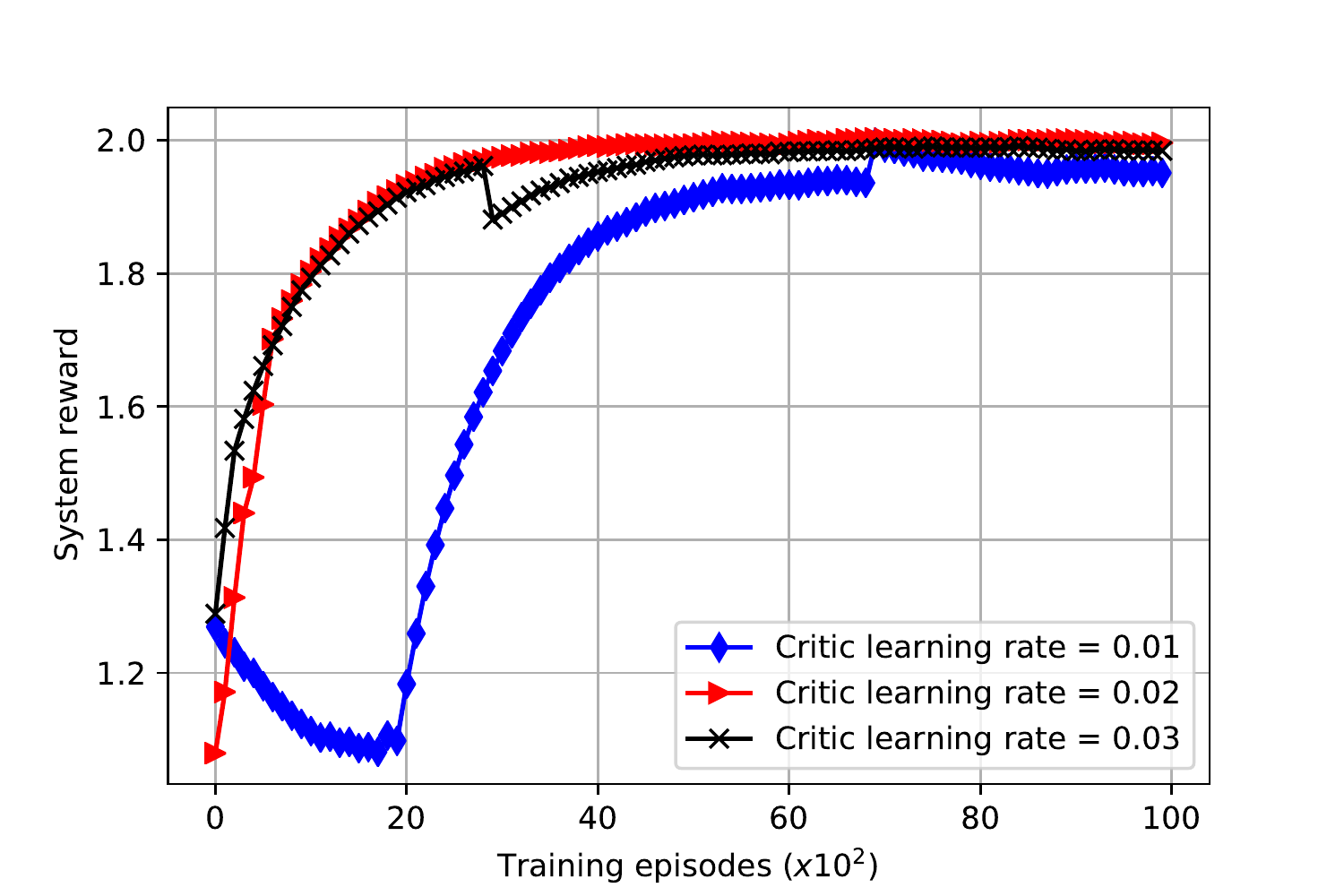} % 7-8.8
		\caption{System rewards with different critic learning rates.   }
	\end{subfigure}
	\caption{Evaluation of the training performance.}
	\label{Convergence_Result}
	\vspace{-5mm}
\end{figure}

We first investigate the training system reward (i.e., system utility as defined in~\eqref{Equa:Optimization1}) performance  by changing the learning rates at both actor and critic, which is important to determine a reliable parameter set for our later simulations, as shown in Fig.~\ref{Convergence_Result}. If the learning rate is too small, the policy training probably requires long time to achieve an optimal solution. However, if we set a large learning rate, \textcolor{black}{the training may become unstable and possibly diverge}. In A2C algorithms, since the actor updates slower than the critic with a timestep, we set up the actor with a learning rate smaller than the critic's learning rate.   We first evaluate the impacts of the actor learning rate, i.e., the backtracking step-size parameter $\upgamma_t^a$ which controls the policy update in the trust region in each iteration of TRPO under the KL divergence constraint. 
% To successfully update the policy, it is well known that the control parameter $\upgamma_t^a$ needs to be updated accordingly by iteratively shrinking until a value small enough to guarantee a decrease in the objective function. 
Fig.~\ref{Convergence_Result}(a) reveals that the learning rate of 0.003 exhibits the highest system reward with the fastest convergence, compared to the case of $\upgamma_t^a = 0.001$ which shows the slowest convergence rate. However, when the learning rate is relatively high ($\upgamma_t^a = 0.005$), the learning process becomes unstable and diverges. Similar to the actor part, we also set up a learning procedure by considering various critic learning rates. As indicated in Fig.~\ref{Convergence_Result}(b), the learning rate $\upgamma_t^c = 0.02$ has the most stable training performance with a quick convergence, compared to other learning settings. Thus, we will use learning rates $\upgamma_t^a = 0.003$ and $\upgamma_t^c = 0.02$ for the following simulations.
% In what follows, to demonstrate the advantage of our proposed scheme, we compare with the other state-of-the-art schemes in various parameter settings for both reward and system latency metrices. 

% \begin{figure}
% 	\centering
% 	\includegraphics [width=0.99\linewidth]{Image/Image-2021-12-22-100852-Discrete-Actor.pdf}
% 	\caption{Comparison of system reward with different action spaces.} 
	
% \end{figure}

\begin{figure}[t!]
	\centering
	\begin{subfigure}[t]{0.25\textwidth}
		\centering
		\includegraphics[width=0.99\linewidth]{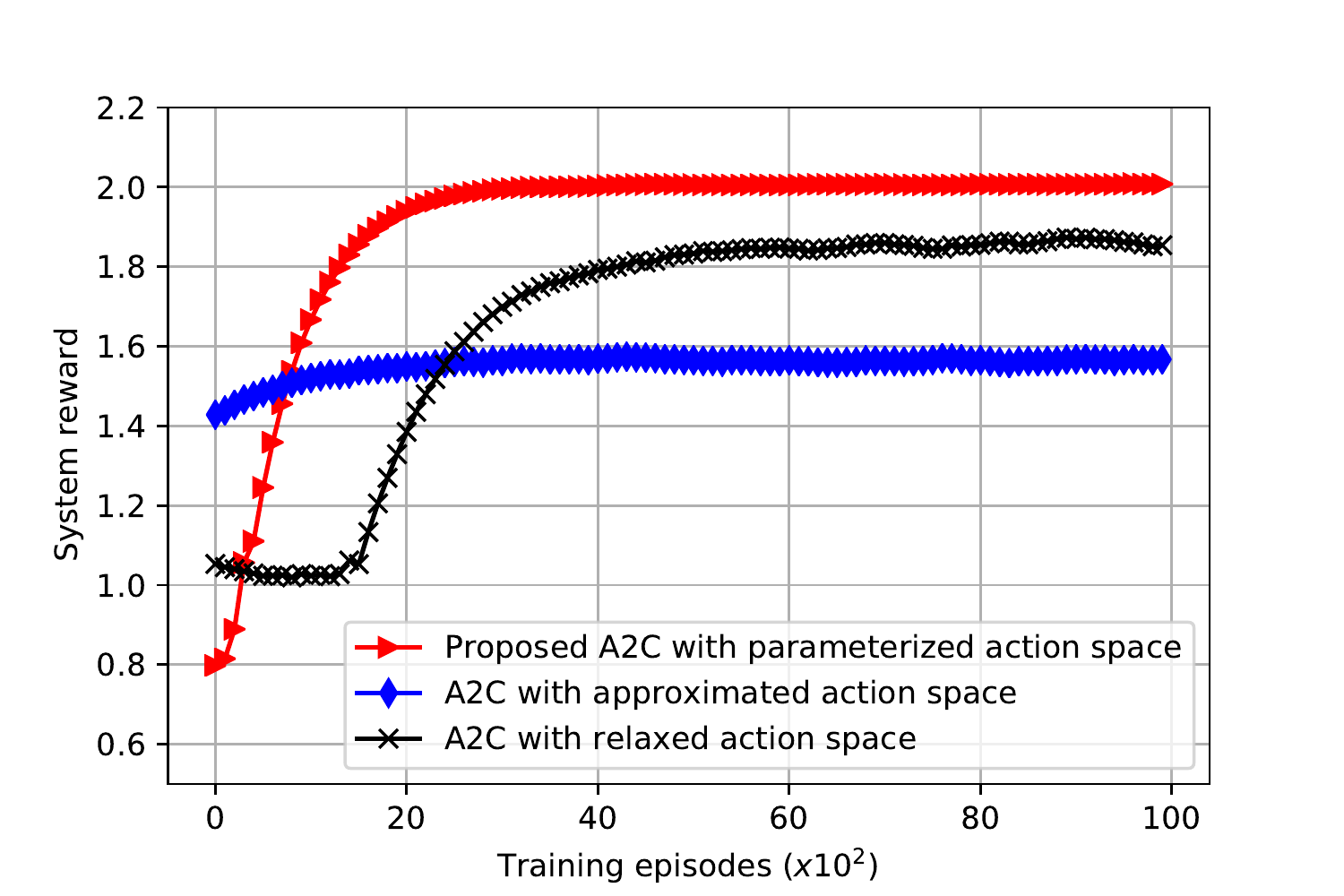}
		\caption{System reward with different action spaces.}
	\end{subfigure}%
	~
	\begin{subfigure}[t]{0.25\textwidth}
		\centering
		\includegraphics[width=0.99\linewidth]{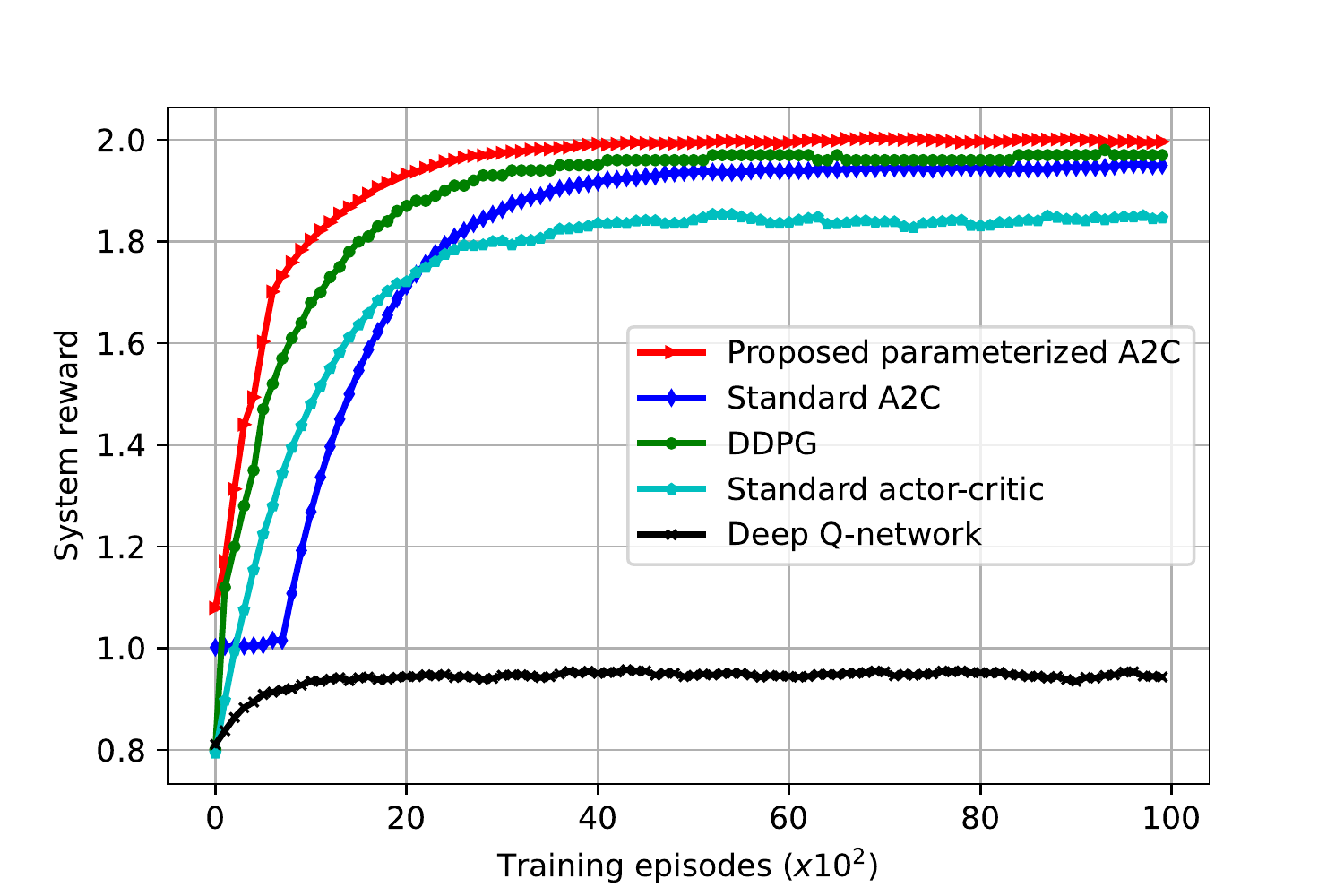}
		\caption{System reward with different learning schemes.  } 
	\end{subfigure}
	\caption{\textcolor{black}{Comparison of system reward via learning curves. }}
	\label{Fig:training_relaxed}
	\vspace{-0.1in}
\end{figure}

\begin{figure}[t!]
	\centering
	\begin{subfigure}[t]{0.25\textwidth}
		\centering
		\includegraphics[width=0.99\linewidth]{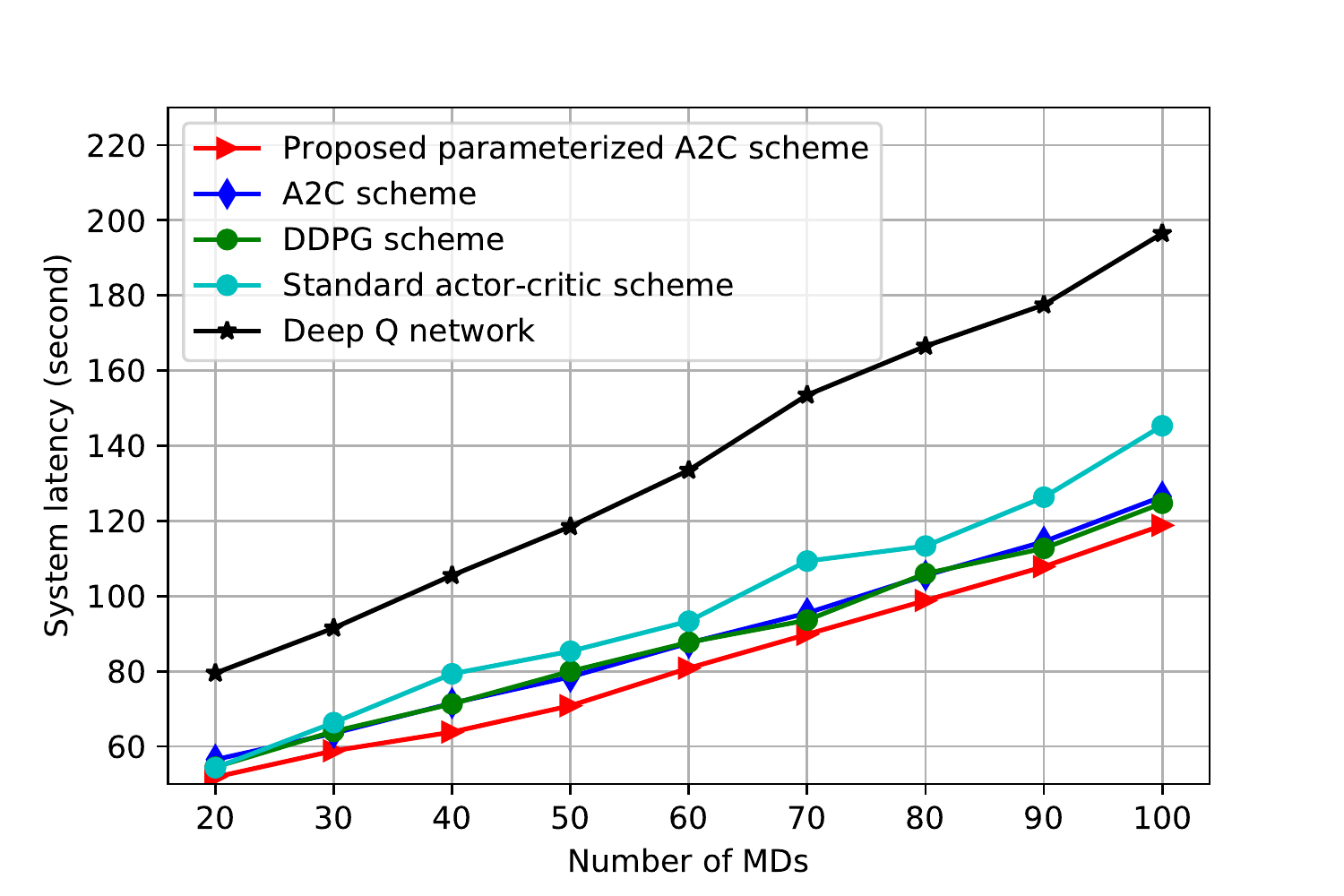} 
		\caption{System latency with different numbers of MDs. }
	\end{subfigure}%
	~
	\begin{subfigure}[t]{0.25\textwidth}
		\centering
		\includegraphics[width=0.99\linewidth]{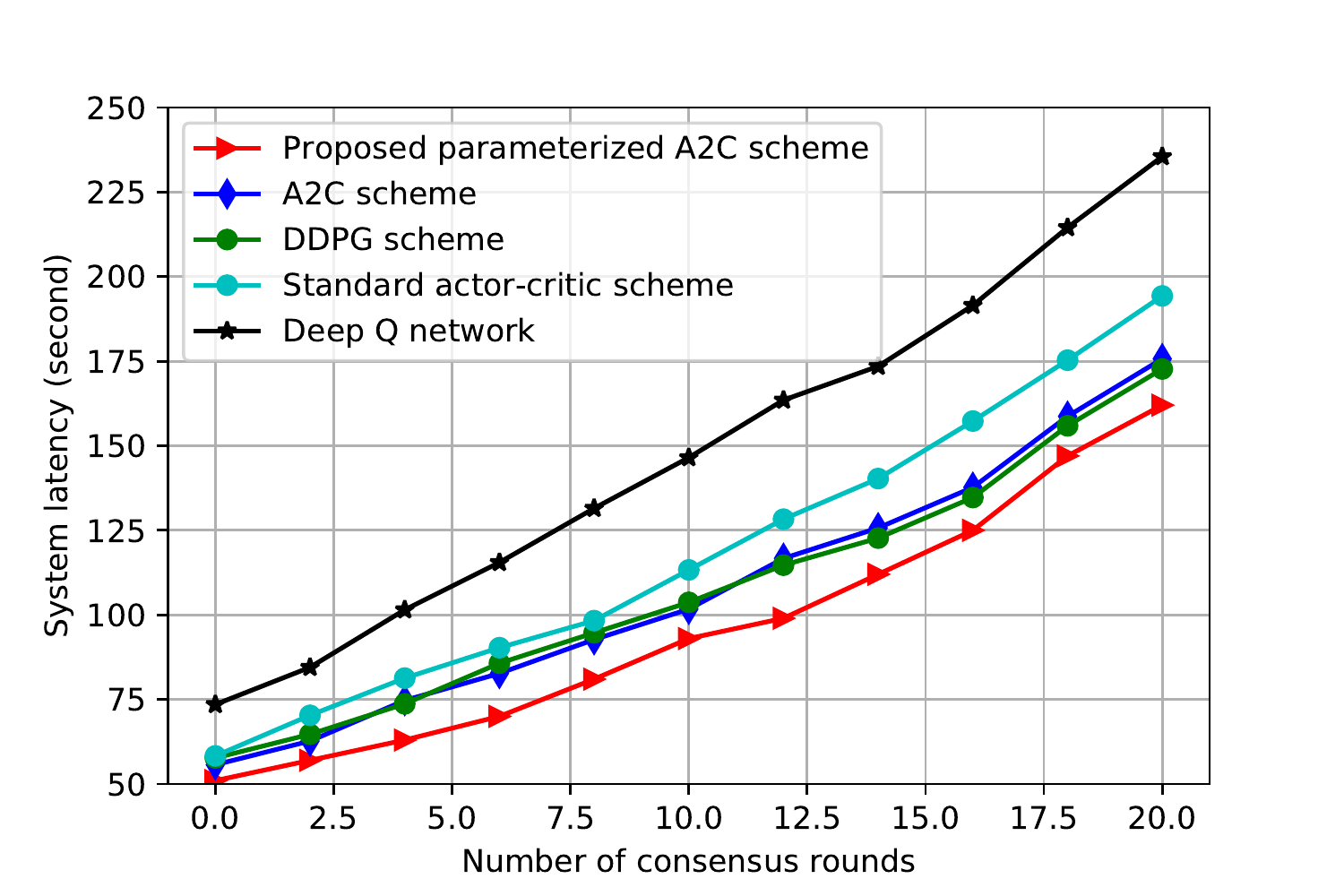} % 7-8.8
		\caption{System latency versus consensus rounds. }
	\end{subfigure}
	\caption{\textcolor{black}{Comparison of system latency between different schemes. }}
	\label{Fig:latency_consensus}
	\vspace{-0.1in}
\end{figure}
\textcolor{black}{Next, we compare the reward performance  obtained by our proposed parameterized A2C algorithm against two baseline schemes:} A2C with relaxed action space \cite{nguyen2021cooperative} and A2C with approximated action space \cite{30}. \textit{For the first baseline scheme}, we first relax the discrete offloading vector into a continuous set by defining a relaxed space as $\mathcal{A}_a = \{f(x_1,), f(x_2),..., f(x_N)\}$, where $f(.)$ is a probability softmax function, and then re-normalize them to approximate discrete offloading vectors for execution.  Compared with our proposed parameterized scheme, this method significantly increases the sampling complexity on the joint action space. \textit{For the second baseline scheme}, we discretize each of continuous allocation vectors into a discrete subset. For example, the transmit power variable of  MD $n$ is discretized into $Z$ levels as $\boldsymbol{P}_n = \left[0, p_{min}, p_{min}. \left(\frac{p_{max}}{p_{min}}\right) ^{\frac{1}{|Z|-2} },...,p_{max}\right]$,  where $p_n=0$  implies  the local execution mode. This discretization process creates a large number of quantization levels for convenient action sampling but also results in quantization noise.
\begin{figure*}[t!]
	\centering
	\begin{subfigure}[t]{0.32\textwidth}
		\centering
		\includegraphics[width=0.99\linewidth]{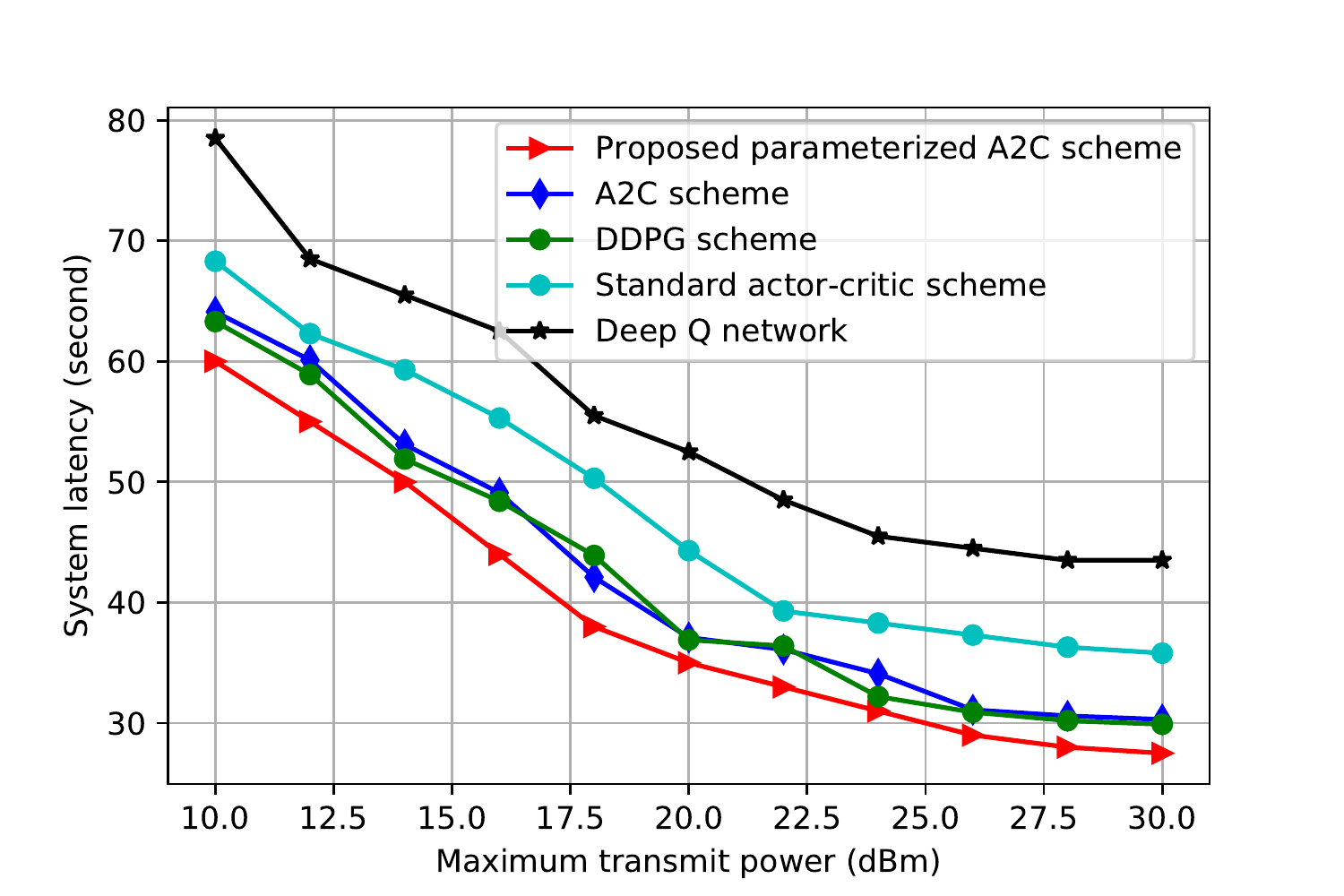} 
		\caption{System latency versus transmit power allocation.}
	\end{subfigure}%
	~
	\begin{subfigure}[t]{0.32\textwidth}
		\centering
		\includegraphics[width=0.99\linewidth]{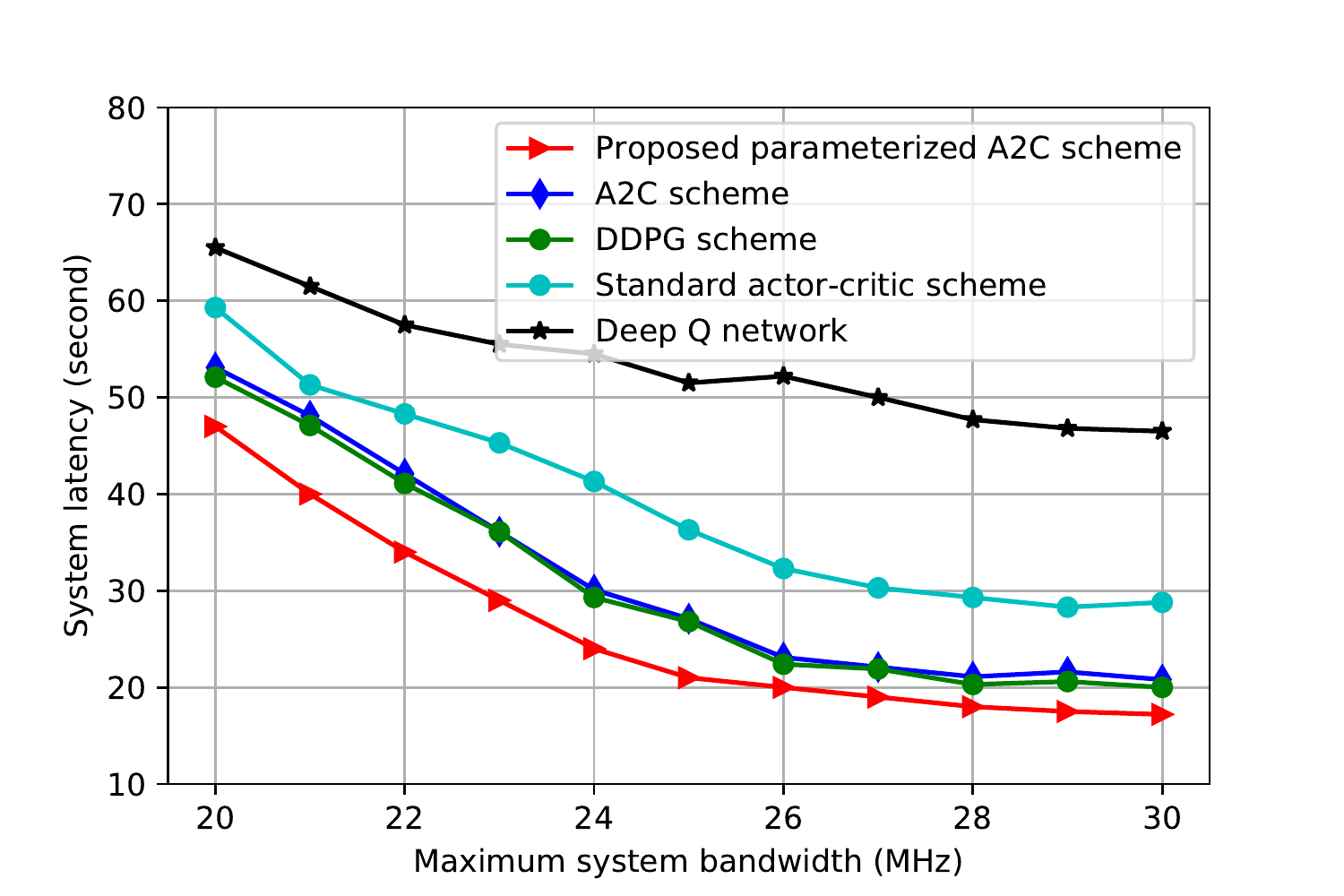} % 7-8.8
		\caption{System latency versus bandwidth allocation.   }
	\end{subfigure}
	~
	\begin{subfigure}[t]{0.32\textwidth}
		\centering
		\includegraphics[width=0.99\linewidth]{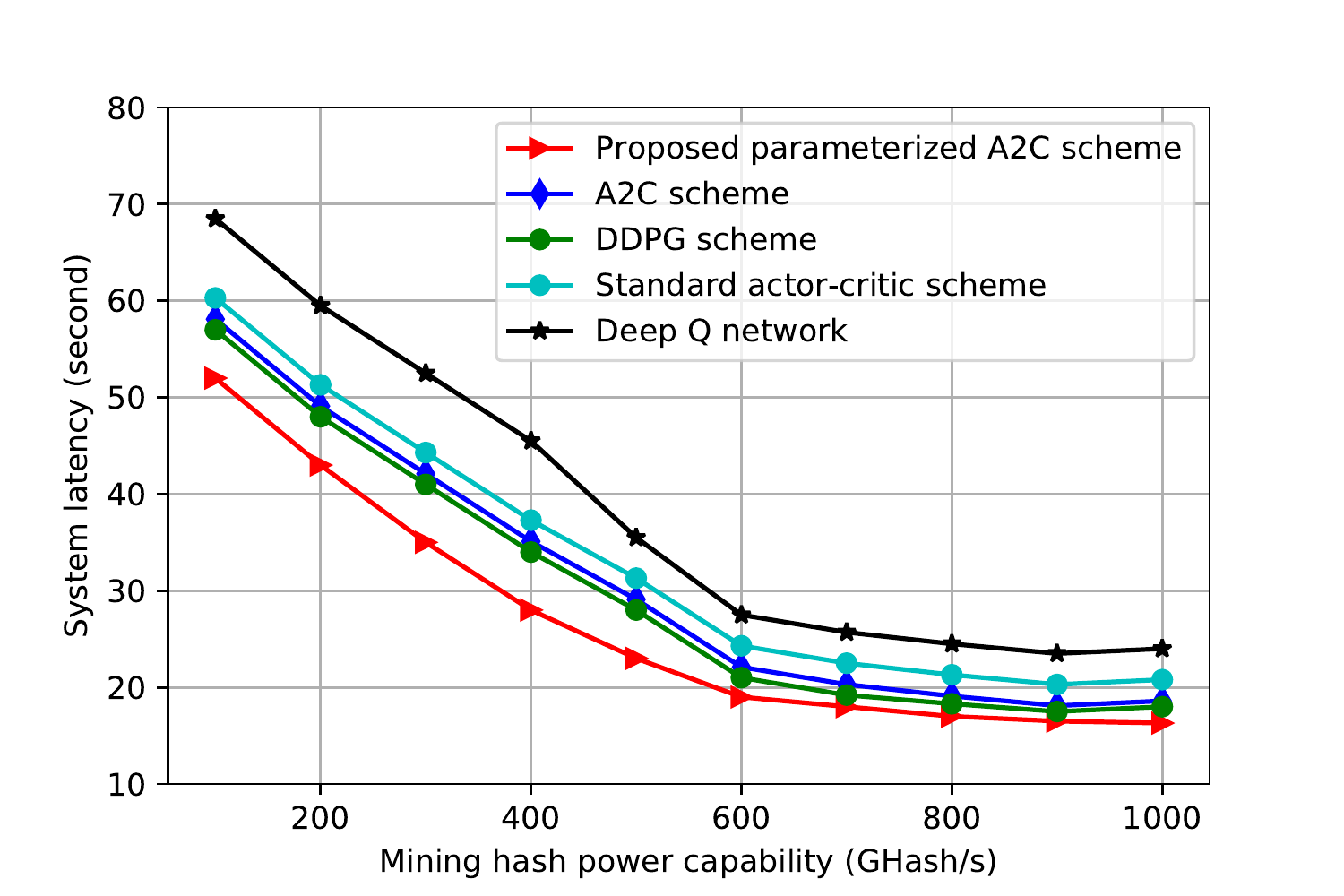} % 7-8.8
		\caption{System latency versus mining hash allocation.   }
	\end{subfigure}
	\caption{\textcolor{black}{Comparison of system latency with different resource allocation settings.}}
	\label{Fig:latency_allocation}
	\vspace{-6mm}
\end{figure*}

\textcolor{black}{From Fig.~\ref{Fig:training_relaxed}(a), our proposed scheme can achieve the best reward performance compared with baseline approaches,}  thanks to a flexible parameterized action sampling solution where both discrete offloading decisions and allocation variables are directly trained without relaxation or approximation. Meanwhile, the A2C scheme with relaxed action space suffers a reward decrease with high training variance since the offloading action selection must be converted into a continuous space of allocation vectors, leading to an extremely high complexity in the action sampling and thus making the training inefficient. Moreover, its complex action sampling requires longer time to reach convergence, i.e., after 4000 episodes compared to 2500 episodes in our proposed scheme. The lowest reward gain is observed at the approximated scheme, where the approximation of action space introduces the quantization error which degrades the policy training. 

% \begin{figure}
% 	\centering
% 	\includegraphics [width=0.99\linewidth]{Image/Image-2021-12-22-142900_compareDQN.pdf}
% 	\caption{Comparison of system reward with different learning schemes.  }  
% 	\label{Fig:training_compareDQN}
% \end{figure}
\textcolor{black}{We then compare our scheme with state-of-the-art DRL approaches: (i) A2C \cite{14}, as in our method but only using the vanilla gradient update at the actor without policy optimization improvement; (ii) deep deterministic policy gradient (DDPG), using actor-critic with deterministic policy training over the continuous action space \cite{nguyen2021cooperative}; (iii) standard actor-critic \cite{37}, i.e., without using advantage function in the policy estimation; and (iv) deep Q-network \cite{23}, which uses action sampling approximation.} From Fig.~\ref{Fig:training_relaxed}(b), we find that our proposed parameterized A2C method is better than baselines in terms of system reward and convergence rate and stability. This is due to the efficient action sampling and improved policy optimization based on trust region enforcement which helps avoiding possible gradient divergence. Thus, a faster and more reliable policy search is achieved toward an optimal solution. \textcolor{black}{The DDPG scheme has a lower reward convergence speed due to the relaxation of the discrete offloading action space. This increases action sampling complexity and thus makes the policy training less efficient compared with our parameterized A2C. On the other hand, we see that DDPG outperforms standard A2C, which demonstrates the importance of the policy optimization improvement step employed in our scheme over a hybrid continuous and discrete action space.} Although the standard A2C scheme achieves a relatively stable reward performance, its reward gain is lower than that of our approach due to the natural gradient descent with greedy policy update.  

% The deep Q-network shows the lowest reward performance caused by the extensively approximated action sampling which suffers from large quantization noise. 
%The standard  actor-critic scheme suffers a high training variance due to the use of original Q value instead of using the advantage between the Q-value and the state-value in our A2C scheme.

\vspace{-4mm}
\subsection{Evaluation of System Latency Performance}

In this subsection, we evaluate the system latency performance  under networking scenarios. Based on the system reward (utility) computed via the above training process, it is straightforward to calculate the system latency via their mathematical relation as mentioned in \eqref{MiningUtility} and~\eqref{Equa:Optimization1}. 

% \begin{figure}
% 	\centering
% 	\includegraphics [width=0.99\linewidth]{Image/Image-2021-12-23-202042_MDs.pdf}
% 	\caption{Comparison of system latency with different numbers of MDs.  }  
% 	\label{Fig:latency_MDs}
% \end{figure}
\textcolor{black}{We investigate the latency performance when varying the numbers of MDs from 20 to 100 in Fig.~\ref{Fig:latency_consensus}(a). The system latency obtained by our proposed parameterized A2C scheme is the lowest across the considered methods.  As expected, the latency of each method increases with the number of MDs, due to a higher offloading and mining latency caused by a higher competition on bandwidth and hash allocation among MDs, our proposed scheme still achieves the best latency performance when the number of MDs increases. For instance, with 100 MDs, the system latency of our scheme is 11\%, 13\%, 25\% and 38\% lower than that of the DDPG, A2C, actor-critic, and deep Q-network schemes, respectively. }

% \begin{figure}
% 	\centering
% 	\includegraphics [width=0.99\linewidth]{Image/Latency-consensus-155423.pdf}
% 	\caption{{Comparison of system latency versus consensus rounds. }}  
% 	\label{Fig:latency_consensus}
% \end{figure}

We then investigate the impact of P2P rounds on the system latency, as indicated in Fig.~\ref{Fig:latency_consensus}(b). We change the number of P2P rounds from $0$ to $20$, where 0 implies the traditional BFL scheme without consensus. Although  the use of consensus procedure at the edge layer enhances the model training performance, it potentially increases the system latency \textcolor{black}{due to the additional delay of P2P model aggregation among ESs. However, with our advanced DRL design, the proposed parameterized A2C scheme achieves the minimal latency compared with other baselines. The increase of P2P rounds requires more time} for model aggregation, leading to a higher system latency, but our approach still has the best performance.

We also investigate the latency performance under different resource allocation scenarios in Fig.~\ref{Fig:latency_allocation} in a BFL environment with 5 ESs and 30 MDs. We first vary the maximum transmit power $P_n$ at each MD $n$ between 10 and 30 dBm and then measure the latency of offloading and mining. As shown in Fig.~\ref{Fig:latency_allocation}(a), the latency of all schemes decreases with respect to the increasing amount of allocated transmit power. This is because a larger transmit power improves the data transmission rate that helps reduce the data offloading latency. Notably, compared with other schemes, our method achieves the most significant and stable latency decrease due to our efficient and robust parameterized policy training to obtain the better offloading trajectory. Another interesting observation is that when the maximum transmit power exceeds 20 dBm, the system latency reduces at a slower rate. This is due to the less impact of such a large power allocation on the offloading latency savings given a certain number of MDs and data sizes.

\textcolor{black}{Moreover, we investigate the latency trends of different algorithms when the maximum system bandwidth $W$ varies from 20 to 30 MHz in Fig.~\ref{Fig:latency_allocation}(b). Similar to the power allocation scenario, more bandwidth would help mitigate the offloading latency. In our simulation setting, the impact of bandwidth on the system latency is significant before the threshold of 25 MHz, where our scheme shows the best latency savings. Compared with DDPG, A2C and actor-critic schemes, our parameterized A2C scheme can reduce the latency by 7\%, 8\% and 17\%, respectively, and achieves a 60\% lower latency compared with the deep Q-network baseline. }

\textcolor{black}{Finally, we compare the latency performance with respect to the changes of mining hash power capability $\Psi_{\mathsf{max}}$. A higher hash allocation would help reduce the block generation latency for a lower mining and system latency.} By increasing the hash budget, each MD has a higher chance to obtain  sufficient hash power to run the block mining. As can be seen in  Fig.~\ref{Fig:latency_allocation}(c), our proposed scheme outperforms other baselines in terms of system latency in each hash allocation setting. \textcolor{black}{The simulation results thus demonstrate the efficiency of our proposed A2C algorithm design in the system latency minimization.}

\textcolor{black}{\subsection{Attack Evaluation}}

\textcolor{black}{\subsubsection{Attack Model and Security Analysis with Blockchain}
A major concern in conventional FL with a single centralized aggregation server is that the server may become compromised via model poisoning attacks launched by adversaries. In this section, we investigate how blockchain can help mitigate  impact of model poisoning attacks on the ML model training performance, which as discussed in Section~\ref{SectionI-motivate} is one of our motivations for integrating blockchain with FL.  Model poisoning attacks mainly consist of untargeted model poisoning attacks and targeted model poisoning attacks \cite{fang2020local}. The former category aims to degrade the accuracy of the global model training, whereas the latter aims to control the model deviation towards their target. Here, we focus on untargeted model poisoning attacks on our BFL system. }

\textcolor{black}{We assume that at a certain global FL round $k$, an ES can be compromised by a  model poisoning attack following two steps. First, the attacker injects certain random noise  $v^{(k)}$ to the aggregated global model $ \boldsymbol{w}^{(k)}$ obtained via~\ref{equa:global_update-final} to manipulate the global model update  $\boldsymbol{w}^{(k)} \leftarrow \boldsymbol{w}^{(k)} + v^{(k)}$. Here we adopt Gaussian noise, i.e., $v^{(k)} = - \varpi \boldsymbol{q}^{(k)}$, where $\varpi$ is a scaling factor which characterizes the magnitude of the compromised model, and $\boldsymbol{q}^{(k)}$ is a random vector sampled from the Gaussian distribution $\mathcal{N}(\boldsymbol{0},\boldsymbol{I})$. Next, the attacked ES broadcasts the compromised global model to local MDs.} 

\textcolor{black}{In our BFL system, blockchain replaces the centralized authority in FL with a decentralized tamper-proof data ledger to monitor the model consensus process as well as mitigate single-point-of-failure. By deploying a blockchain over the edge layer, any model update event over consensus rounds at a certain ES is automatically traced by other ESs. If a poisoning attack takes place at an ES, other ESs can detect this behavior via transaction logs. Here, we adopt the attack detection score metric \cite{mondal2022beas}, which characterizes the abnormal gradient deviation caused by the poisoning attacker at a certain FL round, to detect the occurrence of a  model attack at a certain ES. Since all gradient update information including the compromised gradient is recorded on the digital ledger, blockchain can identify the compromised ES and temporarily disregard it from the ES network, while the global model aggregation process continues.}

\subsubsection{\textcolor{black}{Attack Simulation}}
\textcolor{black}{We investigate the training performance of different FL methods under model poisoning attacks on the two considered datasets. We consider an attack scenario where an adversary compromizes an ES and deploys model poisoning attacks by injecting random noise at the global rounds of 20 and 60 during the model consensus process among ESs.  Our proposed consensus-based BFL scheme with blockchain is compared with other three baselines: consensus-based FL without blockchain, BFL without consensus, and traditional FL. For the consensus-based schemes, we consider that there are 5 consensus rounds and the attacker poisons the model at round 3. For the BFL scheme without consensus, the attacker poisons one of the ESs during the global model aggregation process. For the traditional FL scheme, the attacker poisons the centralized aggregator. }
 
 \textcolor{black}{As illustrated in Fig.~\ref{Fig:attack-SVHN} for the SVHN dataset, the accuracy performance of all schemes is dropped when the attack occurs (i.e., at aggregation rounds 20 and 60). However, our consensus-based BFL scheme achieves the best accuracy rate and highest robustness against the poisoning attack in both IID and non-IID data distribution settings. For instance, from Fig.~\ref{Fig:attack-SVHN}(a), we see that the accuracy of our scheme at the onset of an attack only drops by 3.7\%, as opposed to 7.6\%, 16.9\%, and 32.1\% in the consensus-based FL without blockchain, BFL without consensus, and traditional FL schemes, respectively. The traceable data ledger of blockchain records all the global updating behaviors of ESs across consensus rounds, which allows blockchain to rapidly detect the poisoning attack.  Since all ESs are connected on the shared ledger, the blockchain disregards this ES from the consensus process to protect model training. The consensus-based FL scheme without blockchain has the second highest performance. The poisoned model is involved in the consensus process, which degrades the overall model aggregation, though the impact is dampened since the poisoned model is only one participating in the consensus process. The other two non-consensus-based methods show lower accuracy and less robustness against the attack. The BFL scheme without consensus shows lower training robustness since the attack at one of the ESs significantly affects the global model aggregation at that particular server and its associated MDs. In the traditional FL scheme with a centralized server, the model training is degraded after the attack since the attacker directly poisons the only model aggregator.  We repeat these experiments on the Fashion-MNIST dataset  in Fig.~\ref{Fig:attack-MNIST} which indicates qualitatively similar results. For example, in Fig.~\ref{Fig:attack-MNIST}(a), the accuracy rate of our BFL scheme with blockchain only reduces by 2.67\%, compared with 5.6\% in the consensus-based FL without blockchain, 6.8\% in the BFL without consensus, and 16.7\% in the traditional FL scheme at the onset of poisoning attack.   }
 %The largest impact of the poisoning attack on the model training is shown in the traditional FL scheme, which again demonstrates that the necessity of a decentralized FL design like our blockchain-based BFL scheme.
\begin{figure}[t!]
	\centering
	\begin{subfigure}[t]{0.25\textwidth}
		\centering
		\includegraphics[width=0.99\linewidth]{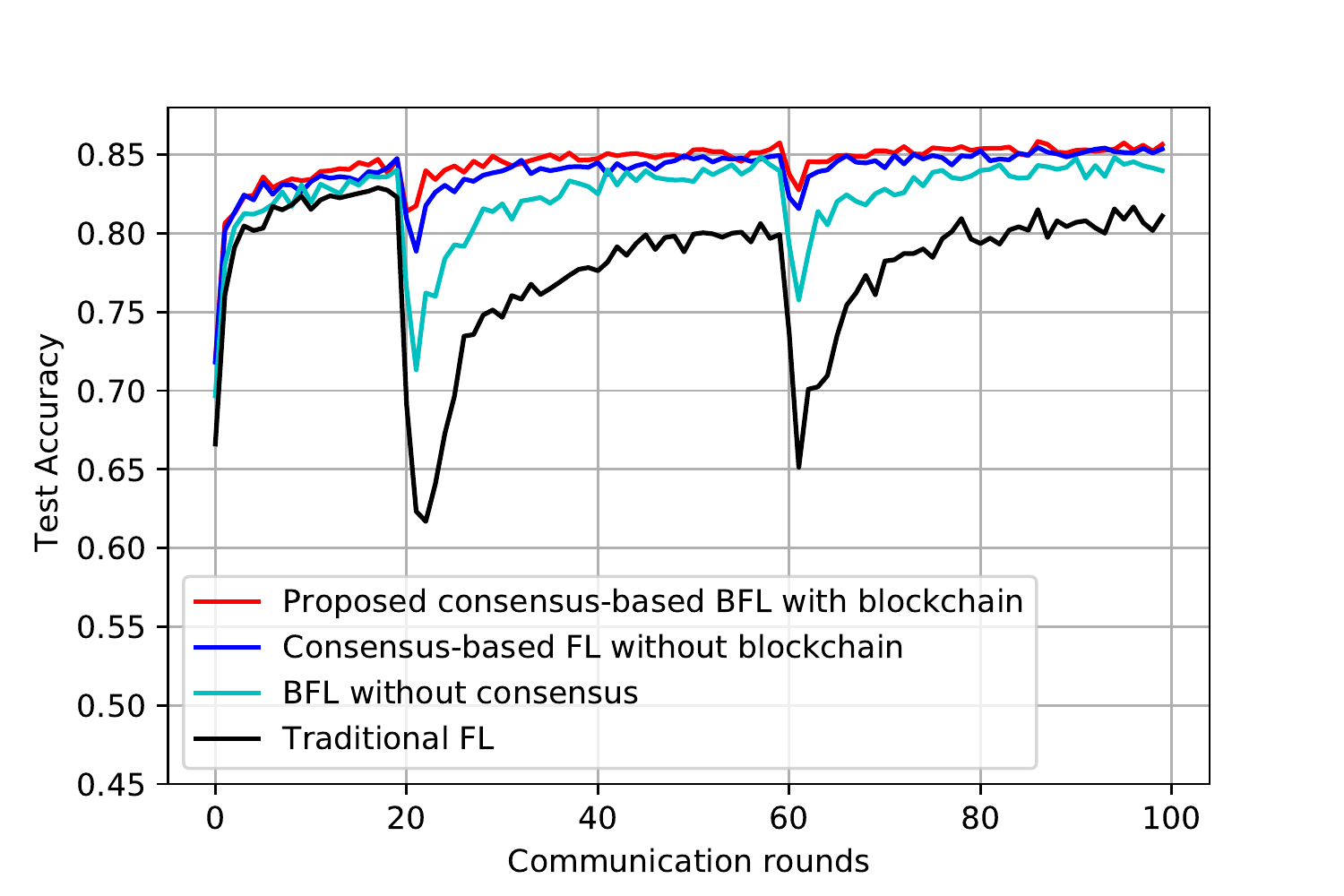} 
		\caption{Accuracy under IID setting. }
	\end{subfigure}%
	~
	\begin{subfigure}[t]{0.25\textwidth}
		\centering
		\includegraphics[width=0.99\linewidth]{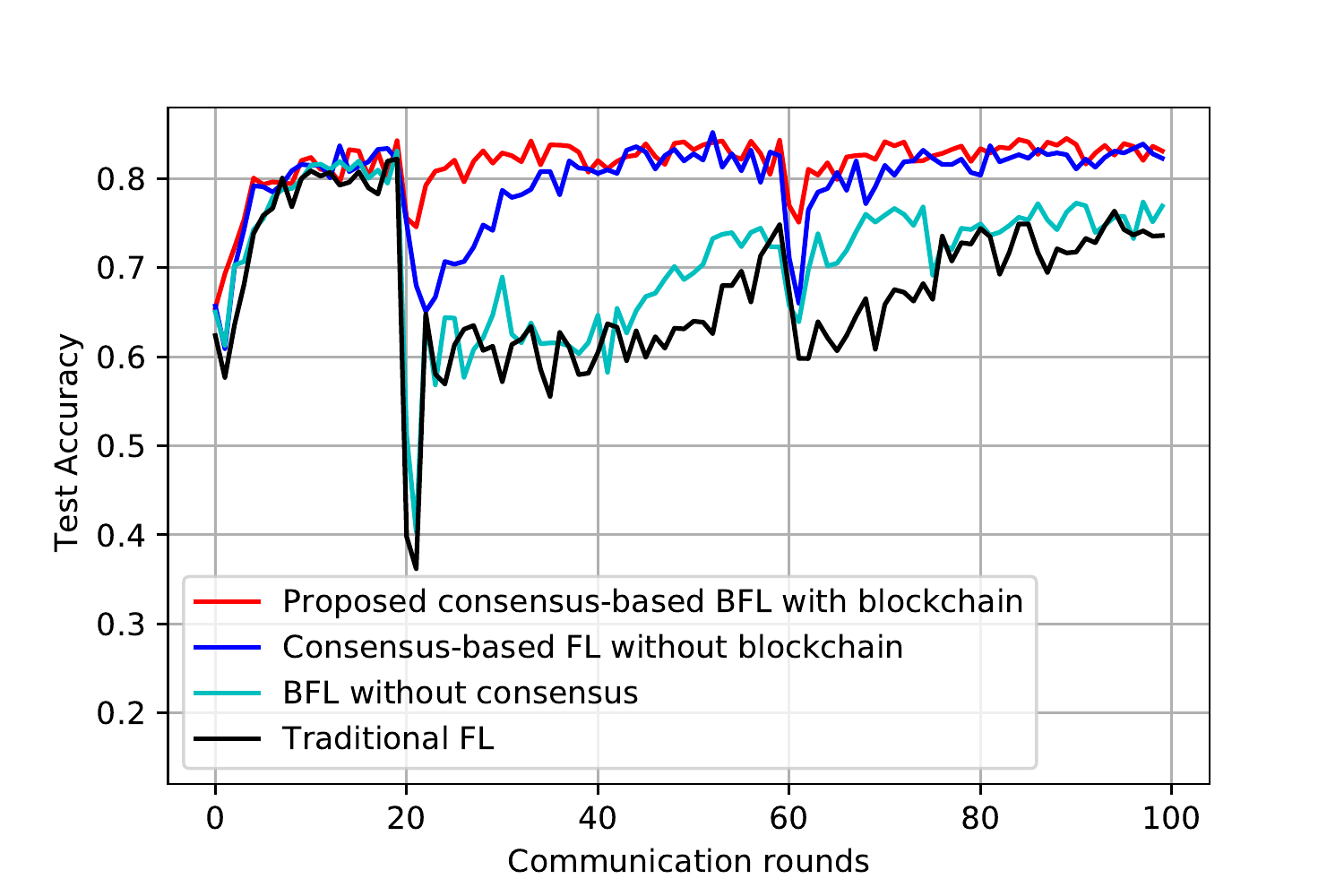} % 7-8.8
		\caption{Accuracy under non-IID setting. }
	\end{subfigure}
	\caption{\textcolor{black}{Comparison of different FL approaches under model poisoning attacks (SVHN dataset). }}
	\label{Fig:attack-SVHN}
	\vspace{-0.1in}
\end{figure}

\begin{figure}[t!]
	\centering
	\begin{subfigure}[t]{0.25\textwidth}
		\centering
		\includegraphics[width=0.99\linewidth]{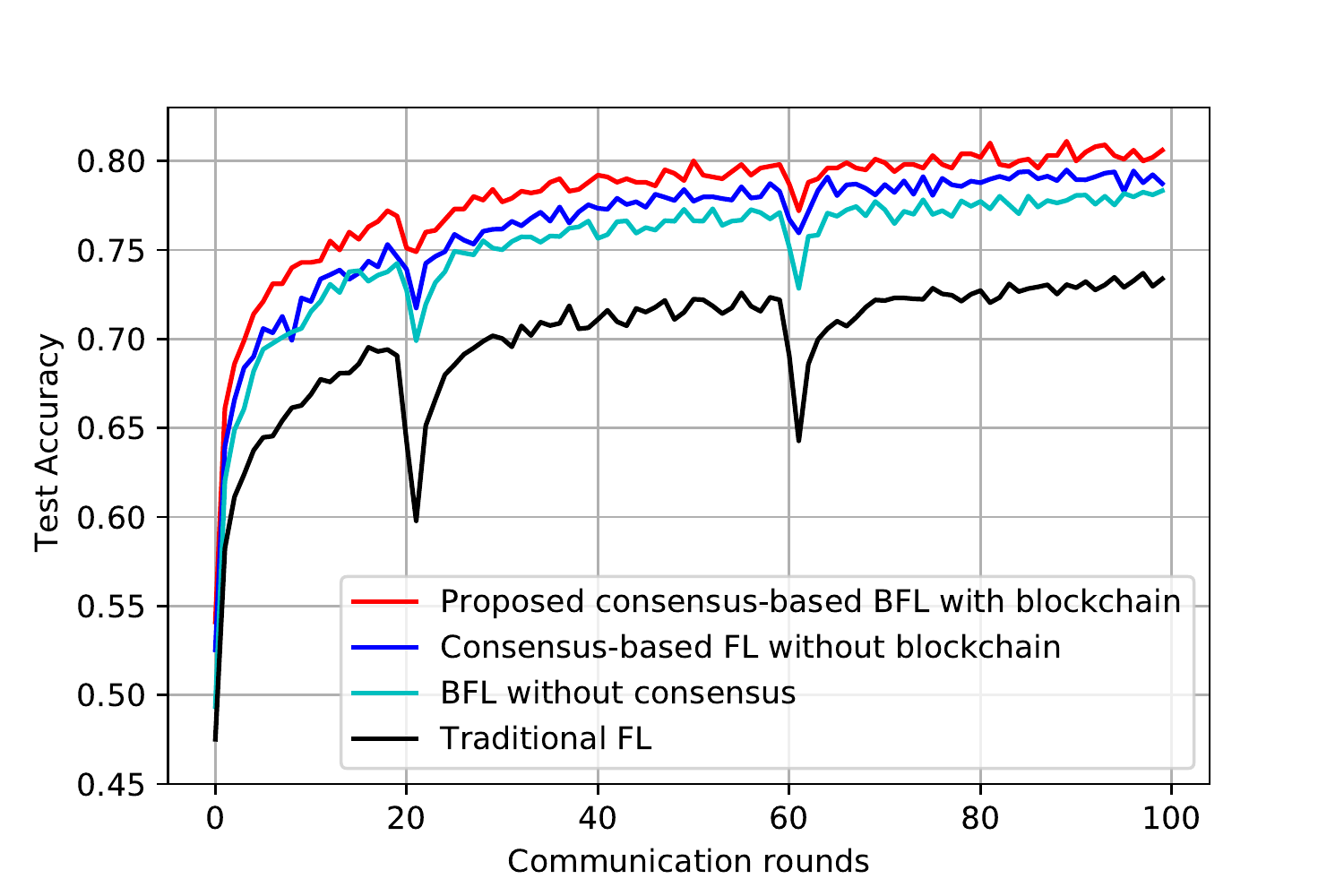} 
		\caption{Accuracy under IID setting. }
	\end{subfigure}%
	~
	\begin{subfigure}[t]{0.25\textwidth}
		\centering
		\includegraphics[width=0.99\linewidth]{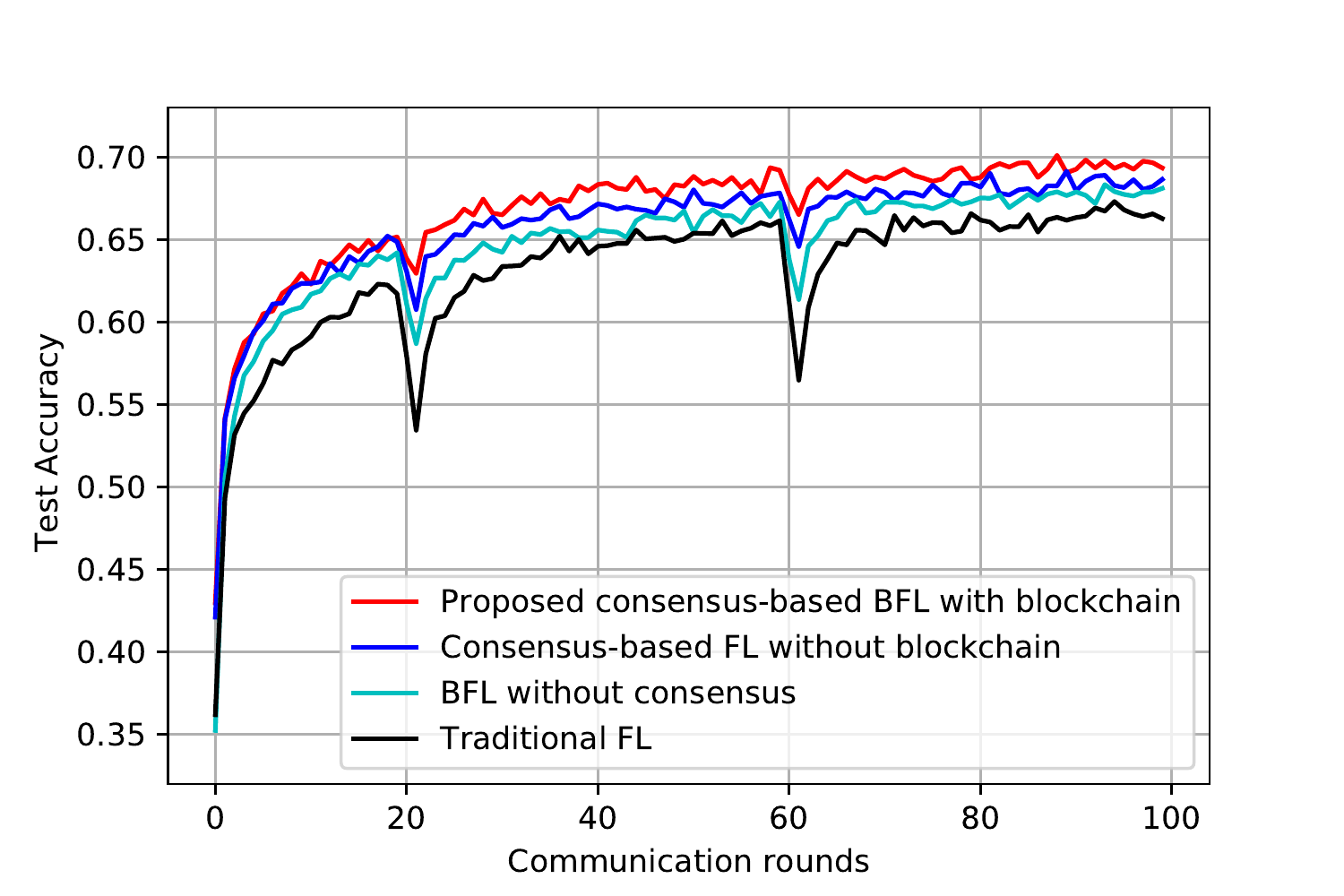} % 7-8.8
		\caption{Accuracy under non-IID setting. }
	\end{subfigure}
	\caption{\textcolor{black}{Comparison of different FL approaches under model poisoning attacks (Fashion-MNIST dataset). }}
	\label{Fig:attack-MNIST}
	\vspace{-0.1in}
\end{figure}

\section{Conclusion}
\label{Section:Conclude}
This paper studied a decentralized BFL system in multi-server edge computing with a holistic design of both \textcolor{black}{offloading-assisted ML model training and mobile block mining schemes.} A model aggregation solution has been developed via P2P-based consensus among ESs to build a global model that is shared with MDs via Blockchain for reliable model learning empowered by block mining. \textcolor{black}{We aimed to minimize the system latency by a  parameterized A2C algorithm with a careful  design of actor and critic. A comprehensive analysis of the convergence properties of our proposed BFL model was given.} Numerical simulations verified the superior performance of our proposed consensus-based BFL scheme over state-of-the-art schemes in terms of higher accuracy and lower loss. The proposed parameterized A2C algorithm  exhibited the faster convergence rate and lower system latency, compared with the existing DRL schemes. \textcolor{black}{Our blockchain-empowered BFL scheme also achieved  high robustness against model poisoning attacks.}
\balance
\bibliography{Ref}

\begin{thebibliography}{10}
\providecommand{\url}[1]{#1}
\csname url@samestyle\endcsname
\providecommand{\newblock}{\relax}
\providecommand{\bibinfo}[2]{#2}
\providecommand{\BIBentrySTDinterwordspacing}{\spaceskip=0pt\relax}
\providecommand{\BIBentryALTinterwordstretchfactor}{4}
\providecommand{\BIBentryALTinterwordspacing}{\spaceskip=\fontdimen2\font plus
\BIBentryALTinterwordstretchfactor\fontdimen3\font minus
  \fontdimen4\font\relax}
\providecommand{\BIBforeignlanguage}[2]{{%
\expandafter\ifx\csname l@#1\endcsname\relax
\typeout{** WARNING: IEEEtran.bst: No hyphenation pattern has been}%
\typeout{** loaded for the language `#1'. Using the pattern for}%
\typeout{** the default language instead.}%
\else
\language=\csname l@#1\endcsname
\fi
#2}}
\providecommand{\BIBdecl}{\relax}
\BIBdecl

\bibitem{2}
S.~Wang \emph{et~al.}, ``Adaptive federated learning in resource constrained
  edge computing systems,'' \emph{IEEE J. Sel. Areas in Commun.}, vol.~37,
  no.~6, pp. 1205--1221, 2019.

\bibitem{3}
D.~C. Nguyen \emph{et~al.}, ``Federated learning for {Internet} of {Things}: A
  comprehensive survey,'' \emph{IEEE Commun. Surveys \& Tutorials}, vol.~23,
  no.~3, pp. 1622--1658, 2021.

\bibitem{4}
C.~Ma, J.~Li, M.~Ding, H.~H. Yang, F.~Shu, T.~Q. Quek, and H.~V. Poor, ``On
  safeguarding privacy and security in the framework of federated learning,''
  \emph{IEEE Net.}, vol.~34, no.~4, pp. 242--248, 2020.

\bibitem{5}
F.~Ayaz, Z.~Sheng, D.~Tian, and Y.~L. Guan, ``A blockchain based federated
  learning for message dissemination in vehicular networks,'' \emph{IEEE Trans.
  Vehicular Techno.}, 2021.

\bibitem{11}
V.~Mothukuri \emph{et~al.}, ``Fabricfl: Blockchain-in-the-loop federated
  learning for trusted decentralized systems,'' \emph{IEEE Systems J.}, 2021.

\bibitem{7}
Y.~He, K.~Huang, G.~Zhang, F.~R. Yu, J.~Chen, and J.~Li, ``Bift: A
  blockchain-based federated learning system for connected and autonomous
  vehicles,'' \emph{IEEE Internet of Things J.}, 2021.

\bibitem{8}
Q.~Hu \emph{et~al.}, ``Blockchain and federated edge learning for
  privacy-preserving mobile crowdsensing,'' \emph{IEEE Internet of Things J.},
  2021.

\bibitem{9}
Y.~Zhao, J.~Zhao, L.~Jiang, R.~Tan, D.~Niyato, Z.~Li, L.~Lyu, and Y.~Liu,
  ``Privacy-preserving blockchain-based federated learning for {IoT} devices,''
  \emph{IEEE Internet of Things J.}, vol.~8, no.~3, pp. 1817--1829, 2020.

\bibitem{10}
D.~C. Nguyen \emph{et~al.}, ``Federated learning meets blockchain in edge
  computing: Opportunities and challenges,'' \emph{IEEE Inter. Things J.},
  2021.

\bibitem{12}
W.~Shi \emph{et~al.}, ``Joint device scheduling and resource allocation for
  latency constrained wireless federated learning,'' \emph{IEEE Trans. Wireless
  Commun.}, vol.~20, no.~1, pp. 453--467, 2020.

\bibitem{13}
Q.~Ma \emph{et~al.}, ``{FedSA}: A semi-asynchronous federated learning
  mechanism in heterogeneous edge computing,'' \emph{IEEE J. Sel. Areas in
  Commun.}, vol.~39, no.~12, pp. 3654--3672, 2021.

\bibitem{14}
H.~Yang \emph{et~al.}, ``Privacy-preserving federated learning for
  {UAV}-enabled networks: Learning-based joint scheduling and resource
  management,'' \emph{IEEE J. Sel. Areas in Commun.}, 2021.

\bibitem{26}
C.~W. Zaw, S.~R. Pandey, K.~Kim, and C.~S. Hong, ``Energy-aware resource
  management for federated learning in multi-access edge computing systems,''
  \emph{IEEE Access}, vol.~9, pp. 34\,938--34\,950, 2021.

\bibitem{15}
Z.~Ji, L.~Chen, N.~Zhao, Y.~Chen, G.~Wei, and F.~R. Yu, ``Computation
  offloading for edge-assisted federated learning,'' \emph{IEEE Trans.
  Vehicular Techno.}, vol.~70, no.~9, pp. 9330--9344, 2021.

\bibitem{16}
Z.~Yang, M.~Chen, W.~Saad, C.~S. Hong, and M.~Shikh-Bahaei, ``Energy efficient
  federated learning over wireless communication networks,'' \emph{IEEE Trans.
  Wireless Commun.}, vol.~20, no.~3, pp. 1935--1949, 2020.

\bibitem{17}
C.~Hu, J.~Jiang, and Z.~Wang, ``Decentralized federated learning: A segmented
  gossip approach,'' \emph{arXiv:1908.07782}, 2019.

\bibitem{18}
H.~Xing, O.~Simeone, and S.~Bi, ``Decentralized federated learning via {SGD}
  over wireless {D2D} networks,'' in \emph{Proc. IEEE 21st Inter. Workshop on
  Signal Processing Advances in Wireless Commun.}, 2020, pp. 1--5.

\bibitem{19}
F.~P.-C. Lin \emph{et~al.}, ``Semi-decentralized federated learning with
  cooperative {D2D} local model aggregations,'' \emph{IEEE J. Sel. Areas in
  Commun.}, vol.~39, no.~12, pp. 3851--3869, 2021.

\bibitem{20}
S.~Hosseinalipour \emph{et~al.}, ``Multi-stage hybrid federated learning over
  large-scale {D2D}-enabled fog networks,'' \emph{IEEE/ACM Trans. Net.}, 2022.

\bibitem{21}
H.~Kim \emph{et~al.}, ``Blockchained on-device federated learning,'' \emph{IEEE
  Commun. Lett.}, vol.~24, no.~6, pp. 1279--1283, 2019.

\bibitem{22}
S.~R. Pokhrel and J.~Choi, ``Federated learning with blockchain for autonomous
  vehicles: Analysis and design challenges,'' \emph{IEEE Trans. Commun.},
  vol.~68, no.~8, pp. 4734--4746, 2020.

\bibitem{23}
Y.~Lu, X.~Huang, K.~Zhang, S.~Maharjan, and Y.~Zhang, ``Communication-efficient
  federated learning and permissioned blockchain for digital twin edge
  networks,'' \emph{IEEE Internet of Things J.}, vol.~8, no.~4, pp. 2276--2288,
  2020.

\bibitem{24}
M.~Aloqaily, I.~Al~Ridhawi, and M.~Guizani, ``Energy-aware blockchain and
  federated learning-supported vehicular networks,'' \emph{IEEE Trans. Intell.
  Transpor. Systems}, 2021.

\bibitem{25}
X.~Deng \emph{et~al.}, ``On dynamic resource allocation for blockchain assisted
  federated learning over wireless channels,'' \emph{arXiv:2105.14708}, 2021.

\bibitem{ganguly2022multi}
B.~Ganguly \emph{et~al.}, ``Multi-edge server-assisted dynamic federated
  learning with an optimized floating aggregation point,'' \emph{arXiv preprint
  arXiv:2203.13950}, 2022.

\bibitem{hosseinalipour2022parallel}
S.~Hosseinalipour \emph{et~al.}, ``Parallel successive learning for dynamic
  distributed model training over heterogeneous wireless networks,''
  \emph{arXiv:2202.02947}, 2022.

\bibitem{31}
X.~Xiong, K.~Zheng, L.~Lei, and L.~Hou, ``Resource allocation based on deep
  reinforcement learning in {IoT} edge computing,'' \emph{IEEE J. Sel. Areas in
  Commun.}, vol.~38, no.~6, pp. 1133--1146, 2020.

\bibitem{37}
S.~Chen \emph{et~al.}, ``Multi-agent deep reinforcement learning-based
  cooperative edge caching for ultra-dense next-generation networks,''
  \emph{IEEE Trans. Commun.}, vol.~69, no.~4, pp. 2441--2456, 2020.

\bibitem{nguyen2020privacy}
D.~C. Nguyen \emph{et~al.}, ``Privacy-preserved task offloading in mobile
  blockchain with deep reinforcement learning,'' \emph{IEEE Trans. Net. and
  Service Manag.}, vol.~17, no.~4, pp. 2536--2549, 2020.

\bibitem{xiao2004fast}
L.~Xiao and S.~Boyd, ``Fast linear iterations for distributed averaging,''
  \emph{Systems \& Control Lett.}, vol.~53, no.~1, pp. 65--78, 2004.

\bibitem{xiao2007distributed}
L.~Xiao, S.~Boyd, and S.-J. Kim, ``Distributed average consensus with
  least-mean-square deviation,'' \emph{J. Parallel and Distributed Comput.},
  vol.~67, no.~1, pp. 33--46, 2007.

\bibitem{dinh2020federated}
C.~T. Dinh, N.~H. Tran, M.~N. Nguyen, C.~S. Hong, W.~Bao, A.~Y. Zomaya, and
  V.~Gramoli, ``Federated learning over wireless networks: Convergence analysis
  and resource allocation,'' \emph{IEEE/ACM Trans. Net.}, vol.~29, no.~1, pp.
  398--409, 2020.

\bibitem{wang2020tackling}
J.~Wang, Q.~Liu, H.~Liang, G.~Joshi, and H.~V. Poor, ``Tackling the objective
  inconsistency problem in heterogeneous federated optimization,'' \emph{Proc.
  Advances in Neural Infor. Processing Systems}, vol.~33, pp. 7611--7623, 2020.

\bibitem{huang2021zkrep}
C.~Huang \emph{et~al.}, ``Zkrep: A privacy-preserving scheme for
  reputation-based blockchain system,'' \emph{IEEE Internet of Things J.},
  vol.~9, no.~6, pp. 4330--4342, 2021.

\bibitem{nguyen2021cooperative}
D.~C. Nguyen \emph{et~al.}, ``Cooperative task offloading and block mining in
  blockchain-based edge computing with multi-agent deep reinforcement
  learning,'' \emph{IEEE Trans. Mobile Comput.}, 2021.

\bibitem{30}
F.~Meng, P.~Chen, L.~Wu, and J.~Cheng, ``Power allocation in multi-user
  cellular networks: Deep reinforcement learning approaches,'' \emph{IEEE
  Trans. Wireless Commun.}, vol.~19, no.~10, pp. 6255--6267, 2020.

\bibitem{32}
H.~Lu \emph{et~al.}, ``Edge {QoE}: Computation offloading with deep
  reinforcement learning for internet of things,'' \emph{IEEE Internet of
  Things J.}, vol.~7, no.~10, pp. 9255--9265, 2020.

\bibitem{33}
M.~Akbari \emph{et~al.}, ``Age of information aware {VNF} scheduling in
  industrial {IoT} using deep reinforcement learning,'' \emph{IEEE J. Sel.
  Areas in Commun.}, 2021.

\bibitem{34}
S.~Qiu \emph{et~al.}, ``On finite-time convergence of actor-critic algorithm,''
  \emph{IEEE J. Sel. Areas in Info. Theory}, vol.~2, no.~2, pp. 652--664, 2021.

\bibitem{35}
L.~Shani \emph{et~al.}, ``Adaptive trust region policy optimization: Global
  convergence and faster rates for regularized {MDPS},'' in \emph{Proc. AAAI
  Conf. Artificial Intell.}, vol.~34, no.~04, 2020, pp. 5668--5675.

\bibitem{36}
S.~Vaswani \emph{et~al.}, ``Painless stochastic gradient: Interpolation,
  line-search, and convergence rates,'' \emph{Proc. Advances in Neural Infor.
  Processing Systems}, vol.~32, pp. 3732--3745, 2019.

\bibitem{fang2020local}
M.~Fang \emph{et~al.}, ``Local model poisoning attacks to {Byzantine}-robust
  federated learning,'' in \emph{Proc. 29th USENIX Security Symposium}, 2020,
  pp. 1605--1622.

\bibitem{mondal2022beas}
A.~Mondal, H.~Virk, and D.~Gupta, ``Beas: Blockchain enabled asynchronous \&
  secure federated machine learning,'' \emph{Proc. Third AAAI WRKSH
  Privacy-Preserving Artif. Intel.}, 2022.

\bibitem{wai2020provably}
H.-T. Wai \emph{et~al.}, ``Provably efficient neural {GTD} for off-policy
  learning,'' \emph{Proc. Advances in Neural Infor. Processing Systems},
  vol.~33, pp. 10\,431--10\,442, 2020.

\bibitem{lohr2019sampling}
S.~L. Lohr, \emph{Sampling: Design and Analysis: Design And Analysis}.\hskip
  1em plus 0.5em minus 0.4em\relax CRC Press, 2019.

\end{thebibliography}
\bibliographystyle{IEEEtran}
\begingroup
\onecolumn

\appendices

\section{Proof of Theorem~\ref{th:main}}\label{app:th}
\label{subsection:Appen-Theorem1}
\subsection{Brief Recap of ML Model Training}
Let set $\mathcal{N}_m^{(k)}$ contain the devices associated with ES $m$ during global aggregation $k$ and also the ES $m$ itself. Also, let $y$ denote the index of an arbitrary ES/MD. Each ES $m$ at global aggregation $k$ acquires its local \textit{aggregated gradient}, which is then normalized by the total number of data points $D^{(k)}$ to form $\nabla {F}^{\mathsf{A},(k)}_{m}$ according to \eqref{equation:modelconsensus}, which can be written in the compact form as
\begin{equation}
      {\nabla {F}}_{m}^{\mathsf{A},(k)}= 
   \sum_{y\in \mathcal{N}_m^{(k)}}\frac{{{D}}^{(k)}_y}{{D}^{(k)} e^{(k)}_y}{\nabla {F}_y^{\mathsf{C},k}}.
\end{equation}
 
%   In the above expression $D_m^{(k)}=\sum_{n\in\mathcal{N}^{(k)}} {D}^{(k)}_y$ denotes the total number of data points at the devices connected to ES $m$ during global aggregation $k$.

Then, the ESs engage in consensus process, where the final parameter at each ES $m$ can be expressed as
\begin{equation}\label{eq:proof26}
    {\nabla {F}}_m^{\mathsf{L},(k)}= \sum_{m\in\mathcal{M}}     {\nabla {F}}_{m}^{\mathsf{A},(k)} + \bm{c}^{(k)}_m= \sum_{m\in\mathcal{M}}
   \sum_{y\in \mathcal{N}_m^{(k)}}\frac{{{D}}^{(k)}_y}{{D}^{(k)} e^{(k)}_y}{\nabla {F}}_y^{\mathsf{C},(k)} + \bm{c}^{(k)}_m,
\end{equation}
where $\bm{c}^{(k)}_m$ denotes the error of consensus, which is caused by finite rounds of P2P communications. The selected ES at aggregation round $k$, denoted by $m_k$, then adds a boosting coefficient $\sum_{m\in\mathcal{M}}\sum_{y'\in \mathcal{N}^{(k)}_m}\frac{{{D}}^{(k)}_{y'}e^{(k)}_{y'}}{{D}^{(k)} }$ to its local gradient to form the following vector:
   \begin{equation}
    \overline{\nabla {F}}_{m_k}^{(k)}=\sum_{m\in\mathcal{M}}\sum_{y'\in \mathcal{N}^{(k)}_m}\frac{{{D}}^{(k)}_{y'}e^{(k)}_{y'}}{{D}^{(k)}}{\nabla {F}}_{m_k}^{\mathsf{L},(k)} = \sum_{m\in\mathcal{M}}\sum_{y'\in \mathcal{N}^{(k)}_m}\frac{{{D}}^{(k)}_{y'}e^{(k)}_{y'}}{{D}^{(k)} }\sum_{m\in\mathcal{M}}
   \sum_{y\in \mathcal{N}_m^{(k)}}\frac{{{D}}^{(k)}_y}{{D}^{(k)} e^{(k)}_y}{\nabla {F}}_y^{\mathsf{C},(k)} + \sum_{m\in\mathcal{M}}\sum_{y'\in \mathcal{N}^{(k)}_m}\frac{{{D}}^{(k)}_{y'}e^{(k)}_{y'}}{{D}^{(k)}}\bm{c}^{(k)}_{m_k}.
\end{equation}
The aggregator ES $m_k$ then updates the global model parameter as follows:
\begin{equation}
    \mathbf{w}^{(k+1)}=\mathbf{w}^{(k)}-\eta_{k} \overline{\nabla {F}}_{m_k}^{(k)}.
    \end{equation}

\subsection{Convergence Analysis}
Using the $\beta$-smoothness of the global loss function (Assumption~\ref{Assup:lossFun}), we have: 
\begin{equation}
 F^{({k})}(\mathbf{w}^{(k+1)}) \leq F^{({k})}(\mathbf{w}^{(k)}) +  \nabla{F^{({k})}(\mathbf{w}^{(k)})}^\top \left( \mathbf{w}^{(k+1)} - \mathbf{w}^{(k)}\right)+ \frac{\beta}{2} \left\Vert \mathbf{w}^{(k+1)} - \mathbf{w}^{(k)}\right\Vert^2.
\end{equation}

Using the updating rule of $\mathbf{w}^{(k+1)}$ and taking the conditional expectation (with respect to data sampling conducted via mini-batch SGD at the last aggregation instance), we have
\begin{equation}\label{eq:res1}
\begin{aligned}
 &\mathbb E_k \left[F^{(k)}(\mathbf{w}^{(k+1)})\right] \leq F^{(k)}(\mathbf{w}^{(k)}) -\\&  \nabla{F^{(k)}(\mathbf{w}^{(k)})}^\top \mathbb{E}_k\left[ \eta_{k}\sum_{m\in\mathcal{M}}\sum_{y'\in \mathcal{N}^{(k)}_m}\frac{{{D}}^{(k)}_{y'}e^{(k)}_{y'}}{{D}^{(k)} }\sum_{m\in\mathcal{M}}
   \sum_{y\in \mathcal{N}_m^{(k)}}\frac{{{D}}^{(k)}_y}{{D}^{(k)} e^{(k)}_y}{\nabla {F}}_y^{\mathsf{C},(k)} + \eta_k\sum_{m\in\mathcal{M}}\sum_{y'\in \mathcal{N}^{(k)}_m}\frac{{{D}}^{(k)}_{y'}e^{(k)}_{y'}}{{D}^{(k)}}\bm{c}^{(k)}_{m_k} \right]\\&+ \frac{\beta}{2} \mathbb{E}_k\left[\left\Vert  \eta_{k}\sum_{m\in\mathcal{M}}\sum_{y'\in \mathcal{N}^{(k)}_m}\frac{{{D}}^{(k)}_{y'}e^{(k)}_{y'}}{{D}^{(k)} }\sum_{m\in\mathcal{M}}
   \sum_{y\in \mathcal{N}_m^{(k)}}\frac{{{D}}^{(k)}_y}{{D}^{(k)} e^{(k)}_y}{\nabla {F}}_y^{\mathsf{C},(k)} + \eta_k\sum_{m\in\mathcal{M}}\sum_{y'\in \mathcal{N}^{(k)}_m}\frac{{{D}}^{(k)}_{y'}e^{(k)}_{y'}}{{D}^{(k)}} \bm{c}^{(k)}_{m_k}\right\Vert^2\right].
    \end{aligned}
\end{equation}

Since ${\nabla {F}}_y^{\mathsf{C},(k)} =-\left( \mathbf{w}_y^{(k),e^{(k)}_y}-\mathbf{w}^{(k)}\right) \Big/ {\eta_{k}}$,  we have
\begin{equation}\label{eq:nablabarFEXP}
    {\nabla {F}}_y^{\mathsf{C},(k)} = {\frac{1}{{{B}}^{(k)}_y}} \sum_{e=1}^{e^{(k)}_y} \sum_{d\in \mathcal{B}^{(k),e}_{y}} \hspace{-1mm} {\nabla  f(\mathbf{w}^{(k),e-1}_{y},d)},
\end{equation}
where $\mathcal{B}_y^{(k),e}$ denotes the set of data points selected to form the mini-batch at the $e$-th local SGD iteration at ES/MD $y$. Since the mini-batch size is fixed during local SGD iterations at each global aggregation round, we defined $B_y^{(k)}=|\mathcal{B}^{(k),e}_{y}|,~\forall e$. Noting that SGD is unbiased it can be easily shown that
\begin{equation}\label{eq:nablabarF}
    \mathbb E_k\left[{\nabla {F}}_y^{\mathsf{C},(k)}\right] =  \sum_{e=1}^{e^{(k)}_y} {\nabla F^{(k)}_y(\mathbf{w}^{(k),e-1}_{y})}.
\end{equation}
Replacing the above result in~\eqref{eq:res1}, we get
\begin{equation}
\begin{aligned}
 \mathbb E_k \left[F^{(k)}(\mathbf{w}^{(k+1)})\right] \leq& F^{(k)}(\mathbf{w}^{(k)}) -  \nabla{F^{(k)}(\mathbf{w}^{(k)})}^\top \mathbb{E}_k\Big[ \eta_{k}\sum_{m\in\mathcal{M}}\sum_{y'\in \mathcal{N}^{(k)}_m}\frac{{{D}}^{(k)}_{y'}e^{(k)}_{y'}}{{D}^{(k)} }
    \sum_{m\in\mathcal{M}}
   \sum_{y\in \mathcal{N}_m^{(k)}}\frac{{{D}}^{(k)}_y}{{D}^{(k)} e^{(k)}_y} \sum_{e=1}^{e^{(k)}_y}  {\nabla  F^{(k)}_y(\mathbf{w}^{(k),e-1}_{y})}\\&
   +\eta_k \sum_{m\in\mathcal{M}}\sum_{y'\in \mathcal{N}^{(k)}_m}\frac{{{D}}^{(k)}_{y'}e^{(k)}_{y'}}{{D}^{(k)}}\bm{c}_{m_k}^{(k)}\Big]\\&+ \frac{\beta}{2} \mathbb{E}_k\left[\left\Vert  \eta_{k}\sum_{m\in\mathcal{M}}\sum_{y'\in \mathcal{N}^{(k)}_m}\frac{{{D}}^{(k)}_{y'}e^{(k)}_{y'}}{{D}^{(k)} }
   \sum_{m\in\mathcal{M}}
   \sum_{y\in \mathcal{N}_m^{(k)}}\frac{{{D}}^{(k)}_y}{{D}^{(k)} e^{(k)}_y}{\nabla {F}}_y^{\mathsf{C},(k)}+\eta_k \sum_{m\in\mathcal{M}}\sum_{y'\in \mathcal{N}^{(k)}_m}\frac{{{D}}^{(k)}_{y'}e^{(k)}_{y'}}{{D}^{(k)}} \bm{c}_{m_k}^{(k)}\right\Vert^2\right].
    \end{aligned}
\end{equation}
Using the linearity of expectation and inner product, we can simplify the above expression as follows:
\begin{equation}\label{eq:proof32}
\begin{aligned}
 \mathbb E_k \left[F^{(k)}(\mathbf{w}^{(k+1)})\right] \leq& F^{(k)}(\mathbf{w}^{(k)}) -  \nabla{F^{(k)}(\mathbf{w}^{(k)})}^\top \mathbb{E}_k\left[ \eta_{k}\sum_{m\in\mathcal{M}}\sum_{y'\in \mathcal{N}^{(k)}_m}\frac{{{D}}^{(k)}_{y'}e^{(k)}_{y'}}{{D}^{(k)} }
    \sum_{m\in\mathcal{M}}
   \sum_{y\in \mathcal{N}_m^{(k)}}\frac{{{D}}^{(k)}_y}{{D}^{(k)} e^{(k)}_y} \sum_{e=1}^{e^{(k)}_y}  {\nabla  F^{(k)}_y(\mathbf{w}^{(k),e-1}_{y})}\right] \\&+\eta_k\sum_{m\in\mathcal{M}}\sum_{y'\in \mathcal{N}^{(k)}_m}\frac{{{D}}^{(k)}_{y'}e^{(k)}_{y'}}{{D}^{(k)}}\nabla{F^{(k)}(\mathbf{w}^{(k)})}^\top \mathbb{E}_k\left[-  \bm{c}_{m_k}^{(k)}\right]\\&+ \frac{\beta}{2} \mathbb{E}_k\left[\left\Vert  \eta_{k}\sum_{m\in\mathcal{M}}\sum_{y'\in \mathcal{N}^{(k)}_m}\frac{{{D}}^{(k)}_{y'}e^{(k)}_{y'}}{{D}^{(k)} }
   \sum_{m\in\mathcal{M}}
   \sum_{y\in \mathcal{N}_m^{(k)}}\frac{{{D}}^{(k)}_y}{{D}^{(k)} e^{(k)}_y}{\nabla {F}}_y^{\mathsf{C},(k)}+\eta_k\sum_{m\in\mathcal{M}}\sum_{y'\in \mathcal{N}^{(k)}_m}\frac{{{D}}^{(k)}_{y'}e^{(k)}_{y'}}{{D}^{(k)}} \bm{c}_{m_k}^{(k)}\right\Vert^2\right].
    \end{aligned}
\end{equation}
Using  Cauchy-Schwartz  and  Young’s inequalities, we have 
\begin{equation}
    \bm{a}^\top \bm{b} \leq \frac{\alpha}{2} \Vert \bm{a} \Vert^2 + \frac{1}{2\alpha} \Vert\bm{b} \Vert^2,~\alpha\in\mathbb{R}^{++}
\end{equation}
for two real valued vectors $\bm{a}$ and $\bm{b}$, implying $\bm{a}^\top \bm{b} \leq \Vert \bm{a} \Vert^2 + \frac{1}{4} \Vert\bm{b} \Vert^2$. Using this inequality in~\eqref{eq:proof32} yields
\begin{equation}\label{eq:proof32-2}
\begin{aligned}
 \mathbb E_k \left[F^{(k)}(\mathbf{w}^{(k+1)})\right] \leq& F^{(k)}(\mathbf{w}^{(k)}) -  \nabla{F^{(k)}(\mathbf{w}^{(k)})}^\top \mathbb{E}_k\left[ \eta_{k}\sum_{m\in\mathcal{M}}\sum_{y'\in \mathcal{N}^{(k)}_m}\frac{{{D}}^{(k)}_{y'}e^{(k)}_{y'}}{{D}^{(k)} }
    \sum_{m\in\mathcal{M}}
   \sum_{y\in \mathcal{N}_m^{(k)}}\frac{{{D}}^{(k)}_y}{{D}^{(k)} e^{(k)}_y} \sum_{e=1}^{e^{(k)}_y}  {\nabla  F^{(k)}_y(\mathbf{w}^{(k),e-1}_{y})}\right] \\&+\frac{\eta_k }{4} \sum_{m\in\mathcal{M}}\sum_{y'\in \mathcal{N}^{(k)}_m}\frac{{{D}}^{(k)}_{y'}e^{(k)}_{y'}}{{D}^{(k)}}\left\Vert\nabla{F^{(k)}(\mathbf{w}^{(k)})}\right\Vert^2 + \eta_k\sum_{m\in\mathcal{M}}\sum_{y'\in \mathcal{N}^{(k)}_m}\frac{{{D}}^{(k)}_{y'}e^{(k)}_{y'}}{{D}^{(k)}} \mathbb{E}_k\left[\left\Vert\bm{c}_{m_k}^{(k)}\right\Vert^2\right]\\&+ \frac{\beta}{2} \mathbb{E}_k\left[\left\Vert  \eta_{k}\sum_{m\in\mathcal{M}}\sum_{y'\in \mathcal{N}^{(k)}_m}\frac{{{D}}^{(k)}_{y'}e^{(k)}_{y'}}{{D}^{(k)} }
   \sum_{m\in\mathcal{M}}
   \sum_{y\in \mathcal{N}_m^{(k)}}\frac{{{D}}^{(k)}_y}{{D}^{(k)} e^{(k)}_y}{\nabla {F}}_y^{\mathsf{C},(k)}+\eta_k\sum_{m\in\mathcal{M}}\sum_{y'\in \mathcal{N}^{(k)}_m}\frac{{{D}}^{(k)}_{y'}e^{(k)}_{y'}}{{D}^{(k)}} \bm{c}_{m_k}^{(k)}\right\Vert^2\right].
    \end{aligned}
\end{equation}
Further, for any two real valued vectors $\bm{a}$ and $\bm{b}$, we have: $2\bm{a}^\top\bm{b}=\Vert \bm{a}\Vert^2+\Vert \bm{b}\Vert^2-\Vert \bm{a}-\bm{b}\Vert^2$. Using this equality in~\eqref{eq:proof32-2} results in the following bound:
\begin{equation}\label{eq:firstlossIneq}
\begin{aligned}
 &\mathbb E_k \left[F^{(k)}(\mathbf{w}^{(k+1)})\right] \leq F^{(k)}(\mathbf{w}^{(k)}) - \frac{\eta_{k}}{2}\sum_{m\in\mathcal{M}}\sum_{y'\in \mathcal{N}^{(k)}_m}\frac{{D}^{(k)}_{y'}e^{(k)}_{y'}}{{D}^{(k)} } \mathbb{E}_k\vast[
 \left\Vert\nabla{F^{(k)}(\mathbf{w}^{(k)})}\right\Vert^2 +\left\Vert 
    \sum_{m\in\mathcal{M}}
   \sum_{y\in \mathcal{N}_m^{(k)}}\frac{{D}^{(k)}_y}{{D}^{(k)} e^{(k)}_y} \sum_{e=1}^{e^{(k)}_y}  {\nabla  F^{(k)}_y(\mathbf{w}^{(k),e-1}_{y})}\right\Vert^2 \hspace{-14mm}\\&-  \left\Vert
    \nabla{F^{(k)}(\mathbf{w}^{(k)})}-
   \sum_{m\in\mathcal{M}}
   \sum_{y\in \mathcal{N}_m^{(k)}}\frac{{D}^{(k)}_y}{{D}^{(k)} e^{(k)}_y} \sum_{e=1}^{e^{(k)}_y}  {\nabla  F^{(k)}_y(\mathbf{w}^{(k),e-1}_{y})}
    \right\Vert^2
    \vast]+\frac{\eta_k}{4} \sum_{m\in\mathcal{M}}\sum_{y'\in \mathcal{N}^{(k)}_m}\frac{{{D}}^{(k)}_{y'}e^{(k)}_{y'}}{{D}^{(k)}}\left\Vert\nabla{F^{(k)}(\mathbf{w}^{(k)})}\right\Vert^2 \\&
    + \eta_k\sum_{m\in\mathcal{M}}\sum_{y'\in \mathcal{N}^{(k)}_m}\frac{{{D}}^{(k)}_{y'}e^{(k)}_{y'}}{{D}^{(k)}}\mathbb{E}_k\left[\left\Vert\bm{c}_{m_k}^{(k)}\right\Vert^2\right]+\\& \frac{\beta \eta^2_{k}}{2} \mathbb{E}_k\left[\left\Vert \sum_{m\in\mathcal{M}}\sum_{y'\in \mathcal{N}^{(k)}_m}\frac{{D}^{(k)}_{y'}e^{(k)}_{y'}}{{D}^{(k)} }
    \sum_{m\in\mathcal{M}}
   \sum_{y\in \mathcal{N}_m^{(k)}}\frac{{D}^{(k)}_y}{{D}^{(k)} e^{(k)}_y}{\nabla {F}}_y^{\mathsf{C},(k)}+ \sum_{m\in\mathcal{M}}\sum_{y'\in \mathcal{N}^{(k)}_m}\frac{{{D}}^{(k)}_{y'}e^{(k)}_{y'}}{{D}^{(k)}}\bm{c}_{m_k}^{(k)}\right\Vert^2\right].\hspace{-4mm}
    \end{aligned}
\end{equation}
Applying the Cauchy-Schwarz inequality (i.e., $\Vert \bm{a}+\bm{b}\Vert^2 \leq 2\Vert \bm{a}\Vert^2 +2\Vert \bm{b} \Vert^2$) on the last term of the above bound yields
\begin{equation}\label{eq:firstlossIneq37}
\begin{aligned}
 &\mathbb E_k \left[F^{(k)}(\mathbf{w}^{(k+1)})\right] \leq F^{(k)}(\mathbf{w}^{(k)}) - \frac{\eta_{k}}{2}\sum_{m\in\mathcal{M}}\sum_{y'\in \mathcal{N}^{(k)}_m}\frac{{D}^{(k)}_{y'}e^{(k)}_{y'}}{{D}^{(k)} } \mathbb{E}_k\vast[
 \left\Vert\nabla{F^{(k)}(\mathbf{w}^{(k)})}\right\Vert^2 +\left\Vert 
    \sum_{m\in\mathcal{M}}
   \sum_{y\in \mathcal{N}_m^{(k)}}\frac{{D}^{(k)}_y}{{D}^{(k)} e^{(k)}_y} \sum_{e=1}^{e^{(k)}_y}  {\nabla  F^{(k)}_y(\mathbf{w}^{(k),e-1}_{y})}\right\Vert^2\hspace{-8mm}\\&-  \left\Vert
    \nabla{F^{(k)}(\mathbf{w}^{(k)})}-
   \sum_{m\in\mathcal{M}}
   \sum_{y\in \mathcal{N}_m^{(k)}}\frac{{D}^{(k)}_y}{{D}^{(k)} e^{(k)}_y} \sum_{e=1}^{e^{(k)}_y}  {\nabla  F^{(k)}_y(\mathbf{w}^{(k),e-1}_{y})}
    \right\Vert^2
    \vast]+\frac{\eta_k}{4}\sum_{m\in\mathcal{M}}\sum_{y'\in \mathcal{N}^{(k)}_m}\frac{{{D}}^{(k)}_{y'}e^{(k)}_{y'}}{{D}^{(k)}} \left\Vert\nabla{F^{(k)}(\mathbf{w}^{(k)})}\right\Vert^2 \\&+ \eta_k\sum_{m\in\mathcal{M}}\sum_{y'\in \mathcal{N}^{(k)}_m}\frac{{{D}}^{(k)}_{y'}e^{(k)}_{y'}}{{D}^{(k)}}\mathbb{E}_k\left[\left\Vert\bm{c}_{m_k}^{(k)}\right\Vert^2\right]\hspace{-4mm}\\&+ {\beta \eta^2_{k}} \underbrace{\mathbb{E}_k\left[\left\Vert \sum_{m\in\mathcal{M}}\sum_{y'\in \mathcal{N}^{(k)}_m}\frac{{D}^{(k)}_{y'}e^{(k)}_{y'}}{{D}^{(k)} }
    \sum_{m\in\mathcal{M}}
   \sum_{y\in \mathcal{N}_m^{(k)}}\frac{{D}^{(k)}_y}{{D}^{(k)} e^{(k)}_y}{\nabla {F}}_y^{\mathsf{C},(k)}\right\Vert^2\right]}_{(a)}+{\beta \eta^2_{k}\left(\sum_{m\in\mathcal{M}}\sum_{y'\in \mathcal{N}^{(k)}_m}\frac{{{D}}^{(k)}_{y'}e^{(k)}_{y'}}{{D}^{(k)}}\right)^2} \mathbb{E}_k\left[\left\Vert \bm{c}_{m_k}^{(k)}\right\Vert^2\right].
    \end{aligned}
\end{equation}
Focusing on term $(a)$ in~\eqref{eq:firstlossIneq37}, using~\eqref{eq:nablabarFEXP}, we upper bound it as follows:
\begin{equation}\label{eq:proof38}
    \begin{aligned}
    &\mathbb{E}_k\left[\left\Vert \sum_{m\in\mathcal{M}}\sum_{y'\in \mathcal{N}^{(k)}_m}\frac{{D}^{(k)}_{y'}e^{(k)}_{y'}}{{D}^{(k)} }
    \sum_{m\in\mathcal{M}}
   \sum_{y\in \mathcal{N}_m^{(k)}}\frac{{D}^{(k)}_y}{{D}^{(k)} e^{(k)}_y}{\nabla {F}}_y^{\mathsf{C},(k)}\right\Vert^2\right]\\&=\left(\sum_{m\in\mathcal{M}}\sum_{y'\in \mathcal{N}^{(k)}_m}\frac{{D}^{(k)}_{y'}e^{(k)}_{y'}}{{D}^{(k)} }\right)^2\mathbb{E}_k\left[\left\Vert 
    \sum_{m\in\mathcal{M}}
   \sum_{y\in \mathcal{N}_m^{(k)}}\frac{{D}^{(k)}_y}{{D}^{(k)} e^{(k)}_y} \sum_{e=1}^{e^{(k)}_y} \sum_{d\in \mathcal{B}^{(k),e}_{y}} \hspace{-1mm} \frac{\nabla  f(\mathbf{w}^{(k),e-1}_{y},d)}{{{B}}^{(k)}_y}\right\Vert^2\right]
   \\&=\left(\sum_{m\in\mathcal{M}}\sum_{y'\in \mathcal{N}^{(k)}_m}\frac{{D}^{(k)}_{y'}e^{(k)}_{y'}}{{D}^{(k)} }\right)^2\mathbb{E}_k\left[\left\Vert 
    \sum_{m\in\mathcal{M}}
   \sum_{y\in \mathcal{N}_m^{(k)}}\frac{{D}^{(k)}_y}{{D}^{(k)} e^{(k)}_y} \sum_{e=1}^{e^{(k)}_y} \left(\sum_{d\in \mathcal{B}^{(k),e}_{y}} \hspace{-1mm} \frac{\nabla  f(\mathbf{w}^{(k),e-1}_{y},d)}{{{B}}^{(k)}_y}-\nabla  F^{(k)}_y(\mathbf{w}^{(k),e-1}_{y})+\nabla  F^{(k)}_y(\mathbf{w}^{(k),e-1}_{y})\right)\right\Vert^2\right]
   \\&=\left(\sum_{m\in\mathcal{M}}\sum_{y'\in \mathcal{N}^{(k)}_m}\frac{{D}^{(k)}_{y'}e^{(k)}_{y'}}{{D}^{(k)} }\right)^2\mathbb{E}_k\left[\left\Vert 
    \sum_{m\in\mathcal{M}}
   \sum_{y\in \mathcal{N}_m^{(k)}}\frac{{D}^{(k)}_y}{{D}^{(k)} e^{(k)}_y} \sum_{e=1}^{e^{(k)}_y} \left(\sum_{d\in \mathcal{B}^{(k),e}_{y}} \hspace{-1mm} \frac{\nabla  f(\mathbf{w}^{(k),e-1}_{y},d)}{{{B}}^{(k)}_y}-\nabla  F^{(k)}_y(\mathbf{w}^{(k),e-1}_{y})+\nabla  F^{(k)}_y(\mathbf{w}^{(k),e-1}_{y})\right)\right\Vert^2\right]\\&
   \leq 2\left(\sum_{m\in\mathcal{M}}\sum_{y'\in \mathcal{N}^{(k)}_m}\frac{{D}^{(k)}_{y'}e^{(k)}_{y'}}{{D}^{(k)} }\right)^2\underbrace{\mathbb{E}_k\left[\left\Vert 
    \sum_{m\in\mathcal{M}}
   \sum_{y\in \mathcal{N}_m^{(k)}}\frac{{D}^{(k)}_y}{{D}^{(k)} e^{(k)}_y} \sum_{e=1}^{e^{(k)}_y} \left(\sum_{d\in \mathcal{B}^{(k),e}_{y}} \hspace{-1mm} \frac{\nabla  f(\mathbf{w}^{(k),e-1}_{y},d)}{{{B}}^{(k)}_y}-\nabla  F^{(k)}_y(\mathbf{w}^{(k),e-1}_{y})\right)\right\Vert^2\right]}_{(a)}
   \\&+2\left(\sum_{m\in\mathcal{M}}\sum_{y'\in \mathcal{N}^{(k)}_m}\frac{{D}^{(k)}_{y'}e^{(k)}_{y'}}{{D}^{(k)} }\right)^2\mathbb{E}_k\left[\left\Vert 
    \sum_{m\in\mathcal{M}}
   \sum_{y\in \mathcal{N}_m^{(k)}}\frac{{D}^{(k)}_y}{{D}^{(k)} e^{(k)}_y} \sum_{e=1}^{e^{(k)}_y}\nabla  F^{(k)}_y(\mathbf{w}^{(k),e-1}_{y})\right\Vert^2\right].
    \end{aligned}
\end{equation}
where the last inequality is obtained via Cauchy-Schwarz inequality. Note that $\sum_{d\in \mathcal{B}^{(k),e}_{y}} \hspace{-1mm} \frac{\nabla  f(\mathbf{w}^{(k),e-1}_{y},d)}{{{B}}^{(k)}_y}-\nabla  F^{(k)}_y(\mathbf{w}^{(k),e-1}_{y})$ in the above expression corresponds to the noise of SGD at local iteration $i$ at device $y$, which is zero mean and iid across the devices and the iterations. Using these facts in term $(a)$ in~\eqref{eq:proof38} implies
\begin{equation}\label{eq:proof40}
    \begin{aligned}
    &\mathbb{E}_k\left[\left\Vert \sum_{m\in\mathcal{M}}\sum_{y'\in \mathcal{N}^{(k)}_m}\frac{{D}^{(k)}_{y'}e^{(k)}_{y'}}{{D}^{(k)} }
    \sum_{m\in\mathcal{M}}
   \sum_{y\in \mathcal{N}_m^{(k)}}\frac{{D}^{(k)}_y}{{D}^{(k)} e^{(k)}_y}{\nabla {F}}_y^{\mathsf{C},(k)}\right\Vert^2\right]\\&=\left(\sum_{m\in\mathcal{M}}\sum_{y'\in \mathcal{N}^{(k)}_m}\frac{{D}^{(k)}_{y'}e^{(k)}_{y'}}{{D}^{(k)} }\right)^2 \sum_{m\in\mathcal{M}}
   \sum_{y\in \mathcal{N}_m^{(k)}}\left(\frac{{D}^{(k)}_y}{{D}^{(k)} e^{(k)}_y}\right)^2\underbrace{\mathbb{E}_k\left[\left\Vert 
     \sum_{e=1}^{e^{(k)}_y} \left(\sum_{d\in \mathcal{B}^{(k),e}_{y}} \hspace{-1mm} \frac{\nabla  f(\mathbf{w}^{(k),e-1}_{y},d)}{{{B}}^{(k)}_y}-\nabla  F^{(k)}_y(\mathbf{w}^{(k),e-1}_{y})\right)\right\Vert^2\right]}_{(a)}\\&+
    \left(\sum_{m\in\mathcal{M}}\sum_{y'\in \mathcal{N}^{(k)}_m}\frac{{D}^{(k)}_{y'}e^{(k)}_{y'}}{{D}^{(k)} }\right)^2  \mathbb{E}_k\left[\left\Vert 
    \sum_{m\in\mathcal{M}}
   \sum_{y\in \mathcal{N}_m^{(k)}}\frac{{D}^{(k)}_y}{{D}^{(k)} e^{(k)}_y} \sum_{e=1}^{e^{(k)}_y}
   \nabla  F^{(k)}_y(\mathbf{w}^{(k),e-1}_{y})\right\Vert^2\right].
    \end{aligned}
\end{equation}
Using Lemma~\ref{lemma:SGDnoiseExp} (Appendix~\ref{app:SGDnoiseExp}) and the fact that SGD noise is zero mean, expanding term $(a)$ in~\eqref{eq:proof40} yields
\begin{equation}
    \begin{aligned}
    &\mathbb{E}_k\left[\left\Vert \sum_{m\in\mathcal{M}}\sum_{y'\in \mathcal{N}^{(k)}_m}\frac{{D}^{(k)}_{y'}e^{(k)}_{y'}}{{D}^{(k)} }
    \sum_{m\in\mathcal{M}}
   \sum_{y\in \mathcal{N}_m^{(k)}}\frac{{D}^{(k)}_y}{{D}^{(k)} e^{(k)}_y}{\nabla {F}}_y^{\mathsf{C},(k)}\right\Vert^2\right]\\&\leq \left(\sum_{m\in\mathcal{M}}\sum_{y'\in \mathcal{N}^{(k)}_m}\frac{{D}^{(k)}_{y'}e^{(k)}_{y'}}{{D}^{(k)} }\right)^2 \sum_{m\in\mathcal{M}}
   \sum_{y\in \mathcal{N}_m^{(k)}}\left(\frac{{D}^{(k)}_y}{{D}^{(k)} e^{(k)}_y}\right)^2\sum_{e=1}^{e^{(k)}_y}
2\left(1-\frac{{B}^{(k)}_y}{{D}^{(k)}_y}\right) \frac{({\sigma}_{y}^{(k)})^2}{{B}^{(k)}_y} \Theta^2_y\\&+
    \left(\sum_{m\in\mathcal{M}}\sum_{y'\in \mathcal{N}^{(k)}_m}\frac{{D}^{(k)}_{y'}e^{(k)}_{y'}}{{D}^{(k)} }\right)^2  \mathbb{E}_k\left[\left\Vert 
    \sum_{m\in\mathcal{M}}
   \sum_{y\in \mathcal{N}_m^{(k)}}\frac{{D}^{(k)}_y}{{D}^{(k)} e^{(k)}_y} \sum_{e=1}^{e^{(k)}_y}
   \nabla  F^{(k)}_y(\mathbf{w}^{(k),e-1}_{y})\right\Vert^2\right].
    \end{aligned}
\end{equation}
Replacing the above result back in~\eqref{eq:firstlossIneq37}, we get
\begin{equation}
\begin{aligned}
 &\mathbb E_k \left[F^{(k)}(\mathbf{w}^{(k+1)})\right] \leq F^{(k)}(\mathbf{w}^{(k)}) - \frac{\eta_{k}}{2}\sum_{m\in\mathcal{M}}\sum_{y'\in \mathcal{N}^{(k)}_m}\frac{{D}^{(k)}_{y'}e^{(k)}_{y'}}{{D}^{(k)} } \mathbb{E}_k\vast[
 \left\Vert\nabla{F^{(k)}(\mathbf{w}^{(k)})}\right\Vert^2 +\left\Vert 
    \sum_{m\in\mathcal{M}}
   \sum_{y\in \mathcal{N}_m^{(k)}}\frac{{D}^{(k)}_y}{{D}^{(k)} e^{(k)}_y} \sum_{e=1}^{e^{(k)}_y}  {\nabla  F^{(k)}_y(\mathbf{w}^{(k),e-1}_{y})}\right\Vert^2\\&-  \left\Vert
    \nabla{F^{(k)}(\mathbf{w}^{(k)})}-
   \sum_{m\in\mathcal{M}}
   \sum_{y\in \mathcal{N}_m^{(k)}}\frac{{D}^{(k)}_y}{{D}^{(k)} e^{(k)}_y} \sum_{e=1}^{e^{(k)}_y}  {\nabla  F^{(k)}_y(\mathbf{w}^{(k),e-1}_{y})}
    \right\Vert^2
    \vast]+\frac{\eta_k}{4}\sum_{m\in\mathcal{M}}\sum_{y'\in \mathcal{N}^{(k)}_m}\frac{{{D}}^{(k)}_{y'}e^{(k)}_{y'}}{{D}^{(k)}} \left\Vert\nabla{F^{(k)}(\mathbf{w}^{(k)})}\right\Vert^2 \\&+ \eta_k\sum_{m\in\mathcal{M}}\sum_{y'\in \mathcal{N}^{(k)}_m}\frac{{{D}}^{(k)}_{y'}e^{(k)}_{y'}}{{D}^{(k)}}\mathbb{E}_k\left[\left\Vert\bm{c}_{m_k}^{(k)}\right\Vert^2\right]+ {\beta \eta^2_{k}} 
    \left(\sum_{m\in\mathcal{M}}\sum_{y'\in \mathcal{N}^{(k)}_m}\frac{{D}^{(k)}_{y'}e^{(k)}_{y'}}{{D}^{(k)} }\right)^2 \sum_{m\in\mathcal{M}}
   \sum_{y\in \mathcal{N}_m^{(k)}}\left(\frac{{D}^{(k)}_y}{{D}^{(k)} e^{(k)}_y}\right)^2\sum_{e=1}^{e^{(k)}_y}
2\left(1-\frac{{B}^{(k)}_y}{{D}^{(k)}_y}\right) \frac{({\sigma}_{y}^{(k)})^2}{{B}^{(k)}_y} \Theta^2_y\\&+
  {\beta \eta^2_{k}} 
  \left(\sum_{m\in\mathcal{M}}\sum_{y'\in \mathcal{N}^{(k)}_m}\frac{{D}^{(k)}_{y'}e^{(k)}_{y'}}{{D}^{(k)} }\right)^2  \mathbb{E}_k\left[\left\Vert 
    \sum_{m\in\mathcal{M}}
   \sum_{y\in \mathcal{N}_m^{(k)}}\frac{{D}^{(k)}_y}{{D}^{(k)} e^{(k)}_y} \sum_{e=1}^{e^{(k)}_y}
   \nabla  F^{(k)}_y(\mathbf{w}^{(k),e-1}_{y})\right\Vert^2\right]
   \\&+{\beta \eta^2_{k}\left(\sum_{m\in\mathcal{M}}\sum_{y'\in \mathcal{N}^{(k)}_m}\frac{{{D}}^{(k)}_{y'}e^{(k)}_{y'}}{{D}^{(k)}}\right)^2} \mathbb{E}_k\left[\left\Vert \bm{c}_{m_k}^{(k)}\right\Vert^2\right].
    \end{aligned}
\end{equation}

With rearranging the terms in the above bound, we get
\begin{equation}\label{eq:firstlossIneq}
\begin{aligned}
 &\mathbb E_k \left[F^{(k)}(\mathbf{w}^{(k+1)})\right] \leq F^{(k)}(\mathbf{w}^{(k)}) -
 \left(\frac{\eta_k}{4}\sum_{m\in\mathcal{M}}\sum_{y'\in \mathcal{N}^{(k)}_m}\frac{{{D}}^{(k)}_{y'}e^{(k)}_{y'}}{{D}^{(k)}}\right)\left\Vert\nabla{F^{(k)}(\mathbf{w}^{(k)})}\right\Vert^2   \\&
 +\eta_k\left(\underbrace{\beta\eta_k\left(\sum_{m\in\mathcal{M}}\sum_{y'\in \mathcal{N}^{(k)}_m}\frac{{{D}}^{(k)}_{y'}e^{(k)}_{y'}}{{D}^{(k)}}\right)^2-\frac{1}{2}\sum_{m\in\mathcal{M}}\sum_{y'\in \mathcal{N}^{(k)}_m}\frac{{{D}}^{(k)}_{y'}e^{(k)}_{y'}}{{D}^{(k)}}}_{(a)}\right) \left\Vert 
    \sum_{m\in\mathcal{M}}
   \sum_{y\in \mathcal{N}_m^{(k)}}\frac{{D}^{(k)}_y}{{D}^{(k)} e^{(k)}_y} \sum_{e=1}^{e^{(k)}_y}  {\nabla  F^{(k)}_y(\mathbf{w}^{(k),e-1}_{y})}\right\Vert^2\\&
   +\left(\frac{\eta_k}{2}\sum_{m\in\mathcal{M}}\sum_{y'\in \mathcal{N}^{(k)}_m}\frac{{{D}}^{(k)}_{y'}e^{(k)}_{y'}}{{D}^{(k)}}\right) \left\Vert
    \nabla{F^{(k)}(\mathbf{w}^{(k)})}-
   \sum_{m\in\mathcal{M}}
   \sum_{y\in \mathcal{N}_m^{(k)}}\frac{{D}^{(k)}_y}{{D}^{(k)} e^{(k)}_y} \sum_{e=1}^{e^{(k)}_y}  {\nabla  F^{(k)}_y(\mathbf{w}^{(k),e-1}_{y})}
    \right\Vert^2\\&
   +{\eta_k}\sum_{m\in\mathcal{M}}\sum_{y'\in \mathcal{N}^{(k)}_m}\frac{{{D}}^{(k)}_{y'}e^{(k)}_{y'}}{{D}^{(k)}} \left(1+\beta{\eta_k}\sum_{m\in\mathcal{M}}\sum_{y'\in \mathcal{N}^{(k)}_m}\frac{{{D}}^{(k)}_{y'}e^{(k)}_{y'}}{{D}^{(k)}}\right) \mathbb{E}_k\left[\left\Vert \bm{c}_{m_k}^{(k)}\right\Vert^2\right]\\&
   +
   {\beta \eta^2_{k}} 
    \left(\sum_{m\in\mathcal{M}}\sum_{y'\in \mathcal{N}^{(k)}_m}\frac{{D}^{(k)}_{y'}e^{(k)}_{y'}}{{D}^{(k)} }\right)^2 \sum_{m\in\mathcal{M}}
   \sum_{y\in \mathcal{N}_m^{(k)}}\left(\frac{{D}^{(k)}_y}{{D}^{(k)} e^{(k)}_y}\right)^2\sum_{e=1}^{e^{(k)}_y}
2\left(1-\frac{{B}^{(k)}_y}{{D}^{(k)}_y}\right) \frac{({\sigma}_{y}^{(k)})^2}{{B}^{(k)}_y} \Theta^2_y.
    \end{aligned}
\end{equation}
Assuming $\eta_k \leq \frac{{D}^{(k)}}{2\beta\sum_{m\in\mathcal{M}}\sum_{y'\in \mathcal{N}^{(k)}_m}{{D}}^{(k)}_{y'}e^{(k)}_{y'}}$ makes term $(a)$ in the above expression negative and also implies $1+\beta{\eta_k}\sum_{m\in\mathcal{M}}\sum_{y'\in \mathcal{N}^{(k)}_m}\frac{{{D}}^{(k)}_{y'}e^{(k)}_{y'}}{{D}^{(k)}} \leq \frac{3}{2}$. Replacing these results in the above bound gives us
\begin{equation}\label{eq:proof44}
\begin{aligned}
 &\mathbb E_k \left[F^{(k)}(\mathbf{w}^{(k+1)})\right] \leq F^{(k)}(\mathbf{w}^{(k)}) -
 \left(\frac{\eta_k}{4}\sum_{m\in\mathcal{M}}\sum_{y'\in \mathcal{N}^{(k)}_m}\frac{{{D}}^{(k)}_{y'}e^{(k)}_{y'}}{{D}^{(k)}}\right)\left\Vert\nabla{F^{(k)}(\mathbf{w}^{(k)})}\right\Vert^2  \\&
   +\left(\frac{\eta_k}{2}\sum_{m\in\mathcal{M}}\sum_{y'\in \mathcal{N}^{(k)}_m}\frac{{{D}}^{(k)}_{y'}e^{(k)}_{y'}}{{D}^{(k)}}\right) \underbrace{\left\Vert
    \nabla{F^{(k)}(\mathbf{w}^{(k)})}-
   \sum_{m\in\mathcal{M}}
   \sum_{y\in \mathcal{N}_m^{(k)}}\frac{{D}^{(k)}_y}{{D}^{(k)} e^{(k)}_y} \sum_{e=1}^{e^{(k)}_y}  {\nabla  F^{(k)}_y(\mathbf{w}^{(k),e-1}_{y})}
    \right\Vert^2}_{(a)}\\&
   +\frac{3}{2}{\eta_k}\sum_{m\in\mathcal{M}}\sum_{y'\in \mathcal{N}^{(k)}_m}\frac{{{D}}^{(k)}_{y'}e^{(k)}_{y'}}{{D}^{(k)}} \mathbb{E}_k\left[\left\Vert \bm{c}_{m_k}^{(k)}\right\Vert^2\right]+2{\beta \eta^2_{k}} 
    \left(\sum_{m\in\mathcal{M}}\sum_{y'\in \mathcal{N}^{(k)}_m}\frac{{D}^{(k)}_{y'}e^{(k)}_{y'}}{{D}^{(k)} }\right)^2 \sum_{m\in\mathcal{M}}
   \sum_{y\in \mathcal{N}_m^{(k)}}\left(\frac{{D}^{(k)}_y}{{D}^{(k)} \sqrt{e^{(k)}_y}}\right)^2
\left(1-\frac{{B}^{(k)}_y}{{D}^{(k)}_y}\right) \frac{({\sigma}_{y}^{(k)})^2}{{B}^{(k)}_y} \Theta^2_y. \hspace{-5mm}
    \end{aligned}
\end{equation}
We next aim to bound term $(a)$ in~\eqref{eq:proof44} as follows:
\begin{equation}\label{eq:proof45}
    \begin{aligned}
    &\left\Vert
    \nabla{F^{(k)}(\mathbf{w}^{(k)})}-
   \sum_{m\in\mathcal{M}}
   \sum_{y\in \mathcal{N}_m^{(k)}}\frac{{D}^{(k)}_y}{{D}^{(k)} e^{(k)}_y} \sum_{e=1}^{e^{(k)}_y}  {\nabla  F^{(k)}_y(\mathbf{w}^{(k),e-1}_{y})}
    \right\Vert^2\\&=\left\Vert
  \sum_{m\in\mathcal{M}}
   \sum_{y\in \mathcal{N}_m^{(k)}}\frac{{D}^{(k)}_y}{{D}^{(k)}  } \nabla{F^{(k)}_y(\mathbf{w}^{(k)})}-
   \sum_{m\in\mathcal{M}}
   \sum_{y\in \mathcal{N}_m^{(k)}}\frac{{D}^{(k)}_y}{{D}^{(k)} e^{(k)}_y} \sum_{e=1}^{e^{(k)}_y}  {\nabla  F^{(k)}_y(\mathbf{w}^{(k),e-1}_{y})}
    \right\Vert^2
    \\&{\leq}   \sum_{m\in\mathcal{M}}
   \sum_{y\in \mathcal{N}_m^{(k)}}\frac{{D}^{(k)}_y}{{D}^{(k)} }\mathbb{E}_k\left[\left\Vert
    \nabla{F^{(k)}_y(\mathbf{w}^{(k)})}-
    \frac{1}{e^{(k)}_y} \sum_{e=1}^{e^{(k)}_y}  {\nabla  F^{(k)}_y(\mathbf{w}^{(k),e-1}_{y})}
    \right\Vert^2\right]
    \\& 
    {\leq} \sum_{m\in\mathcal{M}}
   \sum_{y\in \mathcal{N}_m^{(k)}}\frac{{D}^{(k)}_y}{D^{(k)} e_y^{(k)}}\sum_{e=1}^{e^{(k)}_y} \mathbb{E}_k\left[\left\Vert
   \nabla{F^{(k)}_y(\mathbf{w}^{(k)})}-
     {\nabla  F^{(k)}_y(\mathbf{w}^{(k),e-1}_{y})}
    \right\Vert^2\right]
    \leq \beta^2\sum_{m\in\mathcal{M}}
   \sum_{y\in \mathcal{N}_m^{(k)}}\frac{{D}^{(k)}_y}{D^{(k)} e_y^{(k)}}\sum_{e=1}^{e^{(k)}_y} \mathbb{E}_k\left[\left\Vert
   \mathbf{w}^{(k)}-
     \mathbf{w}^{(k),e-1}_{y}
    \right\Vert^2\right],\hspace{-10mm}
    \end{aligned}
\end{equation}
where we used Jensens's inequality and $\beta$-smoothness of the loss functions.
    
    We then bound the last term in the above bound as follows:
  \begin{align}
    &
    \mathbb{E}_k\left[\left\Vert
   \mathbf{w}^{(k)}-
     \mathbf{w}^{(k),e-1}_{y}
    \right\Vert^2\right]{=}\eta_k^2  \mathbb{E}_k\left[\left\Vert
   \sum_{e'=1}^{e-1} \sum_{d\in \mathcal{B}^{(k),e'}_{y}} \hspace{-1mm} {\frac{\nabla  f(\mathbf{w}^{(k),e'-1}_{y},d)}{{B}^{(k)}_{y}}}
    \right\Vert^2 \right]
    \nonumber\\&= \eta_k^2 \mathbb{E}_k\vast[\Bigg\Vert
    \sum_{e'=1}^{e-1} \Big(\sum_{d\in \mathcal{B}^{(k),e'}_{y}} \hspace{-1mm} {\frac{\nabla  f(\mathbf{w}^{(k),e'-1}_{y},d)}{{B}^{(k)}_{y}}}
-\nabla  F^{(k)}_y(\mathbf{w}^{(k),e'-1}_{y})+\nabla  F^{(k)}_y(\mathbf{w}^{(k),e'-1}_{y})
    \Big)\Bigg\Vert^2 \vast]
    \nonumber \\& {\leq }2 \eta_k^2  \mathbb{E}_k\left[\Bigg\Vert
   \sum_{e'=1}^{e-1} \sum_{d\in \mathcal{B}^{(k),e'}_{y}} \hspace{-1mm} {\frac{\nabla  f(\mathbf{w}^{(k),e'-1}_{y},d)}{{B}^{(k)}_{y}}}
-\nabla  F^{(k)}_y(\mathbf{w}^{(k),e'-1}_{y})\Bigg\Vert^2\right]+2 \eta_k^2 \mathbb{E}_k\left[\Bigg\Vert \sum_{e'=1}^{e-1} \nabla  F^{(k)}_y(\mathbf{w}^{(k),e'-1}_{y}) \Bigg\Vert^2\right]
   \nonumber \\& \nonumber 
  {=} {2 \eta_k^2 \sum_{e'=1}^{e-1}\mathbb{E}_k\left[\Bigg\Vert
   \sum_{d\in \mathcal{B}^{(k),e'}_{y}} \hspace{-1mm} {\frac{\nabla  f(\mathbf{w}^{(k),e'-1}_{y},d)}{{B}^{(k)}_{y}}}
-\nabla  F^{(k)}_y(\mathbf{w}^{(k),e'-1}_{y})\Bigg\Vert^2\right]}+{2 \eta_k^2 \mathbb{E}_k\left[\Bigg\Vert  \sum_{e'=1}^{e-1}\nabla  F^{(k)}_y(\mathbf{w}^{(k),e'-1}_{y}) 
    \Bigg\Vert^2\right]}
    \\& \leq 4 (e-1)\eta_k^2 \left(1-\frac{{B}^{(k)}_y}{{D}^{(k)}_y}\right) \frac{({\sigma}_{y}^{(k)})^2}{{B}^{(k)}_y} \Theta^2_y +\underbrace{2 \eta_k^2 \mathbb{E}_k\left[\Bigg\Vert  \sum_{e'=1}^{e-1}\nabla  F^{(k)}_y(\mathbf{w}^{(k),e'-1}_{y}) 
    \Bigg\Vert^2\right]}_{(a)},\label{eq:proof46}
    \end{align}
where  we used  Cauchy–Schwarz inequality, the fact that noise of SGD  is  zero mean, and Lemma~\ref{lemma:SGDnoiseExp}. We then bound term $(a)$ in~\eqref{eq:proof46} as follows:
\begin{align}\label{eq:B2}
   &2 \eta_k^2 \mathbb{E}_k\left[\Bigg\Vert  \sum_{e'=1}^{e-1}\nabla  F^{(k)}_y(\mathbf{w}^{(k),e'-1}_{y}) 
    \Bigg\Vert^2\right] \leq  2 \eta_k^2 (e-1)\sum_{e'=1}^{e-1}\mathbb{E}_k\left[\Bigg\Vert \nabla  F^{(k)}_y(\mathbf{w}^{(k),e'-1}_{y}) - \nabla F^{(k)}_y(\mathbf{w}^{(k)})+ \nabla F^{(k)}_y(\mathbf{w}^{(k)})
    \Bigg\Vert^2 \right]
   \nonumber \\
    &{\leq} 4 \eta_k^2 (e-1)\sum_{e'=1}^{e-1}\mathbb{E}_k\left[\Bigg\Vert \nabla  F^{(k)}_y(\mathbf{w}^{(k),e'-1}_{y}) - \nabla F^{(k)}_y(\mathbf{w}^{(k)})\Bigg\Vert^2\right] + 4 \eta_k^2 (e-1)\sum_{e'=1}^{e-1} \Bigg\Vert\nabla F^{(k)}_y(\mathbf{w}^{(k)})
    \Bigg\Vert^2 
    \nonumber \\&
    \leq 4 \eta_k^2 (e-1)\sum_{e'=1}^{e-1}\mathbb{E}_k\left[\Bigg\Vert \nabla  F^{(k)}_y(\mathbf{w}^{(k),e'-1}_{y}) - \nabla F^{(k)}_y(\mathbf{w}^{(k)})\Bigg\Vert^2\right] + 4 \eta_k^2 (e-1)\sum_{e'=1}^{e-1} \Bigg\Vert\nabla F^{(k)}_y(\mathbf{w}^{(k)})
    \Bigg\Vert^2 
    \nonumber \\&
    \leq 4\eta_k^2\beta^2 (e-1) \sum_{e'=1}^{e-1}\mathbb{E}_k\left[\Bigg\Vert \mathbf{w}^{(k),e'-1}_{y} - \mathbf{w}^{(k)}\Bigg\Vert^2\right]+ 4 \eta_k^2 (e-1)\sum_{e'=1}^{e-1} \Bigg\Vert\nabla F^{(k)}_y(\mathbf{w}^{(k)})
    \Bigg\Vert^2 ,
\end{align}
where we used Cauchy-Schwarz inequality repeatedly and applied the $\beta$-smoothness of the loss function. Replacing the above expression in~\eqref{eq:proof46} yields 
\begin{align}
  \mathbb{E}_k\left[  \left\Vert
  \mathbf{w}^{(k)}-
     \mathbf{w}^{(k),e-1}_{y}
    \right\Vert^2\right] \leq& 4 (e-1)\eta_k^2 \left(1-\frac{{B}^{(k)}_y}{{D}^{(k)}_y}\right) \frac{({\sigma}_{y}^{(k)})^2}{{B}^{(k)}_y} \Theta^2_y\nonumber \\
     & + 4\eta_k^2\beta^2 (e-1) \sum_{e'=1}^{e-1}\mathbb{E}_k\left[\Bigg\Vert \mathbf{w}^{(k),e'-1}_{y} - \mathbf{w}^{(k)}\Bigg\Vert^2\right]+ 4 \eta_k^2 (e-1)\sum_{e'=1}^{e-1} \Bigg\Vert\nabla F^{(k)}_y(\mathbf{w}^{(k)})
    \Bigg\Vert^2, 
\end{align}
which implies:
\begin{align}
   &\sum_{e=1}^{e_y^{(k)}} \mathbb{E}_k\left[\left\Vert
   \mathbf{w}^{(k)}-
     \mathbf{w}^{(k),e-1}_{y}
    \right\Vert^2\right] \leq 4  \eta_k^2\sum_{e=1}^{e_y^{(k)}}  (e-1) \left(1-\frac{{B}^{(k)}_y}{{D}^{(k)}_y}\right) \frac{({\sigma}_{y}^{(k)})^2}{{B}^{(k)}_y}\Theta_y^2 \nonumber \\
     & + 4\eta_k^2\beta^2 \sum_{e=1}^{e_y^{(k)}}(e-1) \sum_{e'=1}^{e-1}\mathbb{E}_k\left[\Bigg\Vert \mathbf{w}^{(k),e'-1}_{y} - \mathbf{w}^{(k)}\Bigg\Vert^2\right]+ 4 \eta_k^2 \sum_{e=1}^{e_y^{(k)}}(e-1)\sum_{e'=1}^{e-1} \Bigg\Vert\nabla F^{(k)}_y(\mathbf{w}^{(k)})
    \Bigg\Vert^2
   \nonumber \\&\leq
    4  \eta_k^2 \left(e_y^{(k)}\right)\left(e_y^{(k)}-1\right)\left(1-\frac{{B}^{(k)}_y}{{D}^{(k)}_y}\right) \frac{({\sigma}_{y}^{(k)})^2}{{B}^{(k)}_y}\Theta_y^2 \nonumber \\
     & + 4\eta_k^2\beta^2 \left(e_y^{(k)}\right)\left(e_y^{(k)}-1\right) \sum_{e=1}^{e_y^{(k)}}\mathbb{E}_k\left[\Bigg\Vert \mathbf{w}^{(k),e-1}_{y} - \mathbf{w}^{(k)}\Bigg\Vert^2\right]+ 4 \eta_k^2 \left(e_y^{(k)}\right)\left(e_y^{(k)}-1\right) \sum_{e=1}^{e_y^{(k)}} \Bigg\Vert\nabla F^{(k)}_y(\mathbf{w}^{(k)})
    \Bigg\Vert^2,
\end{align}
% and in turn
% \begin{align}
%   \sum_{e=1}^{e_y^{(k)}} \left\Vert
%   \mathbf{w}^{(k)}-
%      \mathbf{w}^{(k),e-1}_{y}
%     \right\Vert^2 \leq&
%     8 \Theta^2 \eta_k^2 \left(e_y^{(k)}\right)\left(e_y^{(k)}-1\right)\sum_{j=1}^{{M}^{(k)}_{y}} \left(1-\frac{{B}^{(k)}_{n,j}}{{M}^{(k)}_{n,j}} \right) \frac{{M}^{(k)}_{n,j}}{\left({D}^{(k)}_{y}\right)^2} \frac{{({M}^{(k)}_{n,j}-1)}
%      \left(\widetilde{\sigma}_{n,j}^{(k)}\right)^2}{{B}^{(k)}_{n,j}} \nonumber \\
%      & + 4\eta_k^2\beta^2 \left(e_y^{(k)}\right)\left(e_y^{(k)}-1\right) \sum_{e=1}^{e_y^{(k)}}\Bigg\Vert \mathbf{w}^{(k),e-1}_{y} - \mathbf{w}^{(k)}\Bigg\Vert^2+ 4 \eta_k^2 \left(e_y^{(k)}\right)\left(e_y^{(k)}-1\right) \sum_{e=1}^{e_y^{(k)}} \Bigg\Vert\nabla F^{(k)}_y(\mathbf{w}^{(k)})
%     \Bigg\Vert^2.
% \end{align}
Assuming $\eta_k \leq \left(2\beta\sqrt{e_y^{(k)}(e_y^{(k)}-1)}\right)^{-1},\forall n$, we get
\begin{align}
    \nonumber\sum_{e=1}^{e_y^{(k)}} \mathbb{E}_k\left[\left\Vert
   \mathbf{w}^{(k)}-
     \mathbf{w}^{(k),e-1}_{y}
    \right\Vert^2\right] \leq& 
    \frac{4  \eta_k^2 \left(e_y^{(k)}\right)\left(e_y^{(k)}-1\right)\left(1-\frac{{B}^{(k)}_y}{{D}^{(k)}_y}\right) \frac{({\sigma}_{y}^{(k)})^2}{{B}^{(k)}_y}\Theta_y^2}{1- 4\eta_k^2\beta^2 e_y^{(k)}\left(e_y^{(k)}-1\right)}\\& + \frac{4 \eta_k^2 \left(e_y^{(k)}\right)^2\left(e_y^{(k)}-1\right)}{1- 4\eta_k^2\beta^2 e_y^{(k)}\left(e_y^{(k)}-1\right)}  \Bigg\Vert\nabla F^{(k)}_y(\mathbf{w}^{(k)})
    \Bigg\Vert^2.
\end{align}
Replacing the above expression in~\eqref{eq:proof45}, we get:
\begin{equation}
    \begin{aligned}
    &\left\Vert
    \nabla{F^{(k)}(\mathbf{w}^{(k)})}-
   \sum_{m\in\mathcal{M}}
   \sum_{y\in \mathcal{N}_m^{(k)}}\frac{{D}^{(k)}_y}{{D}^{(k)} e^{(k)}_y} \sum_{e=1}^{e^{(k)}_y}  {\nabla  F^{(k)}_y(\mathbf{w}^{(k),e-1}_{y})}
    \right\Vert^2
    \leq\beta^2\sum_{m\in\mathcal{M}}
   \sum_{y\in \mathcal{N}_m^{(k)}}\frac{{D}^{(k)}_y}{D^{(k)} e_y^{(k)}}\sum_{e=1}^{e^{(k)}_y} \mathbb{E}_k\left[\left\Vert
   \mathbf{w}^{(k)}-
     \mathbf{w}^{(k),e-1}_{y}
    \right\Vert^2\right]
    \\&
    \leq \beta^2\sum_{m\in\mathcal{M}}
   \sum_{y\in \mathcal{N}_m^{(k)}}\frac{{D}^{(k)}_y}{D^{(k)} e_y^{(k)}}\vast[   \frac{4  \eta_k^2 \left(e_y^{(k)}\right)\left(e_y^{(k)}-1\right)\left(1-\frac{{B}^{(k)}_y}{{D}^{(k)}_y}\right) \frac{({\sigma}_{y}^{(k)})^2}{{B}^{(k)}_y}\Theta_y^2}{1- 4\eta_k^2\beta^2 e_y^{(k)}\left(e_y^{(k)}-1\right)}\\& + \frac{4 \eta_k^2 \left(e_y^{(k)}\right)^2\left(e_y^{(k)}-1\right)}{1- 4\eta_k^2\beta^2 e_y^{(k)}\left(e_y^{(k)}-1\right)}  \Bigg\Vert\nabla F^{(k)}_y(\mathbf{w}^{(k)})
    \Bigg\Vert^2\vast]
    \\&=\sum_{m\in\mathcal{M}}
   \sum_{y\in \mathcal{N}_m^{(k)}}\frac{{D}^{(k)}_y}{D^{(k)}e_y^{(k)}}  \frac{4 \beta^2 \eta_k^2 e_y^{(k)}\left(e_y^{(k)}-1\right)\left(1-\frac{{B}^{(k)}_y}{{D}^{(k)}_y}\right) \frac{({\sigma}_{y}^{(k)})^2}{{B}^{(k)}_y}\Theta_y^2}{1- 4\eta_k^2\beta^2 e_y^{(k)}\left(e_y^{(k)}-1\right)}
   \\&+ \sum_{m\in\mathcal{M}}
   \sum_{y\in \mathcal{N}_m^{(k)}}\frac{{D}^{(k)}_y}{D^{(k)}}\frac{4 \beta^2\eta_k^2 e_y^{(k)}\left(e_y^{(k)}-1\right)}{1- 4\eta_k^2\beta^2 e_y^{(k)}\left(e_y^{(k)}-1\right)}  \Bigg\Vert\nabla F^{(k)}_y(\mathbf{w}^{(k)})
    \Bigg\Vert^2
    \\&
    \leq \sum_{m\in\mathcal{M}}
   \sum_{y\in \mathcal{N}_m^{(k)}}\frac{{D}^{(k)}_y}{D^{(k)}e_y^{(k)}}  \frac{4 \beta^2 \eta_k^2 e_y^{(k)}\left(e_y^{(k)}-1\right)\left(1-\frac{{B}^{(k)}_y}{{D}^{(k)}_y}\right) \frac{({\sigma}_{y}^{(k)})^2}{{B}^{(k)}_y}\Theta_y^2}{1- 4\eta_k^2\beta^2 e_y^{(k)}\left(e_y^{(k)}-1\right)}
   \\&+\frac{4 \beta^2\eta_k^2 e^{(k)}_{\mathsf{max}}\left(e^{(k)}_{\mathsf{max}}-1\right)}{1- 4\eta_k^2\beta^2 e^{(k)}_{\mathsf{max}}\left(e^{(k)}_{\mathsf{max}}-1\right)}  \sum_{m\in\mathcal{M}}
   \sum_{y\in \mathcal{N}_m^{(k)}}\frac{{D}^{(k)}_y}{D^{(k)}} \Bigg\Vert\nabla F^{(k)}_y(\mathbf{w}^{(k)})
    \Bigg\Vert^2\\&
    \leq \sum_{m\in\mathcal{M}}
   \sum_{y\in \mathcal{N}_m^{(k)}}\frac{{D}^{(k)}_y}{D^{(k)}}  \frac{4 \beta^2 \eta_k^2 \left(e_y^{(k)}-1\right)\left(1-\frac{{B}^{(k)}_y}{{D}^{(k)}_y}\right) \frac{({\sigma}_{y}^{(k)})^2}{{B}^{(k)}_y}\Theta_y^2}{1- 4\eta_k^2\beta^2 e_y^{(k)}\left(e_y^{(k)}-1\right)}
   \\&+\frac{4 \beta^2\eta_k^2 e^{(k)}_{\mathsf{max}}\left(e^{(k)}_{\mathsf{max}}-1\right)}{1- 4\eta_k^2\beta^2 e^{(k)}_{\mathsf{max}}\left(e^{(k)}_{\mathsf{max}}-1\right)}  \left( \zeta_1  \left\Vert \nabla F^{(k)}(\mathbf{w}^{(k)})
    \right\Vert^2 + \zeta_2 \right),
    \end{aligned}
\end{equation}
where we used
    $e^{(k)}_{\mathsf{max}}=\max_{n\in\mathcal{N}}\{e^{(k)}_y\}$, and the bounded dissimilarity of local gradients (Assumption~\ref{Assup:Dissimilarity}).
   Replacing the above expression in~\eqref{eq:proof44} gives us
   \begin{equation}
\begin{aligned}
 &\mathbb E_k \left[F^{(k)}(\mathbf{w}^{(k+1)})\right] \leq F^{(k)}(\mathbf{w}^{(k)}) -
 \left(\frac{\eta_k}{4}\sum_{m\in\mathcal{M}}\sum_{y'\in \mathcal{N}^{(k)}_m}\frac{{{D}}^{(k)}_{y'}e^{(k)}_{y'}}{{D}^{(k)}}\right)\left\Vert\nabla{F^{(k)}(\mathbf{w}^{(k)})}\right\Vert^2  \\&
   +\left(\frac{\eta_k}{2}\sum_{m\in\mathcal{M}}\sum_{y'\in \mathcal{N}^{(k)}_m}\frac{{{D}}^{(k)}_{y'}e^{(k)}_{y'}}{{D}^{(k)}}\right) \vast[ \sum_{m\in\mathcal{M}}
   \sum_{y\in \mathcal{N}_m^{(k)}}\frac{{D}^{(k)}_y}{D^{(k)}}  \frac{4 \beta^2 \eta_k^2 \left(e_y^{(k)}-1\right)\left(1-\frac{{B}^{(k)}_y}{{D}^{(k)}_y}\right) \frac{({\sigma}_{y}^{(k)})^2}{{B}^{(k)}_y}\Theta_y^2}{1- 4\eta_k^2\beta^2 e_y^{(k)}\left(e_y^{(k)}-1\right)}
   \\&+\frac{4 \beta^2\eta_k^2 e^{(k)}_{\mathsf{max}}\left(e^{(k)}_{\mathsf{max}}-1\right)}{1- 4\eta_k^2\beta^2 e^{(k)}_{\mathsf{max}}\left(e^{(k)}_{\mathsf{max}}-1\right)}  \left( \zeta_1  \left\Vert \nabla F^{(k)}(\mathbf{w}^{(k)})
    \right\Vert^2 + \zeta_2 \right)\vast]\\&
   +\frac{3}{2}{\eta_k}\sum_{m\in\mathcal{M}}\sum_{y'\in \mathcal{N}^{(k)}_m}\frac{{{D}}^{(k)}_{y'}e^{(k)}_{y'}}{{D}^{(k)}} \mathbb{E}_k\left[\left\Vert \bm{c}_{m_k}^{(k)}\right\Vert^2\right]+2{\beta \eta^2_{k}} 
    \left(\sum_{m\in\mathcal{M}}\sum_{y'\in \mathcal{N}^{(k)}_m}\frac{{D}^{(k)}_{y'}e^{(k)}_{y'}}{{D}^{(k)} }\right)^2 \sum_{m\in\mathcal{M}}
   \sum_{y\in \mathcal{N}_m^{(k)}}\left(\frac{{D}^{(k)}_y}{{D}^{(k)} \sqrt{e^{(k)}_y}}\right)^2
\left(1-\frac{{B}^{(k)}_y}{{D}^{(k)}_y}\right) \frac{({\sigma}_{y}^{(k)})^2}{{B}^{(k)}_y} \Theta^2_y, \hspace{-5mm}
    \end{aligned}
\end{equation}
which implies
 \begin{align}\label{eq:thirdlossIneq}
\hspace{-14mm}
 \mathbb E_k &\left[F^{(k)} (\mathbf{w}^{(k+1)})\right] \leq F^{(k)}(\mathbf{w}^{(k)}) + \frac{\eta_{k}}{2}\sum_{n\in \mathcal{N}}\frac{{D}^{(k)}_{y}e^{(k)}_{y}}{{D}^{(k)} } \left(\frac{4 \eta_k^2\beta^2 \left(e_{\mathsf{max}}^{(k)}\right)\left(e_{\mathsf{max}}^{(k)}-1\right)}{1- 4\eta_k^2\beta^2 e_{\mathsf{max}}^{(k)}\left(e_{\mathsf{max}}^{(k)}-1\right)}\zeta_1-\frac{1}{2} \right)
 \left\Vert\nabla{F^{(k)}(\mathbf{w}^{(k)})}\right\Vert^2 \nonumber\\&\hspace{-8mm}
    \nonumber + \frac{\eta_{k}}{2}\sum_{m\in\mathcal{M}}\sum_{y'\in \mathcal{N}^{(k)}_m}\frac{{D}^{(k)}_{y'}e^{(k)}_{y'}}{{D}^{(k)} }  \vast(\sum_{m\in\mathcal{M}}
   \sum_{y\in \mathcal{N}_m^{(k)}}\frac{{D}^{(k)}_y}{D^{(k)}}  \frac{4 \beta^2 \eta_k^2 \left(e_y^{(k)}-1\right)\left(1-\frac{{B}^{(k)}_y}{{D}^{(k)}_y}\right) \frac{({\sigma}_{y}^{(k)})^2}{{B}^{(k)}_y}\Theta_y^2}{1- 4\eta_k^2\beta^2 e_{\mathsf{max}}^{(k)}\left(e_{\mathsf{max}}^{(k)}-1\right)}
     + \frac{4\zeta_2 \eta_k^2\beta^2 \left(e_{\mathsf{max}}^{(k)}\right)\left(e_{\mathsf{max}}^{(k)}-1\right)}{1- 4\eta_k^2\beta^2 e_{\mathsf{max}}^{(k)}\left(e_{\mathsf{max}}^{(k)}-1\right)}
     \vast)
     \nonumber\\&\hspace{-8mm}
     +\frac{3}{2}{\eta_k}\sum_{m\in\mathcal{M}}\sum_{y'\in \mathcal{N}^{(k)}_m}\frac{{{D}}^{(k)}_{y'}e^{(k)}_{y'}}{{D}^{(k)}} \mathbb{E}_k\left[\left\Vert \bm{c}_{m_k}^{(k)}\right\Vert^2\right]+2{\beta \eta^2_{k}} 
    \left(\sum_{m\in\mathcal{M}}\sum_{y'\in \mathcal{N}^{(k)}_m}\frac{{D}^{(k)}_{y'}e^{(k)}_{y'}}{{D}^{(k)} }\right)^2 \sum_{m\in\mathcal{M}}
   \sum_{y\in \mathcal{N}_m^{(k)}}\left(\frac{{D}^{(k)}_y}{{D}^{(k)} \sqrt{e^{(k)}_y}}\right)^2
\left(1-\frac{{B}^{(k)}_y}{{D}^{(k)}_y}\right) \frac{({\sigma}_{y}^{(k)})^2}{{B}^{(k)}_y} \Theta^2_y.\hspace{-14mm}
    \end{align}
    
    Assuming $\frac{4 \eta_k^2\beta^2 \left(e_{\mathsf{max}}^{(k)}\right)\left(e_{\mathsf{max}}^{(k)}-1\right)}{1- 4\eta_k^2\beta^2 e_{\mathsf{max}}^{(k)}\left(e_{\mathsf{max}}^{(k)}-1\right)}\zeta_1 \leq \frac{1}{4} $, implying $\frac{1}{1- 4\eta_k^2\beta^2 e_{\mathsf{max}}^{(k)}\left(e_{\mathsf{max}}^{(k)}-1\right)} \leq \frac{1+4\zeta_1}{\zeta_1}\leq 5$, which can be obtained under the step size choice of $\eta_k \leq \frac{1}{2\beta} \sqrt{ \frac{1}{(4\zeta_1+1)\left( e_{\mathsf{max}}^{(k)}\left(e_{\mathsf{max}}^{(k)}-1\right)\right)}}$,  we get
    \begin{align}\label{eq:thirdlossIneq}
 \mathbb E_k &\left[F^{(k)} (\mathbf{w}^{(k+1)})\right] \leq F^{(k)}(\mathbf{w}^{(k)}) -\frac{1}{8}{\eta_{k}}\left\Vert\nabla{F^{(k)}(\mathbf{w}^{(k)})}\right\Vert^2\sum_{n\in \mathcal{N}}\frac{{D}^{(k)}_{y}e^{(k)}_{y}}{{D}^{(k)} }
  \nonumber\\&
    \nonumber + \frac{\eta_{k}}{2}\sum_{m\in\mathcal{M}}\sum_{y'\in \mathcal{N}^{(k)}_m}\frac{{D}^{(k)}_{y'}e^{(k)}_{y'}}{{D}^{(k)} }  \vast(\sum_{m\in\mathcal{M}}
   \sum_{y\in \mathcal{N}_m^{(k)}}\frac{{D}^{(k)}_y}{D^{(k)}} {20 \beta^2 \eta_k^2 \left(e_y^{(k)}-1\right)\left(1-\frac{{B}^{(k)}_y}{{D}^{(k)}_y}\right) \frac{({\sigma}_{y}^{(k)})^2}{{B}^{(k)}_y}\Theta_y^2}
     + {20\zeta_2 \eta_k^2\beta^2 \left(e_{\mathsf{max}}^{(k)}\right)\left(e_{\mathsf{max}}^{(k)}-1\right)}
     \vast)
     \nonumber\\&
     +\frac{3}{2}{\eta_k}\sum_{m\in\mathcal{M}}\sum_{y'\in \mathcal{N}^{(k)}_m}\frac{{{D}}^{(k)}_{y'}e^{(k)}_{y'}}{{D}^{(k)}} \mathbb{E}_k\left[\left\Vert \bm{c}_{m_k}^{(k)}\right\Vert^2\right]\nonumber \\&+2{\beta \eta^2_{k}} 
    \left(\sum_{m\in\mathcal{M}}\sum_{y'\in \mathcal{N}^{(k)}_m}\frac{{D}^{(k)}_{y'}e^{(k)}_{y'}}{{D}^{(k)} }\right)^2 \sum_{m\in\mathcal{M}}
   \sum_{y\in \mathcal{N}_m^{(k)}}\left(\frac{{D}^{(k)}_y}{{D}^{(k)} \sqrt{e^{(k)}_y}}\right)^2
\left(1-\frac{{B}^{(k)}_y}{{D}^{(k)}_y}\right) \frac{({\sigma}_{y}^{(k)})^2}{{B}^{(k)}_y} \Theta^2_y.\hspace{-14mm}
    \end{align}
    Rearranging the terms of the above bound gives the following  bound on the norm of gradient:
 \begin{align}\label{eq:final_one_mtep55}
 &\left\Vert\nabla{F^{(k)}(\mathbf{w}^{(k)})}\right\Vert^2  \leq 
      \frac{F^{(k)}(\mathbf{w}^{(k)}) - \mathbb E_k \left[F^{(k)} (\mathbf{w}^{(k+1)})\right]}{\eta_k\sum_{m\in\mathcal{M}}\sum_{y'\in \mathcal{N}^{(k)}_m}{{D}^{(k)}_{y'}e^{(k)}_{y'}}/(8D^{(k)})}
  \nonumber\\&\hspace{-6mm}+
   \sum_{m\in\mathcal{M}}
   \sum_{y\in \mathcal{N}_m^{(k)}}\frac{{D}^{(k)}_y}{D^{(k)}} {80 \beta^2 \eta_k^2 \left(e_y^{(k)}-1\right)\left(1-\frac{{B}^{(k)}_y}{{D}^{(k)}_y}\right) \frac{({\sigma}_{y}^{(k)})^2}{{B}^{(k)}_y}\Theta_y^2}
     + {80\zeta_2 \eta_k^2\beta^2 \left(e_{\mathsf{max}}^{(k)}\right)\left(e_{\mathsf{max}}^{(k)}-1\right)}
     \nonumber\\&\hspace{-8mm}
     +24\mathbb{E}_k\left[\left\Vert \bm{c}_{m_k}^{(k)}\right\Vert^2\right]+16{\beta \eta_{k}} 
    \left(\sum_{m\in\mathcal{M}}\sum_{y'\in \mathcal{N}^{(k)}_m}\frac{{D}^{(k)}_{y'}e^{(k)}_{y'}}{{D}^{(k)} }\right) \sum_{m\in\mathcal{M}}
   \sum_{y\in \mathcal{N}_m^{(k)}}\left(\frac{{D}^{(k)}_y}{{D}^{(k)} \sqrt{e^{(k)}_y}}\right)^2
\left(1-\frac{{B}^{(k)}_y}{{D}^{(k)}_y}\right) \frac{({\sigma}_{y}^{(k)})^2}{{B}^{(k)}_y} \Theta^2_y.
 \end{align}
 
%  where $\Gamma^{(k)}=\frac{\eta_{k}}{2}\sum_{n\in \mathcal{N}}\frac{{D}^{(k)}_{y}e^{(k)}_{y}}{{D}^{(k)} }$.

%   The devices use $\mathbf{w}^{(0)}$ for model inference throughout the waiting period for which the first global aggregation starts and during the first global aggregation, which accounts for $T^{\mathsf{Tot},(1)}+\Omega^{(1)}$ amount of time. Based on a similar argument, the global parameter $\mathbf{w}^{(k)}$ will be used during  $T^{\mathsf{Tot},(k+1)}+\Omega^{(k+1)}$ time instances. Thus, the cumulative average of the global model performance at the devices can be written as follows:
%   \begin{align}\label{eq:timedomain}
%      \frac{1}{{K}} \sum_{k=0}^{K-1} \Vert \nabla F^{(k)}(\mathbf{w}({t})) \Vert^2=&\frac{1}{{K}} \Bigg[ \left(T^{\mathsf{Tot},(1)}+\Omega^{(1)} \right)\Vert \nabla F^{(k)}(\mathbf{w}^{(0)}) \Vert^2
%      \nonumber \\&+
%      \left(T^{\mathsf{Tot},(2)}+\Omega^{(2)} \right)\Vert \nabla F^{(k)}(\mathbf{w}^{(1)}) \Vert^2+
%       \left(T^{\mathsf{Tot},(3)}+\Omega^{(3)} \right)\Vert \nabla F^{(k)}(\mathbf{w}^{(2)}) \Vert^2
%      \nonumber  \\&+ \cdots+ \left(T^{\mathsf{Tot},(K)}+\Omega^{(K)} \right)\Vert \nabla F^{(k)}(\mathbf{w}^{(K-1)}) \Vert^2\Bigg]\\
%      &=\frac{1}{{K}} \sum_{k=0}^{K-1} \Vert \nabla F^{(k)}(\mathbf{w}^{(k)}) \Vert^2 .
%  \end{align}
%  FIX THIS UP! This SUM INDEX is NOT Consistent!!!!
 
Note that $F^{(k)}(\mathbf{w}^{(k)})$ in the first term on the right hand side of the above bound is in fact the loss under which the global aggregation started from, i.e., $F^{(k)}(\mathbf{w}^{(k)} |, {\mathcal{D}}^{(k)})$ and  $F^{(k)}(\mathbf{w}^{(k+1)})$ is the loss under which the global aggregation concludes, i.e., $F^{(k)}(\mathbf{w}^{(k+1)} | {\mathcal{D}}^{(k)})$. We next find the connection between 
$F^{(k)}(\mathbf{w}^{(k)} | {\mathcal{D}}^{(k)})$ and the actual loss under which iteration $k-1$ has been concluded, i.e., $F^{(k-1)}(\mathbf{w}^{(k)} | {\mathcal{D}}^{(k-1)})$. Using the definition of model/concept drift in Definition~\ref{def:conceptdrift}, since the global loss is defined as the loss per data point we have:
\begin{align}\label{eq:driftLoss}
    F^{(k)}(\mathbf{w}^{(k)} | {\mathcal{D}}^{(k)})&
    = \sum_{m\in\mathcal{M}}\sum_{y \in \mathcal{N}^{(k)}_m} \Bigg[\frac{{D}_y^{(k)}}{{D}^{(k)}}F_{y}^{(k)}(\mathbf{w}^{(k)}|{\mathcal{D}}_y^{(k)})
   -
   \frac{{D}_y^{(k-1)}}{{D}^{(k-1)}}
   F^{(k-1)}_{y}(\mathbf{w}^{(k)}|{{D}}_y^{(k-1)})+\frac{{D}_y^{(k-1)}}{{D}^{(k-1)}}
   F^{(k-1)}_{y}(\mathbf{w}^{(k)}|{{D}}_y^{(k-1)})\Bigg]
   \nonumber 
   \\&\leq \sum_{m\in\mathcal{M}}\sum_{y \in \mathcal{N}^{(k)}_m} \Delta_y^{(k)} + \sum_{m\in\mathcal{M}}\sum_{y \in \mathcal{N}^{(k)}_m}
   \frac{{D}_y^{(k-1)}}{{D}^{(k-1)}}
   F^{(k-1)}_{y}(\mathbf{w}^{(k)}|{{D}}_y^{(k-1)})\nonumber
   \\&=\sum_{m\in\mathcal{M}}\sum_{y \in \mathcal{N}^{(k)}_m} \Delta_y^{(k)}+ F^{(k-1)}(\mathbf{w}^{(k)}|{\mathcal{D}}^{(k-1)})= \Delta^{(k)} + F^{(k-1)}(\mathbf{w}^{(k)}),~k\geq 1.
   \end{align}
% where $\Delta_y^{(k)}=\max_{t\in T^{\mathsf{Idle},(k)}} \Delta_y(t)$, $\forall n$ ($T^{\mathsf{Idle},(k)}$ denotes the set of idle times between global aggregations $k-1$ and $k$), and $\Delta^{(k)}=\sum_{n\in \mathcal{N}}\Delta_y^{(k)}$. 

 Taking total expectation from both hand sides of~\eqref{eq:final_one_mtep55} and taking the summation over global aggregation index implies
% {\small
 \begin{align}\label{eq:th1proof}
 &\hspace{-16mm}\frac{1}{K} \sum_{k=0}^{K-1} \mathbb{E} \left[\Vert \nabla F^{(k)}(\mathbf{w}^{({k})}) \Vert^2\right] \leq \frac{1}{K}
   \frac{F^{(0)}(\mathbf{w}^{(0)}) - F^{(0)}(\mathbf{w}^{(1)})}{\eta_k\sum_{m\in\mathcal{M}}\sum_{y'\in \mathcal{N}^{(k)}_m}{{D}^{(k)}_{y'}e^{(k)}_{y'}}/(8D^{(k)})}+ \frac{1}{K}\sum_{k=1}^{K-1}
   \frac{F^{(k-1)}(\mathbf{w}^{(k)}) - F^{(k)}(\mathbf{w}^{(k+1)})}{\eta_k\sum_{m\in\mathcal{M}}\sum_{y'\in \mathcal{N}^{(k)}_m}{{D}^{(k)}_{y'}e^{(k)}_{y'}}/(8D^{(k)})}
  \nonumber  \\&
   +\frac{1}{K}\sum_{k=1}^{K-1}\frac{  \Delta^{(k)}}{\eta_k\sum_{m\in\mathcal{M}}\sum_{y'\in \mathcal{N}^{(k)}_m}{{D}^{(k)}_{y'}e^{(k)}_{y'}}/(8D^{(k)})}
  \nonumber\\&\hspace{-6mm}+\frac{1}{K}\sum_{k=0}^{K-1}
   \sum_{m\in\mathcal{M}}
   \sum_{y\in \mathcal{N}_m^{(k)}}\frac{{D}^{(k)}_y}{D^{(k)}} {80 \beta^2 \eta_k^2 \left(e_y^{(k)}-1\right)\left(1-\frac{{B}^{(k)}_y}{{D}^{(k)}_y}\right) \frac{({\sigma}_{y}^{(k)})^2}{{B}^{(k)}_y}\Theta_y^2}
     +\frac{1}{K}\sum_{k=0}^{K-1} {80\zeta_2 \eta_k^2\beta^2 \left(e_{\mathsf{max}}^{(k)}\right)\left(e_{\mathsf{max}}^{(k)}-1\right)}
     \nonumber\\&\hspace{-8mm}
     +\frac{1}{K}\sum_{k=0}^{K-1} 24\mathbb{E}_k\left[\left\Vert \bm{c}_{m_k}^{(k)}\right\Vert^2\right]+\frac{1}{K}\sum_{k=0}^{K-1}16{\beta \eta_{k}} 
    \left(\sum_{m\in\mathcal{M}}\sum_{y'\in \mathcal{N}^{(k)}_m}\frac{{D}^{(k)}_{y'}e^{(k)}_{y'}}{{D}^{(k)} }\right) \sum_{m\in\mathcal{M}}
   \sum_{y\in \mathcal{N}_m^{(k)}}\left(\frac{{D}^{(k)}_y}{{D}^{(k)} \sqrt{e^{(k)}_y}}\right)^2
\left(1-\frac{{B}^{(k)}_y}{{D}^{(k)}_y}\right) \frac{({\sigma}_{y}^{(k)})^2}{{B}^{(k)}_y} \Theta^2_y.\hspace{-6mm}
 \end{align}
%  }
%  \vspace{-4mm}
% where  $F^{(-1)}(\mathbf{w}^{(0)})= F^{(0)}(\mathbf{w}^{(0)})$ is the initial loss under which the model training starts.
Assuming the choice of step size $\eta_k =\frac{\alpha}{{\sqrt{K e^{(k)}_{\mathsf{avg}}}}}$, where $e^{(k)}_{\mathsf{avg}}$ denotes the average number of local iterations across the participating devices, we get:
 \begin{align}\label{eq:th1proof58}
 &\hspace{-16mm}\frac{1}{K} \sum_{k=0}^{K-1} \mathbb{E} \left[\Vert \nabla F^{(k)}(\mathbf{w}^{({k})}) \Vert^2\right] \leq \frac{8 \sqrt{{e}^{\mathsf{max}}_{\mathsf{avg}}}}{\alpha \hat{e}^{\mathsf{min}}_{\mathsf{avg}}\sqrt{K}}
   \left({F^{(0)}(\mathbf{w}^{(0)}) - F^{{(k)}^\star}}\right)+\frac{8 \sqrt{{e}^{\mathsf{max}}_{\mathsf{avg}}}}{\alpha\hat{e}^{\mathsf{min}}_{\mathsf{avg}}\sqrt{K}}\sum_{k=1}^{K-1}{  \Delta^{(k)}}
  \nonumber\\&\hspace{-6mm}+\frac{80 \beta^2 \alpha^2}{K^2{e}^{\mathsf{min}}_{\mathsf{avg}}}\sum_{k=0}^{K-1}
   \sum_{m\in\mathcal{M}}
   \sum_{y\in \mathcal{N}_m^{(k)}}\frac{{D}^{(k)}_y}{D^{(k)}} { \left(e_y^{(k)}-1\right)\left(1-\frac{{B}^{(k)}_y}{{D}^{(k)}_y}\right) \frac{({\sigma}_{y}^{(k)})^2}{{B}^{(k)}_y}\Theta_y^2}
     +\frac{80 \beta^2 \alpha^2}{K^2{e}^{\mathsf{min}}_{\mathsf{avg}}}\sum_{k=0}^{K-1} {\zeta_2 \left(e_{\mathsf{max}}^{(k)}\right)\left(e_{\mathsf{max}}^{(k)}-1\right)}
     \nonumber\\&\hspace{-8mm}
     +\frac{1}{K}\sum_{k=0}^{K-1} 24\underbrace{\mathbb{E}_k\left[\left\Vert \bm{c}_{m_k}^{(k)}\right\Vert^2\right]}_{(a)}+\frac{16 \beta \alpha \hat{e}^{\mathsf{max}}_{\mathsf{avg}}}{K\sqrt{K}\sqrt{{e}^{\mathsf{min}}_{\mathsf{avg}}}}\sum_{k=0}^{K-1}
   \sum_{m\in\mathcal{M}}
   \sum_{y\in \mathcal{N}_m^{(k)}}\left(\frac{{D}^{(k)}_y}{{D}^{(k)} \sqrt{e^{(k)}_y}}\right)^2
\left(1-\frac{{B}^{(k)}_y}{{D}^{(k)}_y}\right) \frac{({\sigma}_{y}^{(k)})^2}{{B}^{(k)}_y} \Theta^2_y.\hspace{-6mm}
 \end{align}
 
 We next aim to bound term $(a)$ in~\eqref{eq:th1proof58}. Let $\widetilde{\nabla F}_m^{(k)}$ denote the initial local aggregated gradient at ES $m$, containing its own gradient and the gradient of its associated users, and $\widehat{\nabla F}_m^{(k)}$ denote the gradient at ES $m$ after the consensus through P2P communications concludes. 
 
 As explained in~\eqref{eq:proof26}, we have
 \begin{equation}
    {\nabla {F}}_m^{\mathsf{L},(k)}= \sum_{m\in\mathcal{M}}     {\nabla {F}}_m^{\mathsf{A},(k)} + \bm{e}^{(k)}_m.
\end{equation}
 On the other hand, the evolution of the local gradient during the consensus process can be described as
 \begin{equation}\label{eq:con2s}
      {\nabla {F}}^{\mathsf{L},(k)}= \left(\mathbf{\Lambda}^{(k)}\right)^{\varphi^{(k)}} {\nabla {F}}^{\mathsf{A},(k)},
  \end{equation}
 where ${\nabla {F}}^{\mathsf{L},(k)}\in \mathbb{R}^{d \times d}$ denotes a matrix,\footnote{$d$ denotes the size/length of the gradient vector.} with its row $m$ containing the final vector of gradient at ES $m$ (i.e.,  $ {\nabla {F}}_m^{\mathsf{L},(k)}$), and $ \nabla {F}^{\mathsf{A},(k)} \in \mathbb{R}^{d \times d}$ denotes a matrix, with its row $m$ containing the initial gradient vector at ES $m$ (i.e., ${\nabla {F}}_m^{\mathsf{A},(k)}$). In~\eqref{eq:con2s}, $\bm{\Lambda}^{(k)}=[\lambda^{(k)}_{i,j}]_{i,j\in\mathcal{M}}$ is the consensus matrix across the ES network and $\varphi^{(k)}$ denotes the rounds of P2P communications at the aggregation round $k$.
 
  Let matrix ${{\nabla {F}}^{\mathsf{P},(k)}}$ denote
 the matrix of true/perfect average of the gradient vectors across the ESs, which is given by
 \begin{equation}\label{eq:aveDefProof}
     {{\nabla {F}}^{\mathsf{P},(k)}}= \frac{\textbf{1}_{M} \textbf{1}_{M}^\top {\nabla {F}}_m^{\mathsf{A},(k)}}{M}.
 \end{equation}
 Also, let ${{\nabla {F}}^{\mathsf{P},(k)}}_m$ denote the $m$-th row of  ${{\nabla {F}}^{\mathsf{P},(k)}}$.
 Since the consensus error is zero mean \cite{xiao2004fast} across the ESs, we have $ \mathbf{1}^\top \left({\nabla {F}}^{\mathsf{L},(k)} -{\nabla {F}}^{\mathsf{A},(k)}\right)=\mathbf{0}$ implying that $(\mathbf{1} \mathbf{1}^\top)\left({\nabla {F}}^{\mathsf{L},(k)} -{\nabla {F}}^{\mathsf{A},(k)}\right)=\mathbf{0}$. Thus, we get
    \begin{equation}
  \begin{aligned}
      {\nabla {F}}^{\mathsf{L},(k)} -{{\nabla {F}}^{\mathsf{P},(k)}}&=\left(\mathbf{I}-\frac{\mathbf{1} \mathbf{1}^\top}{M}\right)\left({\nabla {F}}^{\mathsf{L},(k)} -{{\nabla {F}}^{\mathsf{P},(k)}}\right)
      =\left(\mathbf{I}-\frac{\mathbf{1} \mathbf{1}^\top}{M}\right)\left( \left(\mathbf{\Lambda}^{(k)}\right)^{\varphi^{(k)}}{\nabla {F}}^{\mathsf{A},(k)}-{{\nabla {F}}^{\mathsf{P},(k)}}\right)\\
      &=\left(\mathbf{I}-\frac{\mathbf{1} \mathbf{1}^\top}{M}\right)\left(\left(\mathbf{\Lambda}^{(k)}\right)^{\varphi^{(k)}} {\nabla {F}}^{\mathsf{A},(k)}-\left(\mathbf{\Lambda}^{(k)}\right)^{\varphi^{(k)}}{{\nabla {F}}^{\mathsf{P},(k)}}\right)
      \\&
      =\left(\left(\mathbf{\Lambda}^{(k)}\right)^{\varphi^{(k)}}-\frac{\mathbf{1} \mathbf{1}^\top}{M}\right)\left({\nabla {F}}^{\mathsf{A},(k)}-{{\nabla {F}}^{\mathsf{P},(k)}}\right),
      \end{aligned}
  \end{equation}
where the above derivations are obtained based on the following properties: (i) $\left(\mathbf{\Lambda}^{(k)}\right)^{\varphi^{(k)}} {{\nabla {F}}^{\mathsf{P},(k)}}={{\nabla {F}}^{\mathsf{P},(k)}}$, and (ii) $\frac{\mathbf{1} \mathbf{1}^\top}{M} \left(\mathbf{\Lambda}^{(k)}\right)^{\varphi^{(k)}}=\frac{\mathbf{1} \mathbf{1}^\top}{M}$.
We finally get
\begin{equation}
\begin{aligned}
    \big\Vert\bm{c}^{(k)}_m\big\Vert ^2 &\leq \textrm{trace}\left(\left({\nabla {F}}^{\mathsf{A},(k)}-{{\nabla {F}}^{\mathsf{P},(k)}}\right)^\top\left(\left(\mathbf{\Lambda}^{(k)}\right)^{\varphi^{(k)}}-\frac{\mathbf{1} \mathbf{1}^\top}{M}\right)^2\left({\nabla {F}}^{\mathsf{A},(k)}-{{\nabla {F}}^{\mathsf{P},(k)}}\right)\right)\\&
    =\textrm{trace}\left(\left({\nabla {F}}^{\mathsf{A},(k)}-{{\nabla {F}}^{\mathsf{P},(k)}}\right)^\top\left(\left(\mathbf{\Lambda}^{(k)}\right)^{\varphi^{(k)}}\left[\mathbf{I}-\frac{\mathbf{1} \mathbf{1}^\top}{M}\right]\right)^2\left({\nabla {F}}^{\mathsf{A},(k)}-{{\nabla {F}}^{\mathsf{P},(k)}}\right)\right)
    \\&
     =\textrm{trace}\left(\left({\nabla {F}}^{\mathsf{A},(k)}-{{\nabla {F}}^{\mathsf{P},(k)}}\right)^\top\left(\left(\mathbf{\Lambda}^{(k)}\right)^{\varphi^{(k)}}\left[\mathbf{I}-\frac{\mathbf{1} \mathbf{1}^\top}{M}\right]^{\varphi^{(k)}}\right)^2\left({\nabla {F}}^{\mathsf{A},(k)}-{{\nabla {F}}^{\mathsf{P},(k)}}\right)\right)
    \\&
    \leq (\lambda^{(k)})^{2\varphi^{(k)}} \sum_{m'\in\mathcal{M}} \Vert {\nabla {F}}^{\mathsf{A},(k)}_{m'} -\nabla {F}^{\mathsf{P},(k)}_{m'} \big\Vert^2
    \\&\leq  \frac{1}{M} (\lambda^{(k)})^{2\varphi^{(k)}} \sum_{m',m''\in\mathcal{M}} \Vert {\nabla {F}}^{\mathsf{A},(k)}_{m'} -\nabla {F}^{\mathsf{A},(k)}_{m''} \big\Vert^2\\&\leq  {M} (\lambda^{(k)})^{2\varphi^{(k)}} \max_{m',m''\in\mathcal{M}} \Vert {\nabla {F}}^{\mathsf{A},(k)}_{m'} -\nabla {F}^{\mathsf{A},(k)}_{m''} \big\Vert^2\leq {M} (\lambda^{(k)})^{2\varphi^{(k)}} \left(\Xi^{(k)}\right)^2.
    \end{aligned}
\end{equation}  

Replacing the above result back in~\eqref{eq:th1proof58}, we get the final result
  \begin{align}\label{eq:finalResTheo}
 &\hspace{-16mm}\frac{1}{K} \sum_{k=0}^{K-1} \mathbb{E} \left[\Vert \nabla F^{(k)}(\mathbf{w}^{({k})}) \Vert^2\right] \leq \frac{8 \sqrt{{e}^{\mathsf{max}}_{\mathsf{avg}}}}{\alpha \hat{e}^{\mathsf{min}}_{\mathsf{avg}}\sqrt{K}}
   \left({F^{(0)}(\mathbf{w}^{(0)}) - F^{{(k)}^\star}}\right)+\frac{8 \sqrt{{e}^{\mathsf{max}}_{\mathsf{avg}}}}{\alpha\hat{e}^{\mathsf{min}}_{\mathsf{avg}}\sqrt{K}}\sum_{k=1}^{K-1}{  \Delta^{(k)}}
  \nonumber\\&\hspace{-6mm}+\frac{80 \beta^2 \alpha^2}{K^2{e}^{\mathsf{min}}_{\mathsf{avg}}}\sum_{k=0}^{K-1}
   \sum_{m\in\mathcal{M}}
   \sum_{y\in \mathcal{N}_m^{(k)}}\frac{{D}^{(k)}_y}{D^{(k)}} { \left(e_y^{(k)}-1\right)\left(1-\frac{{B}^{(k)}_y}{{D}^{(k)}_y}\right) \frac{({\sigma}_{y}^{(k)})^2}{{B}^{(k)}_y}\Theta_y^2}
     +\frac{80 \beta^2 \alpha^2}{K^2{e}^{\mathsf{min}}_{\mathsf{avg}}}\sum_{k=0}^{K-1} {\zeta_2 \left(e_{\mathsf{max}}^{(k)}\right)\left(e_{\mathsf{max}}^{(k)}-1\right)}
     \nonumber\\&\hspace{-8mm}
     +\frac{1}{K}\sum_{k=0}^{K-1} 24{M} (\lambda^{(k)})^{2\varphi^{(k)}} \left(\Xi^{(k)}\right)^2+\frac{16 \beta \alpha \hat{e}^{\mathsf{max}}_{\mathsf{avg}}}{K\sqrt{K}\sqrt{{e}^{\mathsf{min}}_{\mathsf{avg}}}}\sum_{k=0}^{K-1}
   \sum_{m\in\mathcal{M}}
   \sum_{y\in \mathcal{N}_m^{(k)}}\left(\frac{{D}^{(k)}_y}{{D}^{(k)} \sqrt{e^{(k)}_y}}\right)^2
\left(1-\frac{{B}^{(k)}_y}{{D}^{(k)}_y}\right) \frac{({\sigma}_{y}^{(k)})^2}{{B}^{(k)}_y} \Theta^2_y.\hspace{-6mm}
 \end{align}

 \section{Proof of Corollary~\ref{co:main}}\label{app:co}
\label{subsection:Appen-Corollary}

\noindent Assuming $\Delta^{(k)}\leq \frac{\Upsilon}{K}$ and considering 
the assumption $\left(1-\frac{{B}^{(k)}_y}{{D}^{(k)}_y}\right) \frac{({\sigma}_{y}^{(k)})^2}{{B}^{(k)}_y}\Theta_y^2\leq \vartheta$ simplifies the bound in~\eqref{eq:finalResTheoMain} as follows:
\begin{align}\label{eq:co1}
 &\hspace{-16mm}\frac{1}{K} \sum_{k=0}^{K-1} \mathbb{E} \left[\Vert \nabla F^{(k)}(\mathbf{w}^{({k})}) \Vert^2\right] \leq \frac{8 \sqrt{{e}^{\mathsf{max}}_{\mathsf{avg}}}}{\alpha \hat{e}^{\mathsf{min}}_{\mathsf{avg}}\sqrt{K}}
   \left({F^{(0)}(\mathbf{w}^{(0)}) - F^{{(k)}^\star}}\right)+\frac{8\Upsilon \sqrt{{e}^{\mathsf{max}}_{\mathsf{avg}}}}{\alpha\hat{e}^{\mathsf{min}}_{\mathsf{avg}}\sqrt{K}}
  \nonumber\\&\hspace{-6mm}+\frac{80 \beta^2 \alpha^2}{K^2{e}^{\mathsf{min}}_{\mathsf{avg}}}
K \left(e_{\mathsf{max}}-1\right)\vartheta
     +\frac{80 \beta^2 \alpha^2}{K^2{e}^{\mathsf{min}}_{\mathsf{avg}}}K {\zeta_2 \left(e_{\mathsf{max}}\right)\left(e_{\mathsf{max}}-1\right)}
     \nonumber\\&\hspace{-8mm}
     +\frac{1}{K}\sum_{k=0}^{K-1} 24{M} (\lambda^{(k)})^{2\varphi^{(k)}} \left(\Xi^{(k)}\right)^2+\frac{16 \beta \alpha \hat{e}^{\mathsf{max}}_{\mathsf{avg}}}{K\sqrt{K}\sqrt{{e}^{\mathsf{min}}_{\mathsf{avg}}}}K\vartheta.\hspace{-6mm}
 \end{align}
Assuming ${\varphi^{(k)}} \geq \frac{1}{2} \left[\log_{\lambda^{(k)}}\left( \frac{\xi}{\sqrt{K}\left(\Xi^{(k)}\right)^2{M}^{(k)}}\right)\right]^+$ implies that (note that $\lambda^{(k)}<1$)
\begin{align}
&\nonumber {2\varphi^{(k)}} \geq \log_{\lambda^{(k)}}\left( \frac{\xi}{\sqrt{K}\left(\Xi^{(k)}\right)^2{M}^{(k)}}\right)\\
&\Rightarrow (\lambda^{(k)})^{2\varphi^{(k)}} \leq \frac{\xi}{\sqrt{K}\left(\Xi^{(k)}\right)^2{M}^{(k)}}\\
&\Rightarrow \nonumber {M}^{(k)} (\lambda^{(k)})^{2\varphi^{(k)}} \left(\Xi^{(k)}\right)^2 \leq\xi/\sqrt{K}.
\end{align}
Replacing this result back in~\eqref{eq:co1} yields
\begin{align}
&\hspace{-16mm}\frac{1}{K} \sum_{k=0}^{K-1} \mathbb{E} \left[\Vert \nabla F^{(k)}(\mathbf{w}^{({k})}) \Vert^2\right] \leq \frac{8 \sqrt{{e}^{\mathsf{max}}_{\mathsf{avg}}}}{\alpha \hat{e}^{\mathsf{min}}_{\mathsf{avg}}\sqrt{K}}
   \left({F^{(0)}(\mathbf{w}^{(0)}) - F^{{(k)}^\star}}\right)+\frac{8\Upsilon \sqrt{{e}^{\mathsf{max}}_{\mathsf{avg}}}}{\alpha\hat{e}^{\mathsf{min}}_{\mathsf{avg}}\sqrt{K}}
  \nonumber\\& \nonumber\hspace{-6mm}+\frac{80 \beta^2 \alpha^2}{K{e}^{\mathsf{min}}_{\mathsf{avg}}}
 \left(e_{\mathsf{max}}-1\right)\vartheta
     +\frac{80 \beta^2 \alpha^2}{K{e}^{\mathsf{min}}_{\mathsf{avg}}} {\zeta_2 \left(e_{\mathsf{max}}\right)\left(e_{\mathsf{max}}-1\right)}
     + 24\frac{\xi}{\sqrt{K}}+\frac{16 \beta \alpha \hat{e}^{\mathsf{max}}_{\mathsf{avg}}}{\sqrt{K}\sqrt{{e}^{\mathsf{min}}_{\mathsf{avg}}}}\vartheta,
      \end{align}
     and thus
     \begin{align}
&\frac{1}{K} \sum_{k=0}^{K-1} \mathbb{E} \left[\Vert \nabla F^{(k)}(\mathbf{w}^{({k})}) \Vert^2\right] =\mathcal{O}\left(\frac{ \sqrt{{e}^{\mathsf{max}}_{\mathsf{avg}}}}{\alpha \hat{e}^{\mathsf{min}}_{\mathsf{avg}}\sqrt{K}}\right)+\mathcal{O}\left(\frac{\Upsilon \sqrt{{e}^{\mathsf{max}}_{\mathsf{avg}}}}{\alpha\hat{e}^{\mathsf{min}}_{\mathsf{avg}}\sqrt{K}} \right)+ \mathcal{O}\left(\frac{ \beta^2 \alpha^2}{K{e}^{\mathsf{min}}_{\mathsf{avg}}}
 \left(e_{\mathsf{max}}-1\right)\vartheta\right) \nonumber\\&
 + \mathcal{O}\left(
     \frac{\beta^2 \alpha^2}{K{e}^{\mathsf{min}}_{\mathsf{avg}}} {\zeta_2 \left(e_{\mathsf{max}}\right)\left(e_{\mathsf{max}}-1\right)}\right)+ \mathcal{O}\left(\frac{\xi}{\sqrt{K}}\right)+\mathcal{O}\left(\frac{\beta \alpha \hat{e}^{\mathsf{max}}_{\mathsf{avg}}}{\sqrt{K}\sqrt{{e}^{\mathsf{min}}_{\mathsf{avg}}}}\vartheta \right)\nonumber \\&
     = \mathcal{O}\left(1/\sqrt{K}\right), \hspace{-6mm}
 \end{align}
which concludes the proof.

\section{Characterization of the Gradient of the Critic's Loss}
\label{subsection:Appen-LossCritic}
%To compute the critic's loss, the proposed TD solution aims to obtain a fixpoint $\omega$ of the Bellman operator $B_\theta$ over the state distribution. 
Inspired by \cite{wai2020provably}, our proof for the Lipschitz continuity property of the MSPBE loss function's gradient is presented as below. \textcolor{black}{First, the loss function in \eqref{Equa:loss} can be re-written as
\begin{equation}
E_\theta(\omega) = \frac{1}{2} || V_\omega -\Lambda (B_\theta V_\omega)||^2 = \frac{1}{2} || \Lambda (V_\omega - B_\theta V_\omega)  ||^2, (since~V_\omega =\Lambda V_\omega),
\end{equation}
which is upper-bounded by its expectation value given by the loss across the entire system state vector $H$: 
\begin{equation}
\label{Equation:Loss_define}
 E_\theta(\omega) \leq \frac{1}{2} || \Lambda (V_\omega - B_\theta V_\omega)  ||^2_H, 
\end{equation}
where $H$ is the diagonal matrix whose elements are within the stationary system state distribution $d^{\pi_\theta(s)}$, $\forall s \in \mathcal{S}$ as defined in section~\ref{subsection:CriticDesign}. Under this stationary condition, \eqref{Equation:Loss_define} can be further analyzed as:
\begin{align}
&\hspace{-15mm} E_\theta(\omega) = \frac{1}{2} \left[ \Lambda(V_\omega - B_\theta V_\omega) \right]^\top H  \left[ \Lambda(V_\omega - B_\theta V_\omega) \right]
 \\&= \frac{1}{2} (V_\omega - B_\theta V_\omega)^\top \Lambda^\top \Eta  \Lambda (V_\omega - B_\theta V_\omega) \\&= \frac{1}{2} || \bar{\Lambda}\sqrt{\Eta}  (V_\omega - B_\theta V_\omega) ||_2^2,
\hspace{-6mm}\end{align}
where $\bar{\Lambda} = \sqrt{\Eta} \upphi^\top (\upphi \Eta \upphi^\top)^{-1} \upphi \sqrt{\Eta} $ is an orthogonal projector of a column element of $\sqrt{\Eta} \upphi^\top$.} Then, the Hessian product of the loss function is given by
\begin{equation}
\mathds{H}(\omega)= \upphi (I - \gamma \Rho)^\top \sqrt{\Eta} \bar{\Lambda} \sqrt{\Eta} (I - \gamma \Rho)\upphi ^\top,
\end{equation}
where $I \in \mathds{R}^{|\mathcal{S}| \times |\mathcal{S}|}$  is the identity matrix, and $\Rho$ is the transition probability as defined in Section~\ref{Subsection:DRLformulate}. Since  $ E_\theta(\omega)$ is convex, its gradient is $\ell$-Lipschitz if 
\begin{equation}
\label{cond1}
\omega^\top \mathds{H}(\omega)  \omega \leq l||\omega||^2_2
\end{equation}
 holds. Indeed, we have
\begin{dmath}
\omega^\top \mathds{H}(\omega) \omega  = \omega^\top \upphi (I - \gamma \Rho)^\top \sqrt{\Eta} \bar{\Lambda} \sqrt{\Eta} (I - \gamma \Rho)\upphi^\top \omega \\=
||\bar{\Lambda}\sqrt{\Eta}(I - \gamma \Rho)\upphi^\top \omega||^2_2 \\=
||\sqrt{\Eta}\upphi  \omega||_2^2 - 2\gamma \omega^\top \upphi  \Eta \Rho \upphi^\top \omega  +
\gamma^2 ||\bar{\Lambda} \sqrt{\Eta} \Rho  \upphi^\top \omega||^2_2 \\ \leq
||\sqrt{\Eta}\upphi  \omega||_2^2 - 2\gamma \omega^\top \upphi  \Eta \Rho \upphi^\top \omega  +
\gamma^2 ||\sqrt{\Eta} \Rho  \upphi^\top \omega||^2_2 \\=
||\sqrt{\Eta} (I - \gamma \Rho)\upphi^\top \omega||^2_2. 
\end{dmath}
Using the Cauchy-Schwartz inequality, we have 
\begin{dmath}
\omega^\top \mathds{H}(\omega) \omega \leq \left( ||\sqrt{\Eta} \upphi^\top \omega||_2 +\gamma||\sqrt{\Eta} \Rho  \upphi^\top \omega||_2 \right)^2 \leq
\left(  ||\sqrt{\Eta} \upphi^\top \omega||_2^2 + \gamma||\sqrt{\Eta} \upphi^\top \omega||_2 \right)^2 = 
(1+\gamma)^2||\sqrt{\Eta} \upphi^\top \omega||_2^2.
 \end{dmath}
Here, we have
\begin{dmath}
 ||\sqrt{\Eta} \upphi^\top \omega||_2^2 = ||\upphi^\top\omega||^2_\Rho = \sum_{i \in |\mathcal{S}|} d^i \left({(\upphi^i)}^\top \omega\right)^2 \\ \leq 
 \sum_{i \in |\mathcal{S}|} d^i ||\omega^i||^2_2 ||\theta^i||^2_2 \leq \left(\underset{i}{\max}||\upphi^i||^2_2 \right) ||\theta^i||^2_2.
\end{dmath}
Thus, we get
\begin{equation}
\omega^\top \mathds{H}(\omega) \omega \leq (1+\gamma)^2\left(\underset{i}{\max}||\upphi^i||^2_2 \right) ||\theta^i||^2_2,
\end{equation}
which leads to \ref{cond1}, and completes our proof.

\section{Mini-Batch SGD Noise Characterization}
\label{app:SGDnoiseExp}
\begin{lemma}\label{lemma:SGDnoiseExp}
During each aggregation round $t$, for each MD/ES $y \in \mathcal{N} \cup \mathcal{M}$, the variance of stochastic gradient during mini-batch gradient descent iteration $k$, i.e., $\widetilde{\nabla} F_y^{(t)}(\mathbf{x})$ is upper-bounded by
\begin{align}
    \mathbb{E}\left[ \left\Vert \sum_{d\in \mathcal{B}^{(k),e}_{y}} \hspace{-1mm} \frac{\nabla  f(\mathbf{w}^{(k),e-1}_{y},d)}{{{B}}^{(k)}_y}-\nabla  F^{(k)}_y(\mathbf{w}^{(k),e-1}_{y}) \right\Vert^2 \right] & \leq 2\left(1-\frac{{B}^{(k)}_y}{{D}^{(k)}_y}\right) \frac{({\sigma}_{y}^{(k)})^2}{{B}^{(k)}_y} \Theta^2_y,~\forall e,
\end{align}
where $\sigma_y^{(k)}$ denotes the sampled variance of data at the respective node.
\end{lemma}
\begin{proof}
The proof can be easily carried out using the result on the variance of  of simple random sampling in~\cite{lohr2019sampling} and our definition of the data variability in Definition~\ref{Assump:DataVariabilit}, and thus omitted for brevity.
\end{proof}
\end{document}